%% file: icml2023.tex
\documentclass[nohyperref]{article}

\usepackage{microtype}
\usepackage{graphicx}
\usepackage{subcaption}
\usepackage{booktabs} 
\usepackage{thm-restate}

\usepackage{hyperref}
\usepackage{eqparbox}
\usepackage{enumitem}

\usepackage[accepted]{icml2023}

\usepackage{amsmath}
\usepackage{amssymb}
\usepackage{mathtools}
\usepackage{amsthm}
\usepackage{wrapfig}

\usepackage{amsfonts}
\usepackage{color}
\usepackage{pifont}
\usepackage{bbm}

\usepackage[capitalize,noabbrev]{cleveref}

\renewcommand{\algorithmiccomment}[1]{\hfill{\textcolor{gray}{\ensuremath{\triangleright} #1}}}

\theoremstyle{plain}
\newtheorem{theorem}{Theorem}[section]
\newtheorem{Thm}{Theorem}

\newtheorem{Lem}[theorem]{Lemma}
\newtheorem{corollary}[theorem]{Corollary}

\theoremstyle{definition}

\newtheorem{Assump}[theorem]{Assumption}

\usepackage[textsize=tiny]{todonotes}
\usepackage{lipsum}  
\usepackage{tikz}

\newcommand{\entropy}{\calig{H}}
\newcommand{\methodName}{\textsf{DoCoFL}}
\newcommand{\fedavg}{\textsf{FedAvg}}
\newcommand*\circled[1]{\tikz[baseline=(char.base)]{
            \node[shape=circle,draw,inner sep=0.8pt] (char) {#1};}}

\input{Theory/defs}

\icmltitlerunning{DoCoFL: Downlink Compression for Cross-Device Federated Learning}

\begin{document}

\twocolumn[
\icmltitle{DoCoFL: Downlink Compression for Cross-Device Federated Learning}

\begin{icmlauthorlist}
\icmlauthor{Ron Dorfman}{vmware,tech}
\icmlauthor{Shay Vargaftik}{vmware}
\icmlauthor{Yaniv Ben-Itzhak}{vmware}
\icmlauthor{Kfir Y. Levy}{tech}
\end{icmlauthorlist}

\icmlaffiliation{tech}{Viterby Faculty of Electrical and Computer Engineering, Technion, Haifa, Israel}
\icmlaffiliation{vmware}{VMware Research}

\icmlcorrespondingauthor{Ron Dorfman}{rdorfman@campus.technion.ac.il}

\icmlkeywords{Machine Learning, ICML}

\vskip 0.3in
]



\printAffiliationsAndNotice{}  

\input{sections/abstract.tex}

\input{sections/introduction.tex}

\input{sections/background_and_related_work.tex}

\input{sections/docofl.tex}

\input{sections/anchor_compression.tex}

\input{sections/experiments.tex}

\input{sections/conclusion.tex}

\input{sections/acknowledgements}

\bibliography{icml2023}
\bibliographystyle{icml2023}

\appendix
\onecolumn

\input{appendix/weight_compression_negative_convex_example.tex}

\input{appendix/anchor_compression_intuition.tex}
\input{appendix/main_proof.tex}

\input{appendix/ecuq.tex}

\input{appendix/ecuq_vs_sparsification_vs_sketching.tex}

\input{appendix/experiments_config.tex}

\input{appendix/additional_results.tex}

\input{appendix/future_work}

\end{document}

%% file: Theory/defs.tex
\usepackage[utf8x]{inputenc}
\usepackage{amsfonts}
\usepackage{amsmath}
\usepackage{amssymb}
\usepackage{amsthm}
\usepackage{mathtools}
\usepackage{bbm}
\usepackage{graphicx, nicefrac}
\usepackage{subcaption}
\usepackage{booktabs} 
\usepackage{comment}
\usepackage{hyperref}
\usepackage{color}
\usepackage{bm} 
\usepackage{multirow}
\usepackage{fancyhdr}
\usepackage{afterpage}
\usepackage{pifont}

\usepackage{caption}
\DeclareCaptionFormat{myformat}{#3}

\DeclarePairedDelimiter\abs{\lvert}{\rvert}

\DeclarePairedDelimiter\ceil{\lceil}{\rceil}

\newcommand{\RN}[1]{%
  \textup{\uppercase\expandafter{\romannumeral#1}}%
}

\newcommand{\blue}[1]{\textcolor{black}{#1}}

\newcommand{\calig}[1]{\mathcal{#1}}
\newcommand{\bb}[1]{\mathbb{#1}}
\newcommand{\brac}[1]{\left(#1\right)}
\newcommand{\sbrac}[1]{\left[#1\right]}
\newcommand{\cbrac}[1]{\left\{#1\right\}}

\newcommand{\bbE}{\bb{E}}
\newcommand{\reals}{\bb{R}}
\newcommand{\prob}{\bb{P}}

\newcommand{\A}{\calig{A}}

\newcommand{\C}{\calig{C}}

\newcommand{\Ocal}{\calig{O}}

\renewcommand{\P}{\calig{P}}
\newcommand{\Q}{\calig{Q}}
\newcommand{\Scal}{\calig{S}}
\newcommand{\Tcal}{\calig{T}}

\newcommand{\V}{\calig{V}}

\newcommand{\bnorm}[1]{\left\Vert#1\right\Vert}
\newcommand{\norm}[1]{\lVert#1\rVert}

\newcommand*{\QEDW}{\null\nobreak\hfill\ensuremath{\square}}

\def\naturals{{\mathbb N}}

\newcommand{\ignore}[1]{}

\def\reals{{\mathbb R}}

\def\V{{\mathcal V}}

\def\bold0{\mathbf{0}}

\newcommand{\nab}{\nabla}
\newcommand{\nabhat}{\hat{\nabla}}

\newcommand{\modulu}[2]{#1\hspace{0.2em}\text{mod}\hspace{0.2em}{#2}}

%% file: sections/abstract.tex
\begin{abstract}
Many compression techniques have been proposed to reduce the communication overhead of Federated Learning training procedures. However, these are typically designed for compressing model updates, which are expected to decay throughout training. As a result, such methods are inapplicable to downlink (i.e., from the parameter server to clients) compression in the cross-device setting, where heterogeneous clients \textit{may appear only once} during training and thus must download the model parameters. Accordingly, we propose \methodName{} -- a new framework for downlink compression in the cross-device setting. Importantly, \methodName{} can be seamlessly combined with many uplink compression schemes, rendering it suitable for bi-directional compression. Through extensive evaluation, we show that \methodName{} offers significant bi-directional bandwidth reduction while achieving competitive accuracy to that of a baseline without any compression.  
\end{abstract}

%% file: sections/introduction.tex
\section{Introduction}\label{sec:intro}
In recent years, there has been an increasing interest in federated learning (FL) as a paradigm for large-scale machine learning over decentralized data~\cite{konevcny2016federatedb, kairouz2021advances}. FL enables organizations and/or devices, collectively termed \textit{clients}, to jointly build better and more robust models by relying on their collective data and processing power. 
Importantly, the FL training procedure occurs without exchanging or sharing client-specific data, thus ensuring some degree of privacy and compliance with data access rights and regulations~(e.g., the General Data Protection Regulation (GDPR) implemented by the European Union in May, 2018). Instead, in each round, clients perform local optimization using their local data and send only model updates to a central coordinator, also known as the \textit{parameter server} (PS). The PS aggregates these updates and updates the global model, which is then utilized by the clients in subsequent rounds.

One of the main challenges in FL is the communication bottleneck introduced during the distributed training procedure. To illustrate this bottleneck, consider the following example of a real FL deployment presented by~\citet{mcmahan2022federated}: their training involves a small neural network with $1.3$ million parameters; in each round there are $6500$ participating clients; and the model is trained over $2000$ rounds. A simple calculation shows that the total required bandwidth to and from the PS during this training is $\approx61.5$ TB. Since modern machine learning models have many millions (or even billions) of parameters and we might have~more~participants,~FL~may~result~in~excessive~communication~overhead.

To deal with this overhead, many bandwidth reduction techniques have been proposed. These include taking multiple (rather than a single) local optimization steps~\cite{mcmahan2017communication}, quantization techniques~\cite{seide20141, alistarh2017qsgd, wen2017terngrad, bernstein2018signsgd, bernstein2018signsgd_b, karimireddy2019error, jin2020stochastic, shlezinger2020uveqfed}, low-rank decomposition~\cite{vogels2019powersgd}, sketching~\cite{ivkin2019communication, rothchild2020fetchsgd}, and distributed mean estimation~\cite{lyubarskii2010uncertainty, suresh2017distributed, konevcny2018randomized, vargaftik2021drive, vargaftik2022eden,safaryan2022uncertainty}. However, as we detail in \S\ref{sec:background_and_related_work}, a direct application of these techniques is less suitable for downlink compression, i.e., from the PS to the clients, in the cross-device setup in which new and heterogeneous clients may participate at each round and thus must download the model parameters. This is in contrast to the cross-silo setup in which the PS can compress and send a global update (i.e., clients' aggregated update) to all silos.

To the best of our knowledge, only a handful of works consider bi-directional compression, i.e., compression from the clients to the PS and vice versa. These works mainly rely on per-client memory mechanism~\cite{tang2019doublesqueeze, zheng2019communication, liu2020double, philippenko2020bidirectional, philippenko2021preserved,gruntkowska2022ef21} or require keeping an updated copy of the model on all clients \cite{horvoth2022natural}, thus targeting either distributed learning or FL with full or partial but \emph{recurring} participation (e.g., cross-silo FL). Such solutions are less suitable for large-scale cross-device FL, where a client may appear only a handful of times, or even just once, during the entire training procedure.

It is important to stress that the significance of bi-directional bandwidth reduction for cross-device FL goes far beyond cost reduction, energy efficiency, and carbon footprint considerations. In fact, inclusion, fairness, and bias are at the very heart of cross-device FL as, according to recent sources~\cite{wifidiffs, mobilepricediffs}, the price of a wireless connection and its quality admits differences of orders of magnitude among countries. This may prevent large populations from contributing to cross-device FL training due to costly and unstable connectivity, resulting in biased and less accurate models.

Accordingly, in this work we introduce \methodName{}, a novel downlink compression framework specifically designed for cross-device FL. Importantly, it operates independently of many uplink compression techniques, making it suitable for bi-directional compression in cross-device setups. 

The primary challenge addressed by \methodName{} is that clients must download model parameters (i.e., weights) instead of model updates. Unlike updates, which are proportional to gradients and thus their norm is expected to decrease during training, the model parameters do not decay, rendering low-bit compression methods undesirable. As a result, and since clients can only download the updated model weights during their designated participation round, this can lead to a network bottleneck for low-resourced clients. 

To address this bottleneck, \methodName{} decomposes the download burden by utilizing previous models, referred to as  \emph{anchors}, which clients can download prior to their participation round. Then, at the designated participation round, clients only need to download the correction, i.e., the difference between the updated model and the anchor. As the correction is proportionate to the sum of previous updates, it is expected to decay, allowing for the use of low-bit compression methods. To ensure the correction term, PS memory footprint, and PS computational overhead remain manageable, the available anchors are updated periodically. This approach reduces the amount of bandwidth required by the clients \emph{online} (i.e., at their participation round). To reduce the overall downlink bandwidth usage, we further develop and utilize an efficient anchor compression technique with an appealing bandwidth-accuracy tradeoff.

\vspace{-0.5em}
\paragraph{Contributions.} We summarize our contributions below,
\vspace{-0.7em}
\begin{itemize}[leftmargin=*]
\setlength\itemsep{0.2em}
    \item We propose a new framework (\methodName{}) that both enlarges the time window during which clients can obtain the model parameters and reduces the total downlink bandwidth requirements for cross-device FL.
    \item We show that \methodName{} provably converges to a stationary point when not compressing anchors and give an asymptotic convergence rate.  
    \item We design a new compression technique with strong empirical results, which \methodName{} uses for anchor compression and can be of independent interest. We provide the theoretical intuition and empirical evidence for why \methodName{} with anchor compression works.
\end{itemize}
Finally, we show over image classification and language processing tasks that \methodName{} consistently achieves model accuracy that is competitive with an uncompressed baseline, namely, \fedavg{}~\cite{mcmahan2017communication} while reducing bandwidth usage in both directions by order of magnitude.

%% file: sections/background_and_related_work.tex
\section{Background and Related Work}\label{sec:background_and_related_work}
In this section, we overview mostly related work and detail the challenges in designing a bi-directional bandwidth reduction framework for cross-device FL.

\subsection{Uplink vs. Downlink \blue{Compression}}
In the context of FL, uplink (i.e., client to PS) and downlink (i.e., PS to client) compression are inherently different and should not be treated in the same manner. In particular, many recent uplink compression solutions (e.g., \citet{konevcny2016federateda,alistarh2017qsgd}) partially rely on two properties to obtain their effectiveness:

\textbf{Averaging.} 
A fundamental property arises when many clients send their compressed gradients for averaging at the PS. If the clients' estimates are independent and unbiased, the error in estimating their mean by calculating their estimations' mean is decreasing linearly with respect to the number of clients. Thus, having more clients in each round allows for more aggressive and more accurate compression.   

\textbf{Error Decay.} 
Essentially, unbiased compression of updates results in an increased variance in their estimation. This increase can be compensated by decreasing the learning rate. Moreover, the effect of update compression is expected to diminish since the expected update decays as the training process approaches a stationary point. This is not the case when compression model parameters.


For downlink compression, we immediately lose the averaging property since, by design, there is only one source with whom the clients communicate, namely, the PS. Regarding the error decay property, we must further distinguish between different FL setups as described next.
\subsection{Cross-silo vs. Cross-device \blue{FL}}
FL can be divided into two types based on the nature of the participating clients (\citet{kairouz2021advances}, Table 1).

\textbf{Silos.} In cross-silo FL, the clients are typically assumed to be active throughout the training procedure and with sufficient compute and network resources.  
Silos are typically associated with entities such as hospitals that jointly train a model for better diagnosis and treatment~\cite{ng2021federated} or banks that jointly build better models for fraud and anomalous activity detection~\cite{yang2019ffd}.

\input{tables/table_comp_effects.tex}

Indeed, silos allow for the design of efficient compression techniques that rely on client persistency and per-client memory mechanisms that are used for, e.g., 
compressing gradient differences, employing error feedback, and learning control variates~\cite{alistarh2018convergence,karimireddy2020scaffold,philippenko2020bidirectional, gorbunov2021marina,richtarik2021ef21}. While most of these techniques consider only uplink compression, some recent works target bi-directional compression by utilizing the same property of using per-client memory and relying on \emph{repeated} client participation~\cite{tang2019doublesqueeze,liu2020double,philippenko2020bidirectional,philippenko2021preserved,gruntkowska2022ef21}.

\textbf{Devices.} In cross-device FL, clients are typically assumed to be heterogeneous and not persistent to the extent that a client often participates in a \emph{single} out of many thousands of training rounds. Also, in this setup, clients may often admit compute and network constraints. Devices are usually associated with entities such as laptops, smartphones, smartwatches, tablets, and IoT devices. A typical example of a cross-device FL application is keyboard completion for android devices~\cite{mcmahan2017federated}.

Unlike in silos with full or partial but repeated participation, compression techniques for devices that appear only once or a handful of times cannot rely on having some earlier state for or on that device. This renders methods that rely on per-client memory or learned control variates less suitable for such cross-device FL setups. Indeed, recent gradient compression techniques can be readily used for bi-directional compression in the cross-silo setup or only uplink compression in both setups~\cite{konevcny2016federateda,alistarh2017qsgd,suresh2017distributed,ramezani2021nuqsgd,vargaftik2021drive,vargaftik2022eden,safaryan2022uncertainty}.
\subsection{Putting It All Together}
As summarized in Table \ref{tab:comp_effects}, differences in the clients' nature and the compression direction (i.e., uplink vs. downlink) significantly affect the efficiency of bandwidth reduction techniques. 
In the considered setups, downlink compression is more challenging than uplink compression due to the lack of averaging and received considerably less attention in the literature. Moreover, for the cross-device setup, the problem is more acute due to not having error decay as well.

%% file: tables/table_comp_effects.tex
\begin{table}[t]
\centering
\caption{Averaging and Error Decay in different setups.}
\vspace{0.1in}
\resizebox{\linewidth}{!}{%
\begin{tabular}{cc|c|c|}
\cline{3-4}
& \multicolumn{1}{l|}{} & Averaging & Error Decay \\ \hline
\multicolumn{2}{|c|}{Uplink}                                             & {\color[HTML]{32CB00} \checkmark} & {\color[HTML]{32CB00} \checkmark} \\ \hline
\multicolumn{1}{|c|}{}                           & Cross-silo            & {\color[HTML]{FE0000} $\times$}  & {\color[HTML]{32CB00} \checkmark} \\ \cline{2-4} 
\multicolumn{1}{|c|}{\multirow{-2}{*}{Downlink}} & Cross-device          & {\color[HTML]{FE0000} $\times$}  & {\color[HTML]{FE0000} $\times$}  \\ \hline
\end{tabular}%
}
\vspace{0.1in}
\label{tab:comp_effects}
\end{table}

%% file: sections/docofl.tex
\section{\methodName}\label{sec:docofl}
In this section, we present \methodName{}. We start with describing our design goals, which are derived from the challenges outlined in the previous section, followed by a formal definition of the federated optimization problem. Then, in \S\ref{subsec:overview}, we give intuition and introduce our framework. In \S\ref{subsec:client_selection_process}, we detail about an important element of \methodName{}, namely, the client selection process employed by the PS. Finally, we provide a theoretical convergence result in \S\ref{subsec:theory}.

\textbf{Design Goals. } Motivated by the discussion in the previous section, we aim at achieving two goals to deal with the low bandwidth and slow and unstable connectivity conditions that edge devices may experience:
\vspace{-0.3em}
\begin{enumerate}
    \setlength\itemsep{0.1em}
    \item \textit{Enlarging the time window} during which a client can download the model weights from the PS.
    \item \textit{Reducing the bandwidth} requirements in the downlink direction.
\end{enumerate}

Achieving both these goals will enable more heterogeneous clients to participate in the training process, which in turn may reduce bias and improve fairness\footnote{\blue{By fairness, we refer to the situation where clients in regions with limited, unstable, and costly connectivity are keen to participate in the training procedure.}}.

\textbf{Preliminaries. } We use $\lVert\cdot\rVert$ to denote the $L_2$ norm and for every $n\in\naturals$,  $\sbrac{n}\coloneqq\cbrac{1,\ldots,n}$. Let $N$ be the number of clients participating in the federated training procedure. Each client $i\in\sbrac{N}$ is associated with a local loss function $f_i$, and our goal is to minimize the loss with respect to all clients, i.e., to solve
\vspace{-0.5em}
\begin{equation}
    \min_{w\in\reals^d}{f(w)\coloneqq \frac{1}{N}\sum_{i=1}^{N}{f_i(w)}}\; .
\end{equation}
Unlike in standard distributed optimization or cross-silo FL with full or partial but repeated participation, in cross-device FL, only a subset of $S$ clients participate in each optimization round and typically $S\ll N$ (e.g., $S$ in the hundreds/thousands and $N$ in many millions). Thus, clients are not expected to repeatedly participate in the optimization.

Since FL mostly considers non-convex optimization (e.g., neural networks), and global loss minimization of such models is generally intractable, we focus on finding an approximate stationary point, i.e., a point $w$ for which the expected gradient norm $\bbE\norm{\nabla f(w)}$ tends to zero.

For the purpose of formal analysis, we make a few standard assumptions, namely, that $f$ is bounded from below by $f^*$, the local functions $\cbrac{f_i}$ are $\beta$-smooth, i.e., $\norm{\nabla f_i(w) - \nabla f_i(u)}\leq \beta\norm{w-u}, \forall w,u\in\reals^d$, and the access to each local function is done via a stochastic gradient oracle, i.e.,
\begin{Assump}\label{assump:stochastic_grad_oracle}
    For any $w\in\reals^d$, client $i$ computes an unbiased gradient estimator $g^i(w)$ with a variance that is upper bounded by $\sigma^2$, i.e.,
    \begin{equation}
        \bbE[g^i(w)] = \nabla f_i(w), \enskip \bbE\norm{g^i(w) - \nabla f_i(w)}^2\!\leq\!\sigma^2\; .
    \end{equation}
\end{Assump}
Additionally, we assume that the dissimilarity of the local gradients is bounded (i.e., limited client data heterogeneity).
\begin{Assump}\label{assump:bgd}
    There exist constants $G, B\in\reals_{+}$ such that for every $w\in\reals^d$:
    \begin{equation}
        \frac{1}{N}\sum_{i=1}^{N}{\norm{\nabla f_i(w)}^2}\leq G^2 + B^2\norm{\nabla f(w)}^2\; .
    \end{equation}
\end{Assump}
\blue{While some works consider milder~\cite{khaled2020better,haddadpour2021federated} or no~\cite{gorbunov2021marina} assumptions on client heterogeneity in some settings (e.g., exact gradients and/or full participation), this assumption is standard in heterogeneous federated learning~\cite{karimireddy2020scaffold, wang2020tackling,DBLP:conf/iclr/ReddiCZGRKKM21}.}
\subsection{Overview}\label{subsec:overview}
\begin{figure*}
    \centering
    \includegraphics[clip, trim=1.5cm 0cm 0.1cm 0cm,width=0.9\linewidth]{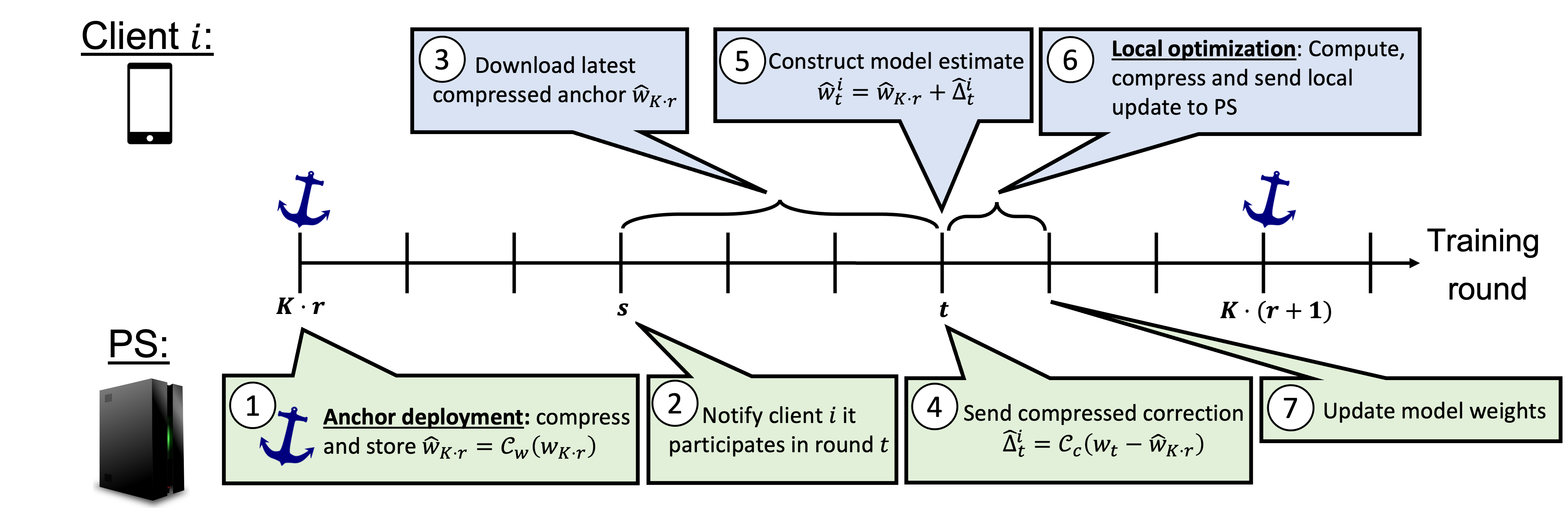}
    \caption{\methodName{}'s training procedure. Here, we illustrate the interaction between the PS and a single client. Typically, multiple clients participate in each training round, and each client may be notified about its participation in a different round.}
    \label{fig:docofl_timeline}
\end{figure*}

A naive approach to reduce bandwidth in the downlink direction is to apply some compression to the model weights and have all participating clients download the compressed weights. That is, in each round $t$, the participating clients $\Scal_t\subseteq\sbrac{N}$ obtain a compressed version of the model weights $\hat{w}_t = \C_w(w_t)$, for some compression operator $\C_w$. The clients can then compute an unbiased gradient estimator at $\hat{w}_t$ and send it back to the PS for aggregation. 

While this method is fairly simple, it has inherent disadvantages with respect to our goals. First, there is no enlarged time window during which clients can download the compressed model weights, as they can only do so at their participation round. Second, unlike with gradient compression, convergence in this setting can be guaranteed only to a proximity that is proportional to the compression error, rendering standard low-bit compression schemes unusable. Indeed, this is the case even for strongly convex functions, as shown by~\citet{chraibi2019distributed} and reinforced by our counter-example in \cref{app:weights_compression_example}.

To tackle these challenges, our approach relies on the following relation: for any $\tau\geq 0$, we can decompose $w_t$ into two ingredients: \emph{Anchor} and \emph{Correction}. Formally,
\begin{equation*}
    \quad w_t = \underbrace{w_{t-\tau}}_{\textit{(i) Anchor}} + \underbrace{w_t - w_{t-\tau}}_{\textit{(ii) Correction}}\;.
\end{equation*}
This implies that:
\begin{itemize}
    \setlength\itemsep{0.1em}
    \item[\textit{(i)}] If a client is notified at round $t-\tau$ about its upcoming participation at round $t$, it can start downloading the anchor, that is, $w_{t-\tau}$, \textit{ahead of its participation round},\footnote{The client can start downloading 
    $w_{t-\tau'}$ at round $t-\tau'$ for some $\tau' \leq \tau$, as long as the download is complete \emph{before} round $t$.}
    \item[\textit{(ii)}] and thus, at round $t$, the client only needs to download the correction, that is, $w_t - w_{t-\tau}$.
\end{itemize}
\vspace{-0.7em}
Yet, merely relying on this relation is not sufficient to achieve our goals; additionally, we seek to \emph{compress} both \emph{(i)} and \emph{(ii)}. However, these terms are inherently different and therefore, should not be treated in the same manner. Essentially, the client has more time to download \emph{(i)}, which is the main ingredient that forms the model weights. Introducing a large error in this term may prevent the model from converging. Conversely, \emph{(ii)} must be downloaded at the participation round of the client, but it is just the sum of $\tau$ recent gradients. 
\begin{algorithm}[t]
\small
\caption{\methodName{} -- Parameter Server
}\label{alg:docofl_ps}
\begin{algorithmic}
    \STATE {\bfseries Input:} Initial weights $w_0\in\reals^d$, learning rate $\eta$, weights (anchors) compression $\C_{w}$, correction compression $\C_{c}$, anchor compression rate $K$, compressed anchors queue $Q\leftarrow\emptyset$ with capacity $\V$, client participation process $\P(\cdot)$

    \FOR{$t=0,\ldots,T-1$}
    \vspace{0.15cm}
    \STATE \textcolor{blue}{$\triangleright\enskip$Anchor Deployment}
    \IF{$\modulu{t}{K} == 0$} 
        \STATE Compress anchor, $\C_{w}(w_t)$
        \STATE $Q.\textup{enqueue}(\C_w(w_{t}))$    \algorithmiccomment{If $Q$ is full, $Q.\textup{dequeue}()$}
    \ENDIF
    \vspace{0.15cm}
    \STATE \textcolor{blue}{$\triangleright\enskip$Client Participation Process}
    \STATE $\Scal_t\leftarrow\P(t)$ \algorithmiccomment{$\abs{\Scal_t} = S$; see \S\ref{subsec:client_selection_process}}
    \vspace{0.15cm}
    \STATE \textcolor{blue}{$\triangleright\enskip$Optimization }
        \FOR{client $i\in\Scal_{t}$ in parallel}
            \STATE Send compressed correction,
            $\hat{\Delta}_{t}^i$ \algorithmiccomment{See \cref{alg:docofl_client}}
            \STATE Obtain compressed local gradient, $\C_g(\hat{g}_t^i)$
        \ENDFOR
        \STATE Aggregate local gradients, $\hat{g}_t\coloneqq \frac{1}{S}\sum_{i\in\Scal_t}{\C_g(\hat{g}_t^i)}$
        \STATE Update weights, $w_{t+1} = w_{t} - \eta \hat{g}_t$
        \vspace{0.15cm}
    \ENDFOR
\end{algorithmic}
\end{algorithm}
\setlength{\textfloatsep}{8pt}
\begin{algorithm}[t]
\small
\caption{\methodName{} -- Client $i$}\label{alg:docofl_client}
\begin{algorithmic}
    \STATE {\bfseries Input:} Gradient compression $\C_g$
        \STATE \textcolor{blue}{\underline{Notification round $s$ (by process $\P$):}} 
        \vspace{0.05cm}
        \STATE Obtain participation round $t$  \algorithmiccomment{$i\in\P(t), s\leq t$}
            \STATE Obtain latest compressed anchor, $y_t^i\leftarrow Q.\textup{top}()$ \algorithmiccomment{Within time window $\sbrac{s,t}$}
            \vspace{0.15cm}
            \STATE \textcolor{blue}{\underline{Participation round $t$:}}
            \STATE Obtain compressed correction,
            $\hat{\Delta}_{t}^i = \C_c(w_t - y_t^i)$
            \STATE Construct current model estimate, $\hat{w}_t^i = y_t^i + \hat{\Delta}_t^i$
            \STATE Compute local gradient, $\hat{g}_t^i = g_t^i(\hat{w}_t^i)$
            \STATE Compress and send local gradient, $\C_{g}(\hat{g}_t^i)$, to PS
\end{algorithmic}
\end{algorithm}

For \emph{(i)}, we develop a new compression technique (see~\S\ref{sec:anchor_compression}), that achieves a better accuracy-bandwidth tradeoff than gradient compression techniques at the cost of higher complexity; we amortize its complexity over several rounds. Using this technique with only several bits per coordinate (e.g., $4$) results in a negligible error. For \emph{(ii)}, we use standard gradient compression techniques with as low as $0.5$ bit per coordinate. Since \emph{(ii)} is only a sum of $\tau$ recent gradients, this error is expected to decay as training progresses.

Overall, we achieve both goals as (1) the download time window is enlarged, with only a small bandwidth fraction that must be used online; and (2) the total downlink bandwidth usage is reduced by up to an order of magnitude compared to existing solutions and without degrading model accuracy.

In Figure~\ref{fig:docofl_timeline}, we show a timeline that illustrates the training procedure of \methodName{}, as we now formally detail on the role of the PS and the clients in our framework.

\textbf{Parameter Server. } As detailed in Algorithm~\ref{alg:docofl_ps}, our PS executes three separate processes throughout the training procedure. 
First, it performs \circled{1} `anchor deployment' -- once in every $K$ rounds, it compresses the model weights and stores the compressed weights in a queue $Q$ of length $\V$. 
Second, the PS employs a client participation process, which we elaborate on in \S\ref{subsec:client_selection_process}; this process determines which clients will participate in a given round. Finally, in each round, the PS obtains the local model updates (i.e., gradients) from the clients participating in that round, computes their average, and \circled{7} updates the model weights. 

\textbf{Client. } Consider a client that is \circled{2} chosen by the PS at some round $s$ to participate in some future round $t$. This means that in round $s\leq t$, the client is notified about its upcoming participation. It can then start \circled{3} downloading the latest anchor stored by the PS; note that the client has a time window of length $t-s$ rounds to download the compressed anchor. At its participation round $t$, the client \circled{4} obtains the compressed correction from the PS, i.e., the compressed difference between the updated model and the compressed anchor obtained earlier, to \circled{5} construct an unbiased estimate of the updated model weights.\footnote{The client obtains the (unbiasedly) compressed difference between $w_t$ and the \textit{compressed} anchor rather than the exact anchor. Thus, the model estimate is unbiased, even if the compressor $\C_w$ is biased. A similar mechanism was used by \citet{horvath2020better} for gradient compression, i.e., `induced compressor'.}~It then \circled{6} locally computes a stochastic gradient, compresses it, and sends the compressed gradient to the PS; see Algorithm~\ref{alg:docofl_client}. 
\subsection{Client Participation Process}\label{subsec:client_selection_process}
An important element in \methodName{} is the client participation process $\P$. For a round number $t$, it returns a subset of clients $\Scal_t\subseteq\sbrac{N}$ of size $S$ to participate in that round. Crucially, in \methodName{}, clients can be notified about their participation prior to their actual participation round. 

This discrepancy between the notification and the participation rounds gives \methodName{} the desired versatility that opens the door for more clients to participate in the optimization process. While current frameworks follow a selection process that notifies a client about its participation just before it takes place, it is possible to consider useful selection/notification processes where some clients (e.g., with weaker connectivity) are notified earlier than others.

Another point to consider is the bias-utility tradeoff, where some choices of $\P$ can allow more clients to participate but may introduce bias in the participation rounds of clients with an untractable effect on the optimization process. Instead, we focus on processes that preserve the property where at each round, the PS can obtain an unbiased estimate of the gradient, which means that we require $\P$ to satisfy the following property: all clients have the same probability of participating in any given round $t$, i.e., $\prob(i\in\P(t)) = S/N$. Surprisingly, such a restriction allows for a wide range of useful selection policies. 

For example, consider a simple scenario where the PS has predetermined time windows $T_s$ and $T_w$ that it associates with ``strongly connected''  and ``weakly connected'' devices, respectively. Then, at each round $t$, the PS randomly selects clients but assigns their participation rounds to $t+T_s$ or $t+T_w$ according to their strength. Observe that this simple, yet very useful scenario satisfies the property we seek after $T_w$ rounds (the first $T_w$ rounds may take longer since, during these initial rounds, the weakly connected clients cannot be notified enough rounds prior to their participation).

\subsection{Theoretical Guarantee}\label{subsec:theory}
The primary challenge in analyzing downlink compression schemes for cross-device FL is that, even when using an unbiased compression method, for which $\mathbb{E}[\C_w(w)] = w$, the resulting gradient estimate $\nabla f(\C_w(w))$ may be biased. This is because, in general, the gradient is not a linear mapping. As mentioned, the resulting bias can hinder convergence; in \Cref{app:weights_compression_example} we show that gradient descent with weights compression may not reach the optimal solution even for strongly convex functions.


Accordingly, we show that \methodName{} converges to a stationary point when $\C_w$ is the identity mapping, i.e., $\C_w(w) = w, \forall w\in\reals^d$ -- a setup that achieves our first goal (i.e., enlarged time window). As our analysis suggests, this identity assumption enables us to effectively bound the gradient bias resulting from the compression. For simplicity, we also assume no uplink compression (i.e., $\C_g$ is also identity), although including it in our analysis is straightforward \blue{for unbiased $\C_g$\footnote{\blue{Unbiasedness in the uplink direction is highly desired since (together with independence) it ensures linearly decaying mean estimation error with respect to the number of clients. For biased $\C_g$, in light of existing results on biased compressors \cite{beznosikov2020biased}, it may be the case that for some biased compressors the theoretical guarantee, with additional challenges, holds.}}}, and we do incorporate it in our experiments. 

Following this result, and the result of \citet{chraibi2019distributed} in the convex case, in \cref{app:goal2_intuition} we give a theoretical intuition and empirical evidence for why \methodName{} works well in setups of interest, when $\C_w$ is \emph{not} the identity function -- achieving our second goal (i.e., total bandwidth reduction).

Before we state our convergence result, we require an additional standard assumption about the correction compression operator $\C_c$, namely, that it has a bounded Normalized Mean-Squared-Error (NMSE)~\cite{philippenko2020bidirectional,richtarik2021ef21,vargaftik2021drive}.
\begin{Assump}\label{assump:bounded_nmse}
    There exists an $\omega\in\reals_{+}$ such that
    \vspace{1mm}
    \begin{equation}
        \bbE\sbrac{\norm{\C_c(w) - w}^2} \leq \omega^2\norm{w}^2, \quad \forall w\in\reals^d\; .
    \end{equation}
\end{Assump}
We now give a convergence result for \methodName{}, namely, Theorem~\ref{thm:main_theorem}. Its full proof is deferred to Appendix~\ref{app:main_proof}; here, we discuss the result and give a proof sketch.
\vspace{2mm}

\begin{restatable}{theorem}{mainThm}
\label{thm:main_theorem}
Let \mbox{$M\!=\! f(w_0) {-} f^*, \tilde{\sigma}^2\!=\! \sigma^2 {+} 4\brac{1 {-} \frac{S}{N}}G^2$}, and $\gamma= 1 + \brac{1 {-} \frac{S}{N}}\frac{B^2}{S}$. Then, \methodName{} with $\C_w$ and $\C_g$ as identity mappings (and appropriate $\eta$) guarantees
    \begin{align} 
        \bbE\sbrac{\frac{1}{T}\sum_{t=0}^{T-1}{\norm{\nabla f(w_t)}^2}}\in \Ocal\left(\sqrt{\frac{M\beta\tilde{\sigma}^2}{TS}} \right. + \nonumber\\ &\hspace{-4.5cm}\left.\frac{(M^2\beta^2\omega^2 K\V\tilde{\sigma}^2)^{1/3}}{T^{2/3}S^{1/3}} + \frac{\gamma M\beta(\omega K\V + 1)}{T}\right) \nonumber .
    \end{align}
\end{restatable}

\vspace{-1mm}
The convergence rate in Theorem~\ref{thm:main_theorem} consists of three terms:
\begin{itemize}
    \setlength\itemsep{0.1em}
    \vspace{-2mm}
    \item The first term $\sqrt{\frac{M\beta\tilde{\sigma}^2}{TS}}$ is a slow statistical term that depends only on the noise level $\tilde{\sigma}^2$ (and the objective's properties); importantly, it is independent of \methodName{}'s hyperparameters, $K, \V$, and $\omega$.  

    \item The last term $\frac{\gamma M\beta(\omega K\V + 1)}{T}$ is a fast deterministic term. When $\omega K\V\in\Ocal(1)$ it decreases proportionally to $1/T$, and otherwise, it is proportional to $\omega K\V/T$.

    \item The middle term $\frac{(M^2\beta^2\omega^2 K\V\tilde{\sigma}^2)^{1/3}}{T^{2/3}S^{1/3}}$ is a moderate term that depends on both noise level and \methodName{}'s hyperparameters through the multiplication $\omega^2 K\V\tilde{\sigma}^2$; it is proportionate to $(T^{2/3}S^{1/3})^{-1}$.
\end{itemize}
We next derive observations from Theorem~\ref{thm:main_theorem}. Henceforth, we omit from $\Ocal(\cdot)$ the dependence on $M, \beta$, and $\gamma$.
\begin{corollary}\label{cor:fedavg_equivalence}
    When $\C_c$ is the identity mapping, i.e., $\omega{=}0$, clients obtain the exact model, and thus our method is equivalent to \fedavg{}. Indeed, we get the same asymptotic rate as \fedavg{}~\citep{mcmahan2017communication, karimireddy2019error}, namely, $\Ocal(\sqrt{\tilde{\sigma}^2/TS} + 1/T)$. 
\end{corollary}

\begin{corollary}
    Suppose $\omega K\V\in\Theta(1)$. In that case, we get the following asymptotic rate:
    \begin{align*}
    \Ocal\brac{\sqrt{\frac{\tilde{\sigma}^2}{TS}} + \frac{(\omega\tilde{\sigma}^2)^{1/3}}{T^{2/3}S^{1/3}} + \frac{1}{T}}\; .
    \end{align*}
    Compared to Corollary~\ref{cor:fedavg_equivalence}, we note that the middle term is the additional cost incurred for utilizing compression. Importantly, it decreases when we improve the correction compression, i.e., reduce $\omega$.
\end{corollary}

\begin{corollary}
    Consider $\omega \!=\! \Theta(1)$. If $K\V\in\Ocal(\sqrt{\tilde{\sigma}^2 T/S})$, the slow term dominates the rate, which is $\Ocal(\sqrt{\tilde{\sigma}^2/TS})$; that is, we can set $K\V$ as large as $\Ocal(\sqrt{\tilde{\sigma}^2 T/S})$ and still get, similarly to \fedavg{}, a speed-up with $S$, the number of participating clients per-round.
\end{corollary}

\noindent\textit{Proof Sketch.} Denote: $\nabla_t\!\coloneqq\!\nabla f(w_t)$ and $\hat{\nabla}_t\!\coloneqq\!\bbE[\hat{g}_t]$. By the update rule, the smoothness of the objective and standard arguments, we obtain that
\begin{align}\label{eq:smoothness_proofsketch}
    \bbE[f(w_{t+1}) - f(w_t)] \leq&
    -\frac{\eta}{2}\bbE\norm{\nab_t}^2 + \frac{\beta\eta^2}{2}\bbE\norm{\hat{g}_t}^2 \nonumber \\ &+ \frac{\eta}{2}\bbE\lVert\nabhat_t - \nab_t\rVert^2\; .
\end{align}
Using the smoothness of $f$, we can bound the last term in the right-hand side, corresponding to the gradient bias, by the clients' average compression error:
\begin{align}\label{eq:proof_sketch_gradient_bias}
    \bbE\lVert\nabhat_t - \nab_t\rVert^2 \leq \beta^2\cdot\bbE\sbrac{\frac{1}{S}\sum_{i\in\Scal_t}{\lVert\hat{w}_t^i - w_t\rVert^2}}\; .
\end{align}
Additionally, in Lemma~\ref{lem:second_moment_bound}, we derive the following bound on the second moment of the stochastic aggregated gradient: 
\vspace{-3mm}
\begin{equation}\label{eq:grads_as_compress_err}
    \bbE\norm{\hat{g}_t}^2\!\leq\!\frac{\tilde{\sigma}^2}{S}{+}4\gamma\bbE\norm{\nab_t}^2{+}\frac{2\beta^2}{S}\bbE\sum_{i\in\Scal_t}{\norm{\hat{w}_t^i-w_t}^2} .
\end{equation} 
Plugging these bounds back to Eq.~\eqref{eq:smoothness_proofsketch}, we obtain:
\begin{align}\label{eq:proof_sketch_sum_over}
    \bbE[f(w_{t+1}) - f(w_t)] \leq& \brac{-\frac{\eta}{2} {+} 2\gamma\beta\eta^2}\bbE\norm{\nab_t}^2 {+} \frac{\beta\eta^2\tilde{\sigma}^2}{2S} \nonumber\\ &\hspace{-2.1cm}+ \brac{\frac{\beta^2\eta}{2} {+} \beta^3\eta^2}\bbE\sbrac{\frac{1}{S}\sum_{i\in\Scal_t}{\lVert\hat{w}_t^i - w_t\rVert^2}} .
    \vspace{-2mm}
\end{align}
\noindent Recall that each client constructs the current model estimate by summing an anchor and a compressed correction, i.e., $\hat{w}_t^i = y_{t}^{i} + \C_c(w_t - y_t^i)$, where $y_t^i$ (i.e., the anchor) is some model from up to $K\V$ rounds ago; for simplicity, assume that all clients obtain the oldest anchor, i.e., $y_t^i = w_{t-K\V}$. Therefore, using Assumption~\ref{assump:bounded_nmse}, we can bound the compression error by the difference between the current model and the obtained anchor, which is proportional to the sum of the last few (aggregated) gradients:
\begin{align*}\label{eq:compress_err_as_grads}
    \bbE\lVert\hat{w}_t^i {-} w_t\rVert^2 \leq \omega^2\bbE\lVert w_t - y_t^i\rVert^2 = \omega^2\eta^2 \bbE\bigg\lVert\sum_{k=t-K\V}^{t-1}{\hat{g}_k}\bigg\rVert^2.
\end{align*}
Denote the client compression error by $e_t^i\coloneqq \bbE\lVert\hat{w}_t^i - w_t\rVert^2$. Decomposing each gradient into bias and variance as $\hat{g}_k = \nabhat_k + \hat{\xi}_k$, where $\bbE[\hat{\xi}_k] = 0$, we get:
\begin{align}
    e_t^i &\leq 2\omega^2\eta^2 \bbE\bigg\lVert\sum_{k=t-K\V}^{t-1}{\hat{\nab}_k}\bigg\rVert^2 + 2\omega^2\eta^2 \bbE\bigg\rVert\sum_{k=t-K\V}^{t-1}{\hat{\xi}_k}\bigg\rVert^2 \nonumber \\ &\hspace{-0.5mm}\leq 2\omega^2\eta^2 K\V\sum_{k=t-K\V}^{t-1}{\bbE\lVert\hat{\nab}_k\rVert^2} {+} 2\omega^2\eta^2 \sum_{k=t-K\V}^{t-1}{\bbE\lVert\hat{\xi}_k\rVert^2} \nonumber ,
\end{align}
where we used the orthogonality of the noises, i.e., $\bbE[\hat{\xi}_k^\top\hat{\xi}_\ell]=0$ for $k\neq\ell$. Plugging-in $\hat{\xi}_k = \hat{g}_k - \nabhat_k$, we obtain:
\begin{align*}
    e_t^i &\!\leq\! 6\omega^2\eta^2 K\V\sum_{k=t-K\V}^{t-1}{\bbE\lVert\nabhat_k\rVert^2} {+} 4\omega^2\eta^2\sum_{k=t-K\V}^{t-1}{\bbE\lVert \hat{g}_k\rVert^2}.
\end{align*}
Using Eq.~\eqref{eq:proof_sketch_gradient_bias} and \eqref{eq:grads_as_compress_err} to bound $\bbE\lVert\nabhat_k\rVert^2$ and $\bbE\lVert\hat{g}_k\rVert^2$, respectively, we get a recursive relation as the client compression error at round $t$ depends on all prior errors. This is due to error accumulation from computing the aggregated gradients at inaccurate iterates. Lemma~\ref{lem:model_estimation_error} provides a (non-recursive) bound on the compression error at round $t$. Plugging this bound back to Eq.\enskip\eqref{eq:proof_sketch_sum_over}, summing over $t=0,\ldots,T-1$, and using some algebra, we get: 
\begin{align*}
    \bbE[f(w_T) - f(w_0)] \leq& -\frac{\eta}{4}\sum_{t=0}^{T-1}{\bbE\norm{\nab_t}^2} \\ &+ \brac{\frac{\beta\eta^2}{2} + 12\beta^2\omega^2 K\V\eta^3} \frac{T\tilde{\sigma}^2}{S}\; .
\end{align*}
Rearranging terms and tuning $\eta$ concludes the proof. $\QEDW$ 

It is important to note that our framework may also introduce some opportunities for system-wise improvements that are not captured by standard analysis. For example, with a larger pool of clients that are able to participate in a training procedure, it may be easier and faster to reach the desired threshold of participants in each round. Also, it may offer access to more data overall with a different resulting model. How to capture and model such potential benefits in a way that is consistent with and useful in real deployments? Indeed, this is an interesting and significant challenge for future work that may yield new FL policies.

%% file: sections/anchor_compression.tex
\section{Anchor Compression}\label{sec:anchor_compression}
Compressing the anchors is an essential building block of \methodName{} for reducing the total downlink bandwidth. While many compression techniques exist, most techniques were designed for gradient compression. Although we can use many such methods in our framework, it is less desirable to use a gradient compression scheme for anchor compression since the compression error of the anchor has a larger impact on the resulting model accuracy than the correction error; recall that the model weights, unlike the correction, do not decay throughout training. Accordingly, we designed a compression technique for that purpose. 

We first observe that this technique is considerably less~restricted on the PS side (i.e., compression) than on the client’s side (i.e., decompression). On the PS side, we typically have more resources and time (a new anchor is deployed only every $K$ rounds) to employ more complex calculations, where at the client side we seek speed and lighter computations. 

Consequently, we devised a compression method called \emph{Entropy-Constrained Uniform Quantization (ECUQ)}. The main idea behind this approach is to approximate Entropy-Constrained Quantization (ECQ), which is an optimal scheme among a large family of quantization techniques~\cite{chou1989entropy}. Intuitively, given some vector, ECQ finds the best quantization values (i.e., those that minimize the mean squared error) such that after quantization and entropy encoding (e.g., Huffman coding,~\citet{huffman1952method}) of the resulting quantized vector, a given budget constraint is respected. However, this approach is slow, complex, and unstable (sensitive to hyperparameters), which renders it unsuitable for online compression of large vectors. 

As we detail in Appendix~\ref{app:ecuq}, ECUQ employs a double binary search to efficiently find the maximal number of \emph{uniformly spaced} quantization values (between the minimal and maximal values of the input vector), such that after quantization, the entropy of the vector would be within a small threshold from a given bandwidth constraint. Since computing the entropy of a vector does not require to actually encode it, the double binary search is executed fast, and only after finding these quantization values, we encode the vector. 

In Appendix~\ref{app:ecuq}, we compare ECUQ, ECQ and a technique based on K-Means clustering (ECK-Means), which also approximates ECQ (see \cref{fig:synth}); our results indicate that ECUQ is always better than ECK-Means and competitive with ECQ while being orders of magnitude faster. 

We also compare ECUQ with four recent gradient compression techniques: \textbf{(1)} Hadamard followed by stochastic quantization (SQ)~\cite{suresh2017distributed}; \textbf{(2)} Kashin’s representation followed by SQ~\cite{lyubarskii2010uncertainty, safaryan2022uncertainty}; \textbf{(3)} QSGD followed by Elias Gamma encoding~\cite{alistarh2017qsgd}; and \textbf{(4)} EDEN~\cite{vargaftik2022eden}. We test these in three different scenarios: \textbf{(1)} Model parameters of a convolutional neural network (CNN) with~$\approx{11}$M parameters; \textbf{(2)} Model parameters of an LSTM network with~$\approx{8}$M parameters; and \textbf{(3)} Vectors from a \textit{LogNormal}$(0, 1)$ distribution with $1$M entries.~{We repeat each experiment ten times and report the mean.} 

As shown in Figure~\ref{fig:ecuq_vs_gradient_compression}, ECUQ consistently offers the best NMSE, which is by up to an order of magnitude better than that of the second best. We also find that ECUQ is sufficiently fast to be used by the PS every several rounds (a typical cross-device FL round may take minutes to hours).

\begin{figure}[t]
    \centering
    \includegraphics[clip, trim=4.4cm 0.35cm 3.6cm 0.4cm, width=\linewidth]{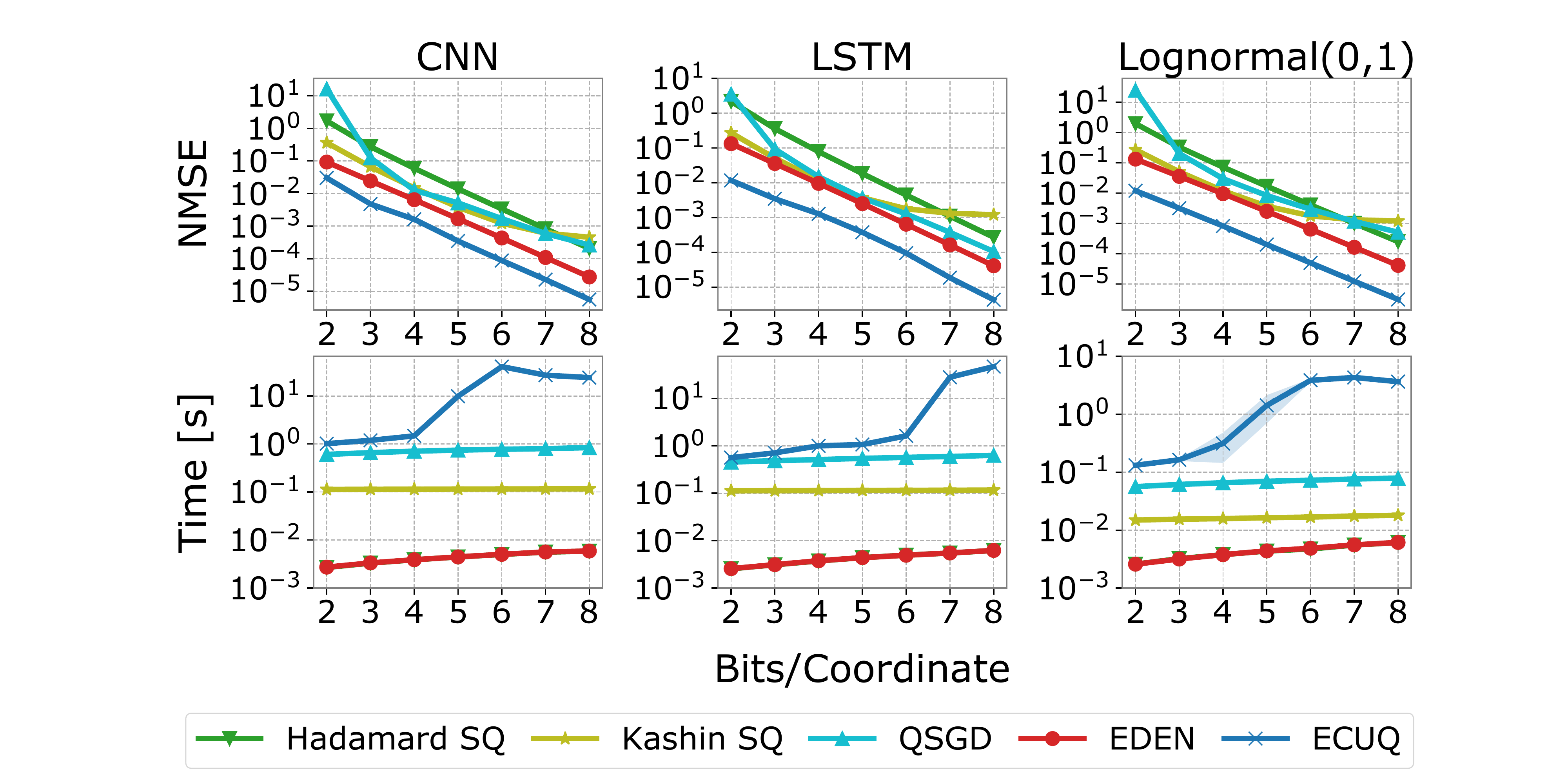}
    \caption{ECUQ vs.~gradient compression methods: NMSE (\textbf{top}) and encoding time (\textbf{bottom}) for three different cases -- recorded model parameters of a CNN (\textbf{left}); LSTM (\textbf{middle}); and vectors drawn from synthetic \textit{LogNormal}$(0, 1)$ distribution (\textbf{right}).} 
    \label{fig:ecuq_vs_gradient_compression}
\end{figure}

\blue{While our comparison here focuses on quantization-based methods, in \Cref{app:ecuq_vs_spars_vs_sketch} we compare ECUQ with three popular compression techniques that do not rely on quantization, namely, Rand-K, Top-K~\cite{alistarh2018convergence}, and Count-Sketch~\cite{charikar2002finding} and show similar trends. Nevertheless, we note that quantization is mostly orthogonal to such techniques and they can be used in conjunction\footnote{\blue{For example, \citet{vargaftik2022eden} use Rand-K as a subroutine alongside quantization to reach a sub-bit compression ratio.}}.}


%% file: sections/experiments.tex
\section{Experiments}\label{sec:experiments}
As previously mentioned, most prior downlink compression methods rely on repeated client participation and/or control variates and are, therefore, less suitable for large-scale cross-device FL where a client may participate only once or a handful of times during the training procedure. \blue{Also, there are prior methods that target model size reduction via sketching~\cite{rabbani2021comfetch} and sparsification~\cite{shah2021model}, but rely on restrictive assumptions and typically result in longer training times and lower accuracy with an increasing number of clients and decreasing participation ratio. Some other model size reduction methods, such as low-rank approximation~(e.g., \citet{DBLP:journals/corr/TaiXWE15}) are orthogonal to \methodName{} and they can be used in conjunction.}

\input{tables/table_exp_config.tex}
\input{tables/table_main_results.tex}

Accordingly, we compare \methodName{} with an uncompressed baseline obtained by running \fedavg{}~\cite{mcmahan2017communication} without any (i.e., uplink or downlink) compression, utilizing full precision (i.e., 32-bit floats) in both directions. Then, we perform an ablation study that shows the consistency of \methodName{} with respect to its hyperparameters.

We cover a wide range of use cases that include two image classification and two language processing tasks with different configurations and data partitioning, as shortly summarized in Table~\ref{tab:exp_config} and further detailed in \cref{app:experiments_config}. 

\textbf{Image Classification. } We use the CIFAR-100 and EMNIST datasets. For CIFAR-100~\cite{krizhevsky2009learning}, the data distribution among the clients is i.i.d. For EMNIST \cite{cohen2017emnist}, the dataset of each client is composed of $10\%$  i.i.d samples from the entire dataset and $90\%$  i.i.d samples of 2 out of 47 classes~\cite{karimireddy2020scaffold}. 

\textbf{Language Processing. } For language processing, we perform a sentiment analysis task on the Amazon Reviews dataset~\cite{zhang2015character} with i.i.d data partitioning; and a next-character prediction task on the Shakespeare dataset~\cite{mcmahan2017communication}, where each client holds data associated with a single role and play.  

In all simulations, we run \methodName{} with ECUQ for anchor compression (i.e., $\C_w$), and EDEN~\cite{vargaftik2022eden} for correction and uplink compression (i.e., $\C_c$ and $\C_g$).

\textbf{Main Results. } In \cref{tab:main_results}, we report the best validation accuracy achieved during training for \fedavg{} and two representative configurations of \methodName{}. It is evident that the validation accuracy of \methodName{} and \fedavg{} is always competitive; in some tasks, \methodName{} performs somewhat better. For example, for EMNIST, \methodName{} reduces the online and total downlink bandwidth by $16\times\!$ and $8\times\!$, respectively, while achieving higher validation accuracy. 

As is often the case in FL, our evaluation indicates that using more bandwidth does not necessarily lead to higher validation accuracy. While using less bandwidth usually impacts the train accuracy, as it implies a larger compression error, it may positively affect the model's generalization ability. We further reinforce these observations in \blue{\cref{app:additional_results}}.

\blue{\textbf{Hyperparameters Ablation. } 
In \cref{fig:anchor_deployment_rate_ablation}, we report the final \emph{train} accuracy of \methodName{} for the CIFAR-100 task with varying values of $K\in\{10,50,100,500\}$ and $\V\in\{3,5,10\}$ under two bandwidth configurations. The results indicate that our framework performs as expected for a wide range of anchor deployment rates and queue capacities. Additionally, in line with our theoretical findings, when the multiplication $K\V$ is too large, the norm of the correction becomes sizable, which can hinder the final accuracy and even convergence. To allow the use of large $K\V$, one may increase the correction bandwidth, trading online bandwidth for a larger anchor download time window.}
\blue{
We defer an ablation study of the anchor and correction bandwidth budgets to \Cref{subsec:bandwidth_budget_ablation}. These results indicate that \methodName{} performs well for a wide range of budgets and provide further intuition for configuring these parameters.  
}

\begin{figure}
    \centering
    \includegraphics[clip, trim=9.1cm 0.52cm 3.6cm 0.05cm ,width=\linewidth]{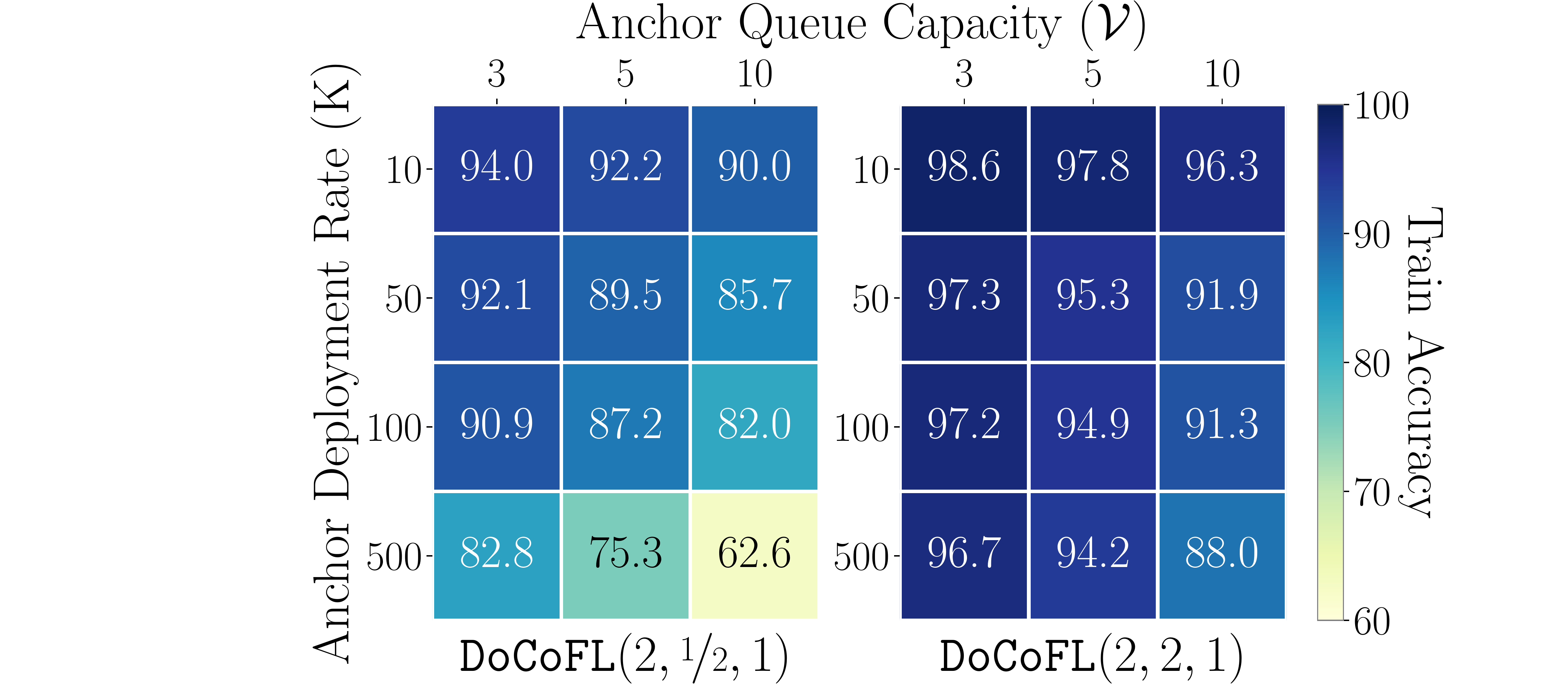}
    \caption{Final train accuracy of \methodName{} for two bandwidth configurations and various values of $K$ and $\V$ on the CIFAR-100 task.}
    \label{fig:anchor_deployment_rate_ablation}
\end{figure}


\blue{\textbf{The Value of the Correction Term. } 
When ignoring the correction, \methodName{} may resemble other frameworks such as delayed gradients (e.g., \citet{stich2019error}) and asynchronous SGD (e.g., \citet{lian2015asynchronous}). In \Cref{subsec:ignoring_correction} we discuss this similarity and convey that ignoring the correction leads to a significant performance drop.}


\blue{\textbf{\methodName{} and \textsf{EF21}.} In \Cref{subsec:docofl_vs_ef21}, we focus on recent advancements based on the \textsf{EF21} technique~\cite{richtarik2021ef21}, which relies on client-side memory. Specifically, we extend \textsf{EF21-PP}~\cite{fatkhullin2021ef21} to support downlink bandwidth reduction using \methodName{} while matching baseline accuracy, where naive model compression results in performance degradation. Also, we discuss some similarities with \textsf{EF21-P} + \textsf{DIANA}~\cite{gruntkowska2022ef21}.}

%% file: tables/table_exp_config.tex

\begin{table}[t]
\centering
\caption{Tasks configuration.}
\vspace{0.1in}
\resizebox{\linewidth}{!}{%
\begin{tabular}{lllll}
\hline
\textbf{Dataset} & \textbf{Net. (\# params)} &\textbf{\# clients ($S$)} & \textbf{Partition}  \\ 
\hline
CIFAR-100 & ResNet-9 ($4.9$M)  & $200~(10)$ & I.I.D \\
EMNIST & LeNet ($65$K) & $1000~(20)$ & Non-I.I.D \\
Amazon & LSTM ($8.3$M) & $500~(10)$ & I.I.D \\
Shakespeare & LSTM ($820$K) & $1129~(20)$ & Non-I.I.D \\ \hline
\end{tabular}%
}
\label{tab:exp_config}
\end{table}

%% file: tables/table_main_results.tex
\begin{table*}[h]
\centering
\caption{Best validation accuracy for different tasks. The configuration triplet $(b_w, b_c, b_g)$ means using $b_w, b_c$, and $b_g$ bits per coordinate for the anchor, correction, and gradient (uplink) compression, respectively. For all tasks, we use $K=10$ and $\V=3$.}
\vspace{0.1in}
\resizebox{\linewidth}{!}{%
\begin{tabular}{cc|clc|clc|clc|clc|}
\cline{3-14}
\multicolumn{1}{l}{}                          &          & \multicolumn{3}{c|}{\textbf{CIFAR-100}}                     & \multicolumn{3}{c|}{\textbf{EMNIST}}                     & \multicolumn{3}{c|}{\textbf{Amazon}}                     & \multicolumn{3}{c|}{\textbf{Shakespeare}}              \\ \cline{3-14} 
\multicolumn{1}{l}{}                          &          & \multicolumn{2}{c|}{$(b_w,b_c,b_g)$}    & Accuracy          & \multicolumn{2}{c|}{$(b_w,b_c,b_g)$}  & Accuracy         & \multicolumn{2}{c|}{$(b_w,b_c,b_g)$}  & Accuracy         & \multicolumn{2}{c|}{$(b_w,b_c,b_g)$}  & Accuracy       \\ \hline
\multicolumn{2}{|c|}{\fedavg{}}                             & \multicolumn{2}{c|}{--}        & $65.03$             & \multicolumn{2}{c|}{--}      & $85.85$            & \multicolumn{2}{c|}{--}      & $\pmb{92.59}$   & \multicolumn{2}{c|}{--}      & $46.10$              \\ \hline
\multicolumn{1}{|c|}{\multirow{2}{*}{\methodName{}}} & Config 1 & \multicolumn{2}{c|}{$(2,2,1)$}   & $64.94$             & \multicolumn{2}{c|}{$(4,4,3)$} & $85.94$            & \multicolumn{2}{c|}{$(6,6,2)$} & $92.51$            & \multicolumn{2}{c|}{$(4,4,4)$} & $45.86$              \\ \cline{2-14} 
\multicolumn{1}{|c|}{}                        & Config 2 & \multicolumn{2}{c|}{$(2,\nicefrac{1}{2},1)$} & $\pmb{65.81}$    & \multicolumn{2}{c|}{$(2,2,3)$} & $\pmb{86.83}$   & \multicolumn{2}{c|}{$(4,4,2)$} & $92.24$            & \multicolumn{2}{c|}{$(2,2,4)$} & $\pmb{46.55}$              \\ \hline
\end{tabular}
}
\label{tab:main_results}
\end{table*}

%% file: sections/conclusion.tex
\section{Conclusion}\label{sec:conclusion}
In this work, we presented \methodName{}, a framework for downlink compression in the challenging cross-device FL setup. By enlarging the clients' model download time window, reducing total downlink bandwidth requirements, and allowing for uplink compression, \methodName{} is designed to allow more resource-constrained and diverse clients to participate in the training procedure. Experiments over various tasks indicate that \methodName{} indeed significantly reduces bi-directional bandwidth usage while performing competitively with an uncompressed baseline. 
\blue{In \cref{app:future_work}, we discuss some directions for future research.}

%% file: sections/acknowledgements.tex
\section*{Acknowledgements}
We thank the reviewers and the area chair for their helpful suggestions. 
KYL is supported by the Israel Science Foundation (grant No.~447/20) and the Technion Center for Machine Learning and Intelligent Systems (MLIS).

%% file: appendix/weight_compression_negative_convex_example.tex
\section{Suboptimality of Gradient Descent with Weights Compression}\label{app:weights_compression_example}
In this section, we give an example of a (strongly) convex function on $\reals$, for which we show that running gradient descent with gradients computed at estimated (i.e., lossily compressed and then decompressed) iterates (rather than at the exact iterates) does not converge to the global minimum. Instead, it converges to a suboptimal solution. 

Let $f:\reals\to\reals$ be the following convex and smooth function:
\begin{equation*}
    f(w) = \frac{1}{2}(w-1)^2 + \frac{1}{2}\sbrac{w-1}_{+}^2\; ,
\end{equation*}
where $\sbrac{w}_{+} = \max\brac{0, w}$. Note that $w^*=1$ is the global minimizer of $f$. 
We analyze the following update rule:
\begin{align*}
    w_{t+1} &= w_t - \eta\nabla f(\hat{w}_t) \\
    \hat{w}_t &= \C_w(w_t)\; ,
\end{align*}
where $\eta>0$ is the step size, and  $\C:\reals\to\reals$ is a randomized, unbiased compression operator with bounded NMSE, i.e.,
\begin{align*}
    \bbE[\C_w(w)] = w, \quad \bbE\abs{\C_w(w) - w}^2\leq \omega_{w}^2\abs{w}^2, \quad \forall w\in\reals\; .
\end{align*}
We can alternatively write: $\hat{w}_t = w_t + \epsilon_t\abs{w_t}$, where $\bbE[\epsilon_t] = 0,\,\, \bbE[\epsilon_t^2]\leq \omega_{w}^2$, and $\cbrac{\epsilon_t}_{t}$ are independent. Thus, we can rewrite the above update rule as
\begin{equation}\label{eq:sfpi}
    w_{t+1} = w_t - \eta\nabla f(w_t + \epsilon_t\abs{w_t})\; .
\end{equation}
In Eq.~\eqref{eq:sfpi}, we repeatedly apply the stochastic mapping: $w\mapsto w - \eta\nabla f(w+\epsilon\abs{w})$. If this process converges in expectation, it converges to a point $\tilde{w}$ for which $\bbE[\nabla f(\tilde{w}+\epsilon\abs{\tilde{w}})]=0$. We show that $w^*=1$ does not satisfy this condition, which will imply that this process does not converge to $w^*$. First, note that $f$ is differentiable, and $\nabla f(w) = (w-1) + \sbrac{w-1}_{+}$. Thus,
\begin{align*}
    \bbE\sbrac{\nabla f(w^* + \epsilon\abs{w^*})} &= \bbE\sbrac{\nabla f(1 + \epsilon)} = \bbE\sbrac{\epsilon + \sbrac{\epsilon}_{+}} = \bbE\sbrac{\epsilon}_{+} \; ,
\end{align*}
where the last equality follows from the linearity of expectation, $\bbE[\epsilon] = 0$. Now, note that unless $\epsilon = 0$ almost surely, we necessarily have $\bbE[\max\cbrac{0,\epsilon}] > 0$, which implies that the iterative update in Eq.~\eqref{eq:sfpi} does not converge in expectation to $w^*=1$. 


%% file: appendix/anchor_compression_intuition.tex
\section{Why \methodName{} with Anchor Compression Works}\label{app:goal2_intuition}
To support \cref{thm:main_theorem}, in which we establish the convergence of \methodName{} when the anchor compression $\C_w$ is identity, in this section we give theoretical intuition and numerical results that convey as for why \methodName{} works when $\C_w$ is not the identity mapping.

Consider the framework we analyze in \cref{app:main_proof}, namely the generalization of \methodName{} given by \cref{alg:general_docoFL}. Adding an anchor compressor (i.e., $\C_w$) implies that each client now obtains a \textbf{compressed} outdated model $\hat{y}_t^i = \C_w(y_t^i)$ and a corresponding correction $\hat{\Delta}_t^i = \C_c(w_t - \hat{y}_t^i)$, and constructs $\hat{w}_t^{i} = \hat{y}_t^i + \hat{\Delta}_t^i$. Thus, adding an anchor compression affects the client's model estimation error $\bbE\norm{\hat{w}_t^i - w_t}^2$, which we bound in Eq.~\eqref{eq:model_estimation_err}. 

Denote by $e_{t, i}^2\coloneqq \norm{y_t^i + \hat{\Delta}_t^i - w_t}^2$ the client's squared estimation error when\textbf{ }not using anchor compression (i.e., when $\C_w$ is identity), and by $\hat{e}_{t,i}^2\coloneqq \norm{\hat{y}_t^i + \hat{\Delta}_t^i - w_t}^2$ the squared estimation error when using anchor compression. If one could show that the following condition holds: 
\begin{equation}\label{eq:condition}
    \bbE[\hat{e}_{t, i}^2] \leq C^2\bbE[e_{t, i}^2]\; ,
\end{equation}
for some moderate $C>0$, then we can simply bound $\bbE[\hat{e}_{t, i}^2]$ in the left-hand side of Eq.~\eqref{eq:model_estimation_err}, and the rest of our analysis holds. However, we know that, in general, this condition does not hold (recall the counter-example in \cref{app:weights_compression_example}).

Nevertheless, we empirically show that it holds in our evaluation where we have non-convex and noisy optimization. More generally, in cross-device FL, client sampling and stochastic gradient estimation add natural noise to the optimization process, and we empirically show that the additional estimation error due to anchor compression with ECUQ is sufficiently low and allows convergence, as conveyed above. 

In \cref{fig:client_stimation_err_ratio} we present the ratio $\rho_t\coloneqq \sum_{i\in\Scal_t}{\hat{e}_{t,i}^2}/\sum_{i\in\Scal_t}{e_{t,i}^2}$ for different anchor compression budgets in CIFAR-100 and Amazon Reviews experiments. First, note that the ratio is mostly stable throughout the entire training. Additionally, when we increase the bandwidth for anchor compression, the ratio decreases, to the extent that for $8$ bits per coordinate, the ratio is $\approx 1$; this is when the error induced by the correction compression dominates the estimation error. 

We note that our intuition gives rise to using an adaptive budget for anchor compression; since $\rho_t$ can be measured by the PS (i.e., it has access to $y_t^i, \hat{y}_t^i$, $w_t$ and the correction compressor $\C_c$), we can keep track of it, and increase the anchor compression budget if $\rho_t$ is too large. We leave such investigation to future work. 

\begin{figure}
    \centering
    \includegraphics[clip, trim=2.5cm 0.05cm 2.5cm 0.2cm, width=0.9\linewidth]{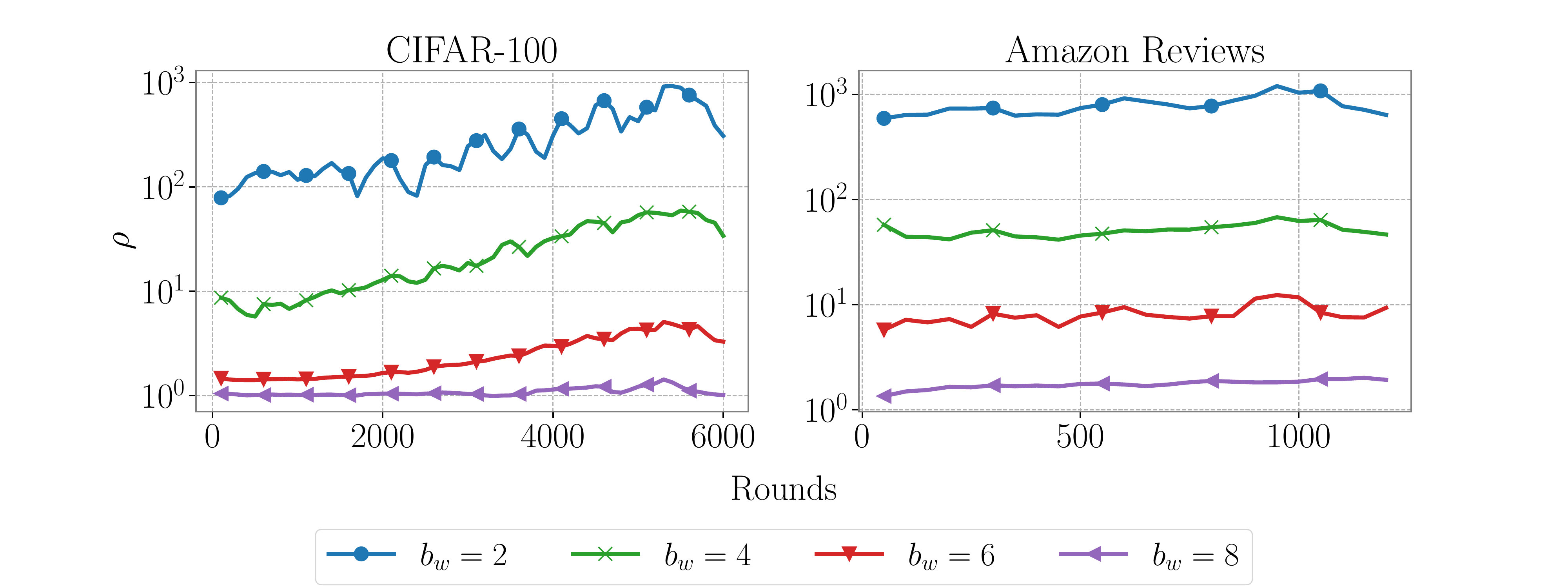}
    \caption{Client estimation error ratio $\rho_t$ on the CIFAR-100 and Amazon Reviews tasks for different anchor compression budgets. For both tasks, we used $2$ bits per coordinate for the correction and gradients (uplink) compression.}
    \label{fig:client_stimation_err_ratio}
\end{figure}

%% file: appendix/main_proof.tex
\section{Proof of Theorem~\ref{thm:main_theorem}}\label{app:main_proof}
In this section we prove Theorem~\ref{thm:main_theorem}, which we restate here for convenience,

\renewcommand{\thesection}{\arabic{section}}
\setcounter{section}{3}
\setcounter{theorem}{3}
\begin{theorem}
    Let $\tilde{\sigma}^2\coloneqq \sigma^2 + 4\brac{1 - \frac{S}{N}}G^2, \gamma\coloneqq 1 + \brac{1 - \frac{S}{N}}\frac{B^2}{S}$, and $M\coloneqq f(w_0) - f^*$. Then, running \methodName{} with $\C_w$ and $\C_g$ as the identity mappings (and with appropriately selected $\eta$) guarantees
    \begin{align}
        \bbE\sbrac{\frac{1}{T}\sum_{t=0}^{T-1}{\norm{\nabla f(w_t)}^2}}\in \Ocal\brac{\sqrt{\frac{M\beta\tilde{\sigma}^2}{TS}} + \frac{(M^2\beta^2\omega^2 K\V\tilde{\sigma}^2)^{1/3}}{T^{2/3}S^{1/3}} + \frac{\gamma M\beta(\omega K\V + 1)}{T}} \nonumber .
    \end{align}
\end{theorem}
\renewcommand{\thesection}{\Alph{section}}

\begin{proof}
To simplify mathematical notations and computations, we analyze a more general framework than \methodName{}, where at each round, each client can download \textit{any} model from up to $\Tcal$ rounds prior to their participation round as anchor. This generalized policy is described in \cref{alg:general_docoFL}. 

\begin{algorithm}[h!]
\caption{Meta-Algorithm (generalization of \methodName{})}\label{alg:general_docoFL}
\begin{algorithmic}
    \STATE {\bfseries Input:} Initial weights $w_0\in\reals^d$, learning rate $\eta$, correction compression $\C_c$, client participation process $\P(\cdot)$
    \FOR{$t=0,\ldots,T-1$}
        \STATE Obtain participating clients, $\Scal_{t}\leftarrow\P(t)$ \algorithmiccomment{$\abs{\Scal_t} = S$}
        \FOR{client $i\in\Scal_{t}$ in parallel}
            \STATE Obtain model weights (anchor), $y_t^i = w_{t - \tau_t^i}$ \algorithmiccomment{$\tau_t^i\in\sbrac{0, \Tcal}$}
            \STATE Obtain compressed correction,
            $\hat{\Delta}_{t}^i = \C_c(w_t - y_t^i)$
            \STATE Construct model estimate, $\hat{w}_t^i = y_t^i + \hat{\Delta}_t^i$
            \STATE Compute local gradient, $\hat{g}_t^i = g_t^i(\hat{w}_t^i)$
            \STATE Communicate $\hat{g}_t^i$ back to server
        \ENDFOR
        \STATE Aggregate local gradients, $\hat{g}_t\coloneqq \frac{1}{S}\sum_{i\in\Scal_t}{\hat{g}_{t}^{i}}$
        \STATE Update weights, $w_{t+1} = w_{t} - \eta \hat{g}_t$
    \ENDFOR
\end{algorithmic}
\end{algorithm}

Using \cref{thm:general_docoFL}, that proves the convergence of \cref{alg:general_docoFL} (see \cref{app:thm:general_docoFL}), we prove \cref{thm:main_theorem}. Namely, \methodName{} with $\C_w$ and $\C_g$ as identity mappings is a private case of \cref{alg:general_docoFL} where $\Tcal=K\V$ and clients can only download models from specific prior rounds (multiplications of $K$). Thus, plugging-in $\Tcal=K\V$ to \cref{thm:general_docoFL} concludes the proof. 
\end{proof}

\subsection{Proof of Theorem~\ref{thm:general_docoFL}}\label{app:thm:general_docoFL}

\begin{Thm}\label{thm:general_docoFL}
    Suppose Assumptions~\ref{assump:stochastic_grad_oracle}-\ref{assump:bounded_nmse} are satisfied. Let $\tilde{\sigma}^2\coloneqq \sigma^2 + 4\brac{1 - \frac{S}{N}}G^2, \gamma\coloneqq 1 + \brac{1 - \frac{S}{N}}\frac{B^2}{S}, \theta\coloneqq\omega\Tcal + 1$, and $M\coloneqq f(w_0) - f^*$. Then, 
    running \cref{alg:general_docoFL} with $\eta = \min\cbrac{\frac{1}{30\gamma\beta\theta}, \sqrt{\frac{2MS}{\beta\tilde{\sigma}^2 T}}, \brac{\frac{MS}{12\beta^2\omega^2\Tcal\tilde{\sigma}^2 T}}^{1/3}}$ guarantees
    \begin{align}
        \bbE\sbrac{\frac{1}{T}\sum_{t=0}^{T-1}{\norm{\nabla f(w_t)}^2}} \leq 4\sqrt{\frac{2M\beta\tilde{\sigma}^2}{TS}} + 8\frac{(12M^2 \beta^2 \omega^2\Tcal\tilde{\sigma}^2)^{1/3}}{T^{2/3}S^{1/3}} + \frac{120\gamma M\beta\theta}{T}\; .
    \end{align}
\end{Thm}

\begin{proof}
For the ease of notation, let $\nab_t\coloneqq \nabla f(w_t)$ and $\tilde{\sigma}^2_S\coloneqq \tilde{\sigma}^2/S$. Throughout our analysis, we sometimes use $\hat{w}_t^i$ even when $i\notin\Scal_t$, which is not well-defined. To resolve this, one can think about the following mathematically equivalent process, where at each round, all clients $i\in\sbrac{N}$ obtain some previous model (anchor) $y_t^i$ and the corresponding correction $\hat{\Delta}_t^i$, but only $i\in\Scal_t$ actually participate in the optimization. In that sense, for all $i\notin\Scal_t$, $\hat{w}_t^i$ is the estimated model of client $i$ if it were to participate in round $t$.



    Let $\nabhat_t\coloneqq\frac{1}{N}\sum_{i\in\sbrac{N}}{\nabla f_i(\hat{w}_t^i)} = \bbE[\hat{g}_t]$. From the $\beta$-smoothness of the objective, 
    \begin{align}\label{eq:one_step_improvement}
        \bbE[f(w_{t+1}) - f(w_t)] &\leq -\eta\bbE[\hat{g}_t^\top\nab_t] + \frac{\beta\eta^2}{2}\bbE\norm{\hat{g}_t}^2 \nonumber \\ &= - \eta\bbE[\nabhat_t^\top\nab_t] + \frac{\beta\eta^2}{2}\bbE\norm{\hat{g}_t}^2 \nonumber \\ &= - \eta\bbE\norm{\nab_t}^2 + \underbrace{\eta\bbE[\nab_t^\top(\nab_t - \nabhat_t)]}_{=(A)} + \frac{\beta\eta^2}{2}\bbE\norm{\hat{g}_t}^2 \; ,
    \end{align}
    where the first equality follows from the law of total expectation, and the second equality from the linearity of expectation. 

    \paragraph{Bounding $(A)$: } Using the inequality $a^\top b\leq \frac{1}{2}\norm{a}^2 + \frac{1}{2}\norm{b}^2$, we get that
    \begin{equation*}
        \eta\bbE[\nab_t^\top(\nab_t - \nabhat_t)] \leq \frac{\eta}{2}\bbE\norm{\nab_t}^2 + \frac{\eta}{2}\bbE\lVert\nab_t - \nabhat_t\rVert^2\; . 
    \end{equation*}
    Focusing on the second term in the right-hand side, we have:
    \begin{align}\label{eq:grad_bias_bound_by_estimation_err}
        \bbE\lVert\nab_t - \nabhat_t\rVert^2 &= \bbE\bnorm{\frac{1}{N}\sum_{i=1}^{N}{\brac{\nabla f_i(w_t) - \nabla f_i(\hat{w}_t^i)}}}^2 \nonumber \\ &\leq \frac{1}{N}\sum_{i=1}^{N}{\bbE\norm{\nabla f_i(w_t) - \nabla f_i(\hat{w}_t^i)}^2} \nonumber \\ &\leq \frac{\beta^2}{N}\sum_{i=1}^{N}{\bbE\lVert\hat{w}_t^i - w_t\rVert^2}\; ,
    \end{align}
    where in the first inequality we used Lemma~\ref{lem:technical_norm_of_sum_to_sum_of_norms}, and the second inequality follows from the $\beta$-smoothness of each $f_i$. Plugging back this bound, we get:
    \begin{equation*}
        \eta\bbE[\nab_t^\top(\nab_t - \nabhat_t)] \leq \frac{\eta}{2}\bbE\norm{\nab_t}^2 + \frac{\beta^2\eta}{2N}\sum_{i=1}^{N}{\bbE\norm{\hat{w}_t^i - w_t}^2}\; .
    \end{equation*}
    Using Lemma~\ref{lem:second_moment_bound} to bound $\bbE\lVert\hat{g}_t\rVert^2$ and the bound on $(A)$, we get from Eq.~\eqref{eq:one_step_improvement} that
    \begin{align*}
        \bbE[f(w_{t+1}) - f(w_t)] \leq& -\frac{\eta}{2}\bbE\norm{\nab_t}^2 +  \frac{\beta^2\eta}{2N}\sum_{i=1}^{N}{\bbE\norm{\hat{w}_t^i - w_t}^2} \\ &+ \frac{\beta\eta^2}{2}\brac{\tilde{\sigma}^2_S + 4\gamma\bbE\norm{\nab_t}^2 + \frac{2\beta^2}{N}\sum_{i=1}^{N}{\bbE\norm{\hat{w}_t^i - w_t}^2}} \nonumber \\ =& \brac{-\frac{\eta}{2} + 2\gamma\beta\eta^2}\bbE\norm{\nab_t}^2 + \frac{\beta\eta^2\tilde{\sigma}^2_S}{2} \\ &+ \brac{\frac{\beta^2\eta}{2} + \beta^3\eta^2}\cdot\frac{1}{N}\sum_{i=1}^{N}{\bbE\norm{\hat{w}_t^i - w_t}^2} \; .
    \end{align*}
    Applying Lemma~\ref{lem:model_estimation_error}, we can bound $\frac{1}{N}\sum_{i=1}^{N}{\bbE\lVert\hat{w}_t^i - w_t\rVert^2}$ to get that
    {\small\begin{align}\label{eq:one_step_improvement_2}
        \bbE[f(w_{t+1}) - f(w_t)] \leq& \brac{-\frac{\eta}{2} + 2\gamma\beta\eta^2}\bbE\norm{\nab_t}^2 + \frac{\beta\eta^2\tilde{\sigma}^2_S}{2} \nonumber \\ &+ \brac{\frac{\beta^2\eta}{2} + \beta^3\eta^2}\sbrac{\alpha\brac{1 + \sum_{k=1}^{t}{k(\rho\Tcal)^{k}}}\tilde{\sigma}_S^2 + \frac{2\gamma}{\beta^2\Tcal}\sum_{k=1}^{t}{(\rho\Tcal)^k\sum_{\ell=t-k\Tcal}^{t-k}{\bbE\lVert\nab_\ell\rVert^2}}} \nonumber \\ =& \brac{-\frac{\eta}{2} + 2\gamma\beta\eta^2}\bbE\norm{\nab_t}^2 + \sbrac{\frac{\beta\eta^2}{2} + \alpha\brac{\frac{\beta^2\eta}{2} + \beta^3\eta^2}\brac{1 + \sum_{k=1}^{t}{k(\rho\Tcal)^k}}}\tilde{\sigma}_S^2 \nonumber \\ &+ \brac{\frac{\beta^2\eta}{2} + \beta^3\eta^2}\frac{2\gamma}{\beta^2\Tcal}\underbrace{\sum_{k=1}^{t}{(\rho\Tcal)^k\sum_{\ell=t-k\Tcal}^{t-k}{\bbE\lVert\nab_\ell\rVert^2}}}_{=(B)}\; ,
    \end{align}}
    
    \noindent where $\alpha=4\omega^2\eta^2\Tcal$ and $\rho=20\beta^2\omega^2\eta^2\Tcal$. Since $\eta\leq \frac{1}{30\gamma\beta(\omega\Tcal + 1)}\leq \frac{1}{\sqrt{40}\beta\omega\Tcal}$, it holds that $\rho\Tcal = 20\beta^2\omega^2\Tcal^2\eta^2 \leq 1/2 < 1$, and thus we can bound the coefficient of $\tilde{\sigma}_S^2$ using Lemma~\ref{lem:term_by_term_differentiation} as 
    \begin{equation}\label{eq:linear_times_geomteric_series}
        \sum_{k=1}^{t}{k(\rho\Tcal)^k} \leq \sum_{k=1}^{\infty}{k(\rho\Tcal)^k} = \frac{\rho\Tcal}{(1 - \rho\Tcal)^2} \leq 4\rho\Tcal\leq 2\; ,
    \end{equation}
    where we used $\frac{1}{(1-\rho\Tcal)^2}\leq 4$, and $\rho\Tcal\leq 1/2$. 
    \paragraph{Bounding $(B)$: } To bound $(B)$, we change the summation order. Consider a fixed $\ell\in\naturals$. Note that $(\rho\Tcal)^k$ appears as a coefficient of $\bbE\norm{\nab_\ell}^2$ if and only if $t-k\Tcal\leq \ell\leq t-k$, which is equivalent to $\frac{t-\ell}{\Tcal}\leq k\leq t-\ell$. Therefore, we have 
    \begin{align*}
        \sum_{k=1}^{t}{(\rho\Tcal)^k\sum_{\ell=t-k\Tcal}^{t-k}{\bbE\norm{\nab_\ell}^2}} &= \sum_{\ell=0}^{t-1}{\brac{\sum_{k=\ceil{\frac{t-\ell}{\Tcal}}}^{t-\ell}{(\rho\Tcal)^k}}\bbE\norm{\nab_\ell}^2} \\ &\leq \sum_{\ell=0}^{t-1}{\brac{\sum_{k=\ceil{\frac{t-\ell}{\Tcal}}}^{\infty}{(\rho\Tcal)^k}}\bbE\norm{\nab_\ell}^2} \\ &= \frac{1}{1-\rho\Tcal}\sum_{\ell=0}^{t-1}{(\rho\Tcal)^{\ceil{\frac{t-\ell}{\Tcal}}}\bbE\norm{\nab_{\ell}}^2}
    \end{align*}
    Plugging this bound and Eq.~\eqref{eq:linear_times_geomteric_series} back to Eq.~\eqref{eq:one_step_improvement_2} gives
    \begin{align}
        \bbE[f(w_{t+1}) - f(w_t)] &\leq \brac{-\frac{\eta}{2} + 2\gamma\beta\eta^2}\bbE\norm{\nab_t}^2 + \brac{\frac{\beta\eta^2}{2} + 3\alpha\brac{\frac{\beta^2\eta}{2} + \beta^3\eta^2}}\tilde{\sigma}^2_S \nonumber \\ &\quad+ \brac{\eta + 2\beta\eta^2}\frac{\gamma}{(1-\rho\Tcal)\Tcal}\sum_{k=0}^{t-1}{(\rho\Tcal)^{\ceil{\frac{t-k}{\Tcal}}}\bbE\norm{\nab_k}^2} \nonumber \; .
    \end{align}
    Summing over $t=0,\ldots,T-1$, we obtain
    \begin{align}\label{eq:sum_over_T_err}
        \bbE[f(w_{T}) - f(w_0)] =& \sum_{t=0}^{T-1}{\bbE[f(w_{t+1}) - f(w_t)]} \nonumber \\ \leq& \brac{-\frac{\eta}{2} + 2\gamma\beta\eta^2}\sum_{t=0}^{T-1}{\bbE\norm{\nab_t}^2} + \brac{\frac{\beta\eta^2}{2} + 3\alpha\brac{\frac{\beta^2\eta}{2} + \beta^3\eta^2}}T\tilde{\sigma}^2_S \nonumber \\ &+ \brac{\eta + 2\beta\eta^2}\frac{\gamma}{(1-\rho\Tcal)\Tcal}\underbrace{\sum_{t=0}^{T-1}{\sum_{k=0}^{t-1}{(\rho\Tcal)^{\ceil{\frac{t-k}{\Tcal}}}\bbE\norm{\nab_k}^2}}}_{=(C)}\; .
    \end{align}
    Focusing on $(C)$, we can change the outer summation bounds as 
    \begin{equation}\label{eq:double_sum_to_be_bounded_by_lemma}
        \sum_{t=0}^{T-1}{\sum_{k=0}^{t-1}{(\rho\Tcal)^{\ceil{\frac{t-k}{\Tcal}}}\bbE\norm{\nab_k}^2}} = \sum_{t=1}^{T-1}{\sum_{k=0}^{t-1}{(\rho\Tcal)^{\ceil{\frac{t-k}{\Tcal}}}\bbE\norm{\nab_k}^2}} \leq \sum_{t=1}^{T}{\sum_{k=0}^{t-1}{(\rho\Tcal)^{\ceil{\frac{t-k}{\Tcal}}}\bbE\norm{\nab_k}^2}}\; .
    \end{equation}
    Now, we can bound the right-hand side using Lemma~\ref{lem:double_sum_of_weighted_grads_to_single_sum} with $a=\rho\Tcal<1$ and $x_k = \bbE\norm{\nab_k}^2\geq 0$ to get that
    \begin{equation*}
        \sum_{t=1}^{T}{\sum_{k=0}^{t-1}{(\rho\Tcal)^{\ceil{\frac{t-k}{\Tcal}}}\bbE\norm{\nab_k}^2}} \leq \Tcal\frac{\rho\Tcal}{1-\rho\Tcal}\sum_{t=0}^{T-1}{\bbE\norm{\nab_t}^2}\; .
    \end{equation*}
    Plugging this bound back to Eq.~\eqref{eq:sum_over_T_err} and using $\frac{1}{(1-\rho\Tcal)^2}\leq 4$ gives
    \begin{align*}
        \bbE[f(w_T) - f(w_0)] &\leq \brac{-\frac{\eta}{2} + 2\gamma\beta\eta^2}\sum_{t=0}^{T-1}{\bbE\norm{\nab_t}^2} + \brac{\frac{\beta\eta^2}{2} + 3\alpha\brac{\frac{\beta^2\eta}{2} + \beta^3\eta^2}}T\tilde{\sigma}^2_S \\ &\quad+ \brac{\eta + 2\beta\eta^2}\frac{\gamma\rho\Tcal}{(1-\rho\Tcal)^2}\sum_{t=0}^{T-1}{\bbE\norm{\nab_t}^2} \\ &\leq \brac{-\frac{\eta}{2} + 2\gamma\beta\eta^2+4\gamma\rho\Tcal\brac{\eta + 2\beta\eta^2}}\sum_{t=0}^{T-1}{\bbE\norm{\nab_t}^2} \\ &\quad+ \brac{\frac{\beta\eta^2}{2} + \frac{3\alpha\beta^2}{2}\brac{\eta + 2\beta\eta^2}}T\tilde{\sigma}^2_S \\ &\leq \brac{-\frac{\eta}{2} + 2\gamma\beta\eta^2+8\gamma\rho\Tcal\eta}\sum_{t=0}^{T-1}{\bbE\norm{\nab_t}^2} + \brac{\frac{\beta\eta^2}{2} + 3\alpha\beta^2\eta}T\tilde{\sigma}^2_S \; ,
    \end{align*}
    where in the last inequality we used the fact that $\eta\leq\frac{1}{30\gamma\beta\theta}\leq\frac{1}{30\beta}$ to bound $2\beta\eta^2\leq \eta$.

    Substituting $\alpha$ and $\rho$, we obtain:
    \begin{align*}
        \bbE[f(w_T) - f(w_0)] &\leq \brac{-\frac{\eta}{2} + 2\gamma\beta\eta^2+160\gamma\beta^2\omega^2\Tcal^2\eta^3}\sum_{t=0}^{T-1}{\bbE\norm{\nab_t}^2} \\ &\quad+ \brac{\frac{\beta\eta^2}{2} + 12\beta^2\omega^2\Tcal\eta^3} T\tilde{\sigma}^2_S\; .
    \end{align*}
    Since $\eta\leq\frac{1}{30\gamma\beta\theta}$, we can bound the coefficient of $\sum_{t=0}^{T-1}{\bbE\norm{\nab_t}^2}$ using Lemma~\ref{lem:sum_of_grads_coeff_bound}. We get:
    \begin{equation}\label{eq:final_bound_etas}
        \bbE[f(w_{T}) - f(w_0)] \leq -\frac{\eta}{4}\sum_{t=0}^{T-1}{\bbE\norm{\nab_t}^2} + \brac{\frac{\beta\eta^2}{2} + 12\beta^2\omega^2\Tcal\eta^3} T\tilde{\sigma}^2_S\; .
    \end{equation}
    Rearranging terms, multiplying by $4/\eta T$, and plugging $\tilde{\sigma}_S^2 = \tilde{\sigma}^2/S$ then gives
    \begin{align*}
        \bbE\sbrac{\frac{1}{T}\sum_{t=0}^{T-1}{\norm{\nab_t}^2}} &\leq \frac{4M}{\eta T} + \frac{2\beta\tilde{\sigma}^2}{S} \eta + \frac{48\beta^2\omega^2\Tcal\tilde{\sigma}^2}{S}\eta^2 \; ,
    \end{align*}
    where we also used $\bbE[f(w_0) - f(w_T)]\leq f(w_0) - f^*\leq M$. Applying Lemma~\ref{lem:learning_rate_bound} with our learning rate $\eta$, we finally obtain:
    \begin{align*}
        \bbE\sbrac{\frac{1}{T}\sum_{t=0}^{T-1}{\norm{\nab_t}^2}} \leq& \frac{4M}{T}\brac{30\gamma\beta\theta + \sqrt{\frac{\beta\tilde{\sigma}^2 T}{2MS}} + \brac{\frac{12\beta^2\omega^2\Tcal\tilde{\sigma}^2 T}{MS}}^{1/3}} + \frac{2\beta\tilde{\sigma}^2}{S}\cdot\sqrt{\frac{2MS}{\beta\tilde{\sigma}^2 T}} \\ &+ \frac{48\beta^2\omega^2\Tcal\tilde{\sigma}^2}{S}\cdot\brac{\frac{MS}{12\beta^2\omega^2\Tcal\tilde{\sigma}^2 T}}^{2/3} \\ =& 4\sqrt{\frac{2M\beta\tilde{\sigma}^2}{TS}} + 8\frac{(12M^2 \beta^2 \omega^2\Tcal\tilde{\sigma}^2)^{1/3}}{T^{2/3}S^{1/3}} + \frac{120\gamma\beta\theta}{T} \; ,
    \end{align*}
    which concludes the proof.
   
\end{proof} 

\subsection{Technical Lemmata}
In this section, we introduce some technical results used throughout our analysis. We start with the following lemma, yielding a bound on the second moment of the aggregated gradients that our PS uses to update its model. 

\setcounter{theorem}{0}

\begin{Lem}\label{lem:second_moment_bound}
    Consider the notations of \cref{thm:general_docoFL}. For every $t\in\sbrac{T}$, it holds that
    \begin{equation}
        \bbE\norm{\hat{g}_t}^2 \leq \tilde{\sigma}^2_S + 4\gamma\bbE\norm{\nab_t}^2 + 2\beta^2\cdot\frac{1}{N}\sum_{i=1}^{N}{\bbE\norm{\hat{w}_t^i - w_t}^2}\; .
    \end{equation}
\end{Lem}
\begin{proof}
    Since $\hat{g}_t$ is an aggregation of local gradient, we can write,  
    \begin{align}
        \bbE\norm{\hat{g}_t}^2 &= \bbE\bnorm{\frac{1}{S}\sum_{i\in\Scal_t}{\hat{g}_t^i}}^2 = \bbE\bnorm{\frac{1}{S}\sum_{i\in\Scal_t}{\brac{\hat{g}_t^i - \nabla f_i(\hat{w}_t^i) + \nabla f_i(\hat{w}_t^i)}}}^2 \nonumber \\ &= \bbE\bnorm{\frac{1}{S}\sum_{i\in\Scal_t}{\brac{\hat{g}_t^i - \nabla f_i(\hat{w}_t^i)}}}^2 + \bbE\bnorm{\frac{1}{S}\sum_{i\in\Scal_t}{\nabla f_i(\hat{w}_t^i)}}^2 \nonumber\; ,
    \end{align}
    where the last equality follows from Assumption~\ref{assump:stochastic_grad_oracle} as $\bbE[(\hat{g}_t^i(\hat{w}_t^i) - \nabla f_i(\hat{w}_t^i))^\top \nabla f_i(\hat{w}_t^i)] = 0$. Note that the first term in the right-hand side is the variance of the average of $S$ independent random variables with zero mean and variance bounded by $\sigma^2$; therefore, it is bounded by $\sigma^2/S$. Thus, we get that
    \begin{equation}\label{eq:main_proof_second_moment}
        \bbE\norm{\hat{g}_t}^2 \leq \frac{\sigma^2}{S} + \bbE\bnorm{\frac{1}{S}\sum_{i\in\Scal_t}{\nabla f_i(\hat{w}_t^i)}}^2\; .
    \end{equation}
    Focusing on the second term in the right-hand side, we have that
    \begin{equation}
        \bbE\bnorm{\frac{1}{S}\sum_{i\in\Scal_t}{\nabla f_i(\hat{w}_t^i)}}^2 \leq \underbrace{2\bbE\bnorm{\frac{1}{S}\sum_{i\in\Scal_t}{\brac{\nabla f_i(\hat{w}_t^i) - \nabla f_i(w_t)}}}^2}_{(A)} + \underbrace{2\bbE\bnorm{\frac{1}{S}\sum_{i\in\Scal_t}{\nabla f_i(w_t)}}^2}_{=(B)}\; ,
    \end{equation}
    where we used the inequality $\norm{a + b}^2 \leq 2\norm{a}^2 + 2\norm{b}^2$.

    \paragraph{Bounding $(A)$: } Using Lemma~\ref{lem:technical_norm_of_sum_to_sum_of_norms} and the $\beta$-smoothness of the objective, we get that 
    \begin{align}
        2\bbE\bnorm{\frac{1}{S}\sum_{i\in\Scal_t}{\brac{\nabla f_i(\hat{w}_t^i) - \nabla f_i(w_t)}}}^2 &\leq \frac{2}{S}\bbE\sbrac{\sum_{i\in\Scal_t}{\norm{\nabla f_i(\hat{w}_t^i) - \nabla f_i(w_t)}^2}} \nonumber \\ &\leq \frac{2\beta^2}{S}\bbE\sbrac{\sum_{i\in\Scal_t}{\norm{\hat{w}_t^i - w_t}^2}} \nonumber \\ &= \frac{2\beta^2}{S}\bbE\sbrac{\sum_{i=1}^{N}{\norm{\hat{w}_t^i - w_t}^2}\cdot\mathbbm{1}_{\cbrac{i\in\Scal_t}}}\nonumber\\ &= \frac{2\beta^2}{N}\sum_{i=1}^{N}{\bbE\norm{\hat{w}_t^i - w_t}^2} \nonumber\; ,
    \end{align}
    where the last equality follows from our assumption about the client participation process $\P(\cdot)$, which guarantees that $\prob(i\in\Scal_t) = S/N$, independently of the optimization process.

    \paragraph{Bounding $(B)$: } By the law of total expectation, $(B)$ can be written as follows,
    \begin{equation}\label{eq:main_proof_bounding_B}
        2\bbE\bnorm{\frac{1}{S}\sum_{i\in\Scal_t}{\nabla f_i(w_t)}}^2 = 2\bbE\bnorm{\sum_{i=1}^{N}{\brac{\frac{1}{S}\nabla f_i(w_t)\cdot\mathbbm{1}_{\cbrac{i\in\Scal_t}}}}}^2 = 2\bbE\sbrac{\bbE\sbrac{\bnorm{\sum_{i=1}^{N}{\brac{\frac{1}{S}\nabla f_i(w_t)\cdot\mathbbm{1}_{\cbrac{i\in\Scal_t}}}}}^2~\vrule~w_t}}\; .
    \end{equation}
    Thus, we can use Lemma~\ref{lem:technical_norm_of_sum_of_iid} with $X_i = \frac{1}{S}\nabla f_i(w_t)\cdot\mathbbm{1}_{\cbrac{i\in\Scal_t}},~ i\in\sbrac{N}$ to bound the inner expectation. Using $\prob(i\in\Scal_t) = S/N$, we have that
    \begin{equation*}
        \bbE[X_i|w_t] = \frac{1}{S}\nabla f_i(w_t)\cdot\frac{S}{N} = \frac{1}{N}\nabla f_i(w_t), 
    \end{equation*}
    and, 
    \begin{align*}
        \bbE[\norm{X_i - \bbE[X_i|w_t]}^2|w_t] &= \norm{\nabla f_i(w_t)}^2 \cdot\bbE\sbrac{\brac{\frac{1}{S}\cdot\mathbbm{1_{\cbrac{i\in\Scal_t}}} - \frac{1}{N}}^2} =  \norm{\nabla f(w_t)}^2\cdot\text{Var}\brac{\frac{1}{S}\cdot\mathbbm{1}_{\cbrac{i\in\Scal_t}}} \\ &=  \frac{\norm{\nabla f(w_t)}^2}{S^2}\cdot \frac{S}{N}\brac{1 - \frac{S}{N}} = \frac{\norm{\nabla f_i(w_t)}^2}{SN}\brac{1 - \frac{S}{N}}\; ,
    \end{align*}
    where we used the fact that for any event $\A$, the following holds: $\text{Var}(\mathbbm{1}_{\A}) = \prob(\A)\cdot(1 - \prob(\A))$. Therefore, using Lemma~\ref{lem:technical_norm_of_sum_of_iid}, we obtain that 
    \begin{align*}
        \bbE\sbrac{\bnorm{\sum_{i=1}^{N}{\brac{\frac{1}{S}\nabla f_i(w_t)\cdot\mathbbm{1}_{\cbrac{i\in\Scal_t}}}}}^2~\vrule~w_t} &\leq 2\bnorm{\frac{1}{N}\sum_{i=1}^{N}{\nabla f_i(w_t)}}^2 + \frac{2}{SN}\brac{1 - \frac{S}{N}}\sum_{i=1}^{N}{\norm{\nabla f_i(w_t)}^2} \\ &\leq 2\norm{\nab_t}^2 + \frac{2}{S}\brac{1-\frac{S}{N}}\brac{G^2 + B^2\norm{\nab_t}^2} \\ &= \brac{1 - \frac{S}{N}}\frac{2G^2}{S} + 2\underbrace{\brac{1 + \brac{1 - \frac{S}{N}}\frac{B^2}{S}}}_{\coloneqq\gamma}\norm{\nab_t}^2 \; ,
    \end{align*}
    where in the second inequality we used the bounded gradient dissimilarity assumption (Assumption~\ref{assump:bgd}). Plugging back to Eq.~\eqref{eq:main_proof_bounding_B}, we get the following bound on $(B)$:
    \begin{equation*}
        2\bbE\bnorm{\frac{1}{S}\sum_{i\in\Scal_t}{\nabla f_i(w_t)}}^2 \leq \brac{1 - \frac{S}{N}}\frac{4G^2}{S} + 4\gamma\bbE\norm{\nab_t}^2\; .
    \end{equation*}
    
    Plugging the bounds on $(A)$ and $(B)$ in Eq.~\eqref{eq:main_proof_second_moment} finally gives
    \begin{align*}
        \bbE\norm{\hat{g}_t}^2 &\leq \underbrace{\frac{\sigma^2}{S} + \brac{1 - \frac{S}{N}}\frac{4G^2}{S}}_{=\tilde{\sigma}^2/S } + 4\gamma\bbE\norm{\nab_t}^2 + \frac{2\beta^2}{N}\sum_{i=1}^{N}{\bbE\norm{\hat{w}_t^i - w_t}^2} \\ &= \tilde{\sigma}^2_S + 4\gamma\bbE\norm{\nab_t}^2 + \frac{2\beta^2}{N}\sum_{i=1}^{N}{\bbE\norm{\hat{w}_t^i - w_t}^2}\; ,
    \end{align*}
    which concludes the proof. 
\end{proof}

The next result establishes a bound on the model estimation error of the clients.
\begin{Lem}\label{lem:model_estimation_error}
    Consider the notations of Theorem~\ref{thm:general_docoFL}. Let $\alpha\coloneqq 4\omega^2\eta^2\Tcal, \rho\coloneqq 20\beta^2\omega^2\eta^2\Tcal$, and $\nabla_{-\ell}\coloneqq 0, \enskip\forall \ell\in\naturals$. Then, the following result holds:
    \begin{equation*}
        \frac{1}{N}\sum_{i=1}^{N}{\bbE\lVert\hat{w}_t^i - w_t\rVert^2} \leq \alpha\brac{1 + \sum_{k=1}^{t}{k(\rho\Tcal)^{k}}}\tilde{\sigma}_S^2 + \frac{2\gamma}{\beta^2\Tcal}\sum_{k=1}^{t}{(\rho\Tcal)^k\sum_{\ell=t-k\Tcal}^{t-k}{\bbE\lVert\nab_\ell\rVert^2}}
    \end{equation*}
\end{Lem}
\begin{proof}
    We use strong induction. Particularly, to prove the result holds at round $t$, we rely on its correctness over the $\Tcal$ prior rounds, i.e., for every $s=t-\Tcal, \ldots,t-1$. Thus, in our base case, we show that the result holds up to round $\Tcal$. \\
    We start with some general observations that hold for any $t$.  Recall that $\hat{w}_t^i = y_{t}^i + 
    \C(w_t - y_t^i)$. From Assumption~\ref{assump:bounded_nmse}, we have
    \begin{equation}\label{eq:model_estimation_err}
        \bbE\norm{\hat{w}_t^i - w_t}^2 = \bbE\norm{\C_c(w_t - y_t^i) - (w_t - y_t^i)}^2 \leq \omega^2\bbE\norm{w_t - y_t^i}^2 \; .
    \end{equation}
    Unrolling the update rule for $w_t$, we have for all $i\in\sbrac{N}$ that 
    \begin{equation*}
        w_{t} = w_{t -\tau_t^i} - \eta\sum_{k=t-\tau_t^i}^{t-1}{\hat{g}_k} = y_t^i - \eta\sum_{k=t-\tau_t^i}^{t-1}{\hat{g}_k} \; .
    \end{equation*}
    Let $\hat{g}_{-k}\coloneqq 0$ for all $k\in\naturals$. Additionally, let $\hat{\xi}_k = \hat{g}_k - \nabhat_k$ for all $k$, where $\nabhat_k = \bbE[\hat{g}_k]$, as defined in the proof of Theorem~\ref{thm:general_docoFL}. Plugging back to Eq.~\eqref{eq:model_estimation_err}, we get that
    \begin{align}\label{eq:temp_bound_client_estimation_err}
        \bbE\norm{\hat{w}_t^i - w_t}^2 &\leq \omega^2\eta^2\bbE\norm{\sum_{k=t-\tau_t^i}^{t-1}{\hat{g}_k}}^2 \leq 2\omega^2\eta^2\bbE\norm{\sum_{k=t-\tau_t^i}^{t-1}{\nabhat_k}}^2 + 2\omega^2\eta^2\bbE\norm{\sum_{k=t-\tau_t^i}^{t-1}{\hat{\xi}_k}}^2\; ,
    \end{align}
    where the last inequality follows from $\norm{a+b}^2\leq 2\norm{a}^2 + 2\norm{b}^2$. Using Lemma~\ref{lem:technical_norm_of_sum_to_sum_of_norms}, we can bound the first term in the right-hand side, as 
    \begin{align*}
        \bbE\norm{\sum_{k=t-\tau_t^i}^{t-1}{\nabhat_k}}^2 \leq \tau_t^i\sum_{k=t-\tau_t^i}^{t-1}{\bbE\lVert\nabhat_k\rVert^2} \leq \Tcal\sum_{k=t-\Tcal}^{t-1}{\bbE\lVert\nabhat_k\rVert^2} \; ,
    \end{align*}
    where the last inequality follows from $\tau_t^i\leq\Tcal$. Since $\bbE[\hat{\xi}_k] = 0$, and $\bbE[\hat{\xi}_k^\top\hat{\xi}_\ell] = 0$ for all $k,\ell$, we can apply Lemma~\ref{lem:sqr_norm_of_sum_to_sum_of_sqr_norm} to bound the second term in the right-hand side as follows:
    \begin{align*}
        \bbE\norm{\sum_{k=t-\tau_t^i}^{t-1}{\hat{\xi}_k}}^2 &= \sum_{k=t-\tau_t^i}^{t-1}{\bbE\lVert \hat{\xi}_k\rVert^2} \leq \sum_{k=t-\Tcal}^{t-1}{\bbE\lVert \hat{\xi}_k\rVert^2} \leq 2\sum_{k=t-\Tcal}^{t-1}{\bbE\lVert \nabhat_k\rVert^2} + 2\sum_{k=t-\Tcal}^{t-1}{\bbE\lVert \hat{g}_k\rVert^2}\; ,
    \end{align*} 
    where we used $\tau_t^i\leq \Tcal$, and $\norm{a-b}^2 \leq 2\norm{a}^2 + 2\norm{b}^2$. 
    
    Plugging-in both bounds to Eq.~\eqref{eq:temp_bound_client_estimation_err}, we obtain:
    \begin{align}\label{eq:model_estimate_err_sum_of_grads_pre}
        \bbE\norm{\hat{w}_t^i - w_t}^2  &\leq (2\omega^2\eta^2\Tcal + 4\omega^2\eta^2)\sum_{k=t-\Tcal}^{t-1}{\bbE\lVert\nabhat_k\rVert^2} + 4\omega^2\eta^2\sum_{k=t-\Tcal}^{t-1}{\bbE\lVert \hat{g}_k\rVert^2} \nonumber \\ &\leq 6\omega^2\eta^2\Tcal\sum_{k=t-\Tcal}^{t-1}{\bbE\lVert\nabhat_k\rVert^2} + 4\omega^2\eta^2\sum_{k=t-\Tcal}^{t-1}{\bbE\lVert \hat{g}_k\rVert^2} \; .
    \end{align}
    Note that we can bound $\bbE\lVert\nabhat_k\rVert^2$ as: 
    \begin{align*}
        \bbE\lVert\nabhat_k\rVert^2 &\leq 2\bbE\lVert\nabhat_k - \nab_k\rVert^2 + 2\bbE\lVert\nab_k\rVert^2 \leq \frac{2\beta^2}{N}\sum_{i=1}^{N}{\bbE\lVert\hat{w}_k^i - w_k\rVert^2} + 2\bbE\lVert\nab_k\rVert^2\; ,
    \end{align*}
    where in the last inequality we used Eq.~\eqref{eq:grad_bias_bound_by_estimation_err} to bound $\bbE\lVert\nabhat_k - \nab_k\rVert^2$. 
    
    For the ease of notation, denote: $e_t^i\coloneqq \bbE\lVert\hat{w}_t^i - w_t\rVert^2$, and $e_t\coloneqq\frac{1}{N}\sum_{i=1}^{N}{e_t^i}$. 
    Therefore, we obtain from Eq.~\eqref{eq:model_estimate_err_sum_of_grads_pre} that
    \begin{align}\label{eq:model_estimate_err_sum_of_grads}
        e_t^i &\leq 12\beta^2\omega^2\eta^2\Tcal\sum_{k=t-\Tcal}^{t-1}{e_k} + 12\omega^2\eta^2\Tcal\sum_{k=t-\Tcal}^{t-1}{\bbE\norm{\nab_k}^2} + 4\omega^2\eta^2\sum_{k=t-\Tcal}^{t-1}{\bbE\lVert\hat{g}_k\rVert^2}\; .
    \end{align}
    
    \textbf{Base Case: } For $t=0$, each client obtains the exact model weights, i.e., $\hat{w}_0^i = w_0$, which trivially implies result. For every $t=1,\ldots,\Tcal$ and $i\in\sbrac{N}$, we have from Eq.~\eqref{eq:model_estimate_err_sum_of_grads} that
    \begin{equation}\label{eq:model_estimate_err_base_case}
        e_t^i \leq 12\beta^2\omega^2\eta^2\Tcal\sum_{k=0}^{t-1}{e_k} + 12\omega^2\eta^2\Tcal\sum_{k=0}^{t-1}{\bbE\norm{\nab_k}^2} + 4\omega^2\eta^2\sum_{k=0}^{t-1}{\bbE\lVert\hat{g}_k\rVert^2}\; .
    \end{equation}
    Using Lemma~\ref{lem:second_moment_bound} to bound $\bbE\norm{\hat{g}_k}^2$, we get:
    \begin{align*}
        e_t^i &\leq 12\beta^2\omega^2\eta^2\Tcal\sum_{k=0}^{t-1}{e_k} + 12\omega^2\eta^2\Tcal\sum_{k=0}^{t-1}{\bbE\norm{\nab_k}^2} + 4\omega^2\eta^2\sum_{k=0}^{t-1}{\brac{\tilde{\sigma}_S^2 + 4\gamma\bbE\lVert\nab_k\rVert^2 + 2\beta^2 e_k}} \\ &\leq 4\omega^2\eta^2\Tcal\tilde{\sigma}_S^2 + \brac{12\omega^2\eta^2\Tcal + 16\gamma\omega^2\eta^2}\sum_{k=0}^{t-1}{\bbE\lVert\nab_k\rVert^2} + \brac{12\beta^2\omega^2\eta^2\Tcal + 8\beta^2\omega^2\eta^2}\sum_{k=0}^{t-1}{e_k} \\ &\leq 4\omega^2\eta^2\Tcal\tilde{\sigma}_S^2 + 28\gamma\omega^2\eta^2\Tcal\sum_{k=0}^{t-1}{\bbE\lVert\nab_k\rVert^2} + 20\beta^2\omega^2\eta^2\Tcal\sum_{k=0}^{t-1}{e_k}\; .
    \end{align*}
    Note that this bound on $e_t^i$ is independent of $i$, and thus, it holds for the average of $e_t^i$ over $i\in\sbrac{N}$, namely, $e_t$. Therefore, Eq.~\eqref{eq:model_estimate_err_sum_of_grads} implies a recursive bound on $e_t$; for every $t=1,\ldots,\Tcal$:
    \begin{align}\label{eq:recursion}
        e_t \leq \alpha\tilde{\sigma}_S^2 + \nu\sum_{k=0}^{t-1}{\bbE\lVert\nab_k\rVert^2} + \rho\sum_{k=0}^{t-1}{e_k}\; ,
    \end{align}
    where we denoted $\nu\coloneqq 28\gamma\omega^2\eta^2\Tcal$. 
    Plugging-in this bound instead of $e_k$ in the right-hand side, we obtain:
    \begin{align*}
        e_t &\leq \alpha\tilde{\sigma}_S^2 + \nu\sum_{k=0}^{t-1}{\bbE\lVert\nab_k\rVert^2} + \rho\sum_{k=0}^{t-1}{\brac{\alpha\tilde{\sigma}_S^2 + \nu\sum_{\ell=0}^{k-1}{\bbE\lVert\nab_\ell\rVert^2} + \rho\sum_{\ell=0}^{k-1}{e_\ell}}} \\ &\leq \alpha\brac{1 + \rho\Tcal}\tilde{\sigma}_S^2 + \nu\sum_{k=0}^{t-1}{\bbE\lVert\nab_k\rVert^2} + \nu\rho\sum_{k=0}^{t-1}{\sum_{\ell=0}^{k-1}{\bbE\lVert\nab_\ell\rVert^2}} + \rho^2\sum_{k=0}^{t-1}{\sum_{\ell=0}^{k-1}{e_\ell}}\; ,
    \end{align*}
    where we used $t\leq\Tcal$. Note that we can bound the double sums in right-hand side using Lemma~\ref{lem:double_sum_to_single_sum} as
    \begin{equation*}
        \sum_{k=0}^{t-1}{\sum_{\ell=0}^{k-1}{\bbE\norm{\nab_\ell}^2}} \leq t\sum_{k=0}^{t-2}{\bbE\norm{\nab_k}^2} \leq \Tcal\sum_{k=0}^{t-2}{\bbE\norm{\nab_k}^2}\; ,
    \end{equation*}
    and similarly, 
    \begin{equation*}
        \sum_{k=0}^{t-1}{\sum_{\ell=0}^{k-1}{e_\ell}} \leq \Tcal\sum_{k=0}^{t-2}{e_k}\; .
    \end{equation*}
    Plugging-back, we get:
    \begin{align*}
        e_t \leq \alpha\brac{1 + \rho\Tcal}\tilde{\sigma}_S^2 + \nu\sum_{k=0}^{t-1}{\bbE\lVert\nab_k\rVert^2} + \nu\rho\Tcal\sum_{k=0}^{t-2}{\bbE\lVert\nab_\ell\rVert^2} + \rho^2\Tcal\sum_{k=0}^{t-2}{e_k}\; .
    \end{align*}
    We can once again apply Eq.~\eqref{eq:recursion} to bound $e_k$, and obtain:
    \begin{align*}
        e_t \leq&  \alpha\brac{1 + \rho\Tcal}\tilde{\sigma}_S^2 + \nu\sum_{k=0}^{t-1}{\bbE\lVert\nab_k\rVert^2} + \nu\rho\Tcal\sum_{k=0}^{t-2}{\bbE\lVert\nab_\ell\rVert^2} + \rho^2\Tcal\sum_{k=0}^{t-2}{\brac{\alpha\tilde{\sigma}_S^2 + \nu\sum_{\ell=0}^{k-1}{\bbE\lVert\nab_\ell\rVert^2} + \rho\sum_{\ell=0}^{k-1}{e_\ell}}} \\ \leq& \alpha\brac{1 + \rho\Tcal + \rho^2\Tcal^2}\tilde{\sigma}_S^2 + \nu\sum_{k=0}^{t-1}{\bbE\lVert\nab_k\rVert^2} + \nu\rho\Tcal\sum_{k=0}^{t-2}{\bbE\lVert\nab_\ell\rVert^2} + \nu\rho^2\Tcal\sum_{k=0}^{t-2}{\sum_{\ell=0}^{k-1}{\bbE\lVert\nab_\ell\rVert^2}} + \rho^3\Tcal\sum_{k=0}^{t-2}{\sum_{\ell=0}^{k-1}{e_\ell}}\\ \leq& \alpha\brac{1{+}\rho\Tcal {+}\rho^2\Tcal^2}\tilde{\sigma}_S^2{+} \nu\sum_{k=0}^{t-1}{\bbE\lVert\nab_k\rVert^2} {+} \nu\rho\Tcal\sum_{k=0}^{t-2}{\bbE\lVert\nab_\ell\rVert^2}{+} \nu\rho^2\Tcal^2\sum_{k=0}^{t-3}{\bbE\lVert\nab_\ell\rVert^2} + \rho^3\Tcal^2\sum_{k=0}^{t-3}{e_k}\; ,
    \end{align*}
    where in the last inequality we used Lemma~\ref{lem:double_sum_to_single_sum}. Repeating this process of alternately applying Eq.~\eqref{eq:recursion} to bound $e_k$ and Lemma~\ref{lem:double_sum_to_single_sum}, finally gives:
    \begin{align*}
        e_t \leq \alpha\brac{1 + \sum_{k=1}^{t}{(\rho\Tcal)^{k}}}\tilde{\sigma}_S^2 + \frac{\nu}{\rho\Tcal}\sum_{k=1}^{t}{(\rho\Tcal)^k\sum_{\ell=0}^{t-k}{\bbE\lVert\nab_\ell\rVert^2}}\; .
    \end{align*}
    Plugging $\nu$ and $\rho$, we can bound the coefficient $\nu/\rho\Tcal$ as:
    \begin{align*}
        \frac{\nu}{\rho\Tcal} = \frac{28\gamma\omega^2\eta^2\Tcal}{20\beta^2\omega^2\eta^2\Tcal^2} \leq \frac{2\gamma}{\beta^2\Tcal}\; .
    \end{align*}
    Using $(\rho\Tcal)^k\leq k(\rho\Tcal)^k$, which holds for any $k\geq 1$, we then obtain:
    \begin{align*}
        e_t &\leq \alpha\brac{1 + \sum_{k=1}^{t}{k(\rho\Tcal)^{k}}}\tilde{\sigma}_S^2 + \frac{2\gamma}{\beta^2\Tcal}\sum_{k=1}^{t}{(\rho\Tcal)^k\sum_{\ell=0}^{t-k}{\bbE\lVert\nab_\ell\rVert^2}}\; .
    \end{align*}
    Note that for all $t\leq\Tcal$ and $k\geq 1$, we have $t-k\Tcal\leq 0$. Therefore, since for $\nabla_{-\ell}=0$ for all $\ell\in\naturals$, we can equivalently write:
    \begin{align*}
        e_t &\leq \alpha\brac{1 + \sum_{k=1}^{t}{k(\rho\Tcal)^{k}}}\tilde{\sigma}_S^2 + \frac{2\gamma}{\beta^2\Tcal}\sum_{k=1}^{t}{(\rho\Tcal)^k\sum_{\ell=t-k\Tcal}^{t-k}{\bbE\lVert\nab_\ell\rVert^2}}\; ,
    \end{align*}
    which establishes the result for the base case. 

    \textbf{Induction step: } The induction hypothesis is that the following holds:
    \begin{equation}\label{eq:induction_hypothesis}
        e_s \leq \alpha\brac{1 + \sum_{k=1}^{s}{k(\rho\Tcal)^{k}}}\tilde{\sigma}_S^2 + \frac{2\gamma}{\beta^2\Tcal}\sum_{k=1}^{s}{(\rho\Tcal)^k\sum_{\ell=s-k\Tcal}^{s-k}{\bbE\lVert\nab_\ell\rVert^2}}, \quad\forall s=t-\Tcal,\ldots, t-1 \; .
    \end{equation}
    We focus on Eq.~\eqref{eq:model_estimate_err_sum_of_grads}. Using Lemma~\ref{lem:second_moment_bound} to bound $\bbE\norm{\hat{g}_k}^2$ and following similar steps to those used to derive Eq.~\eqref{eq:recursion}, we get:
    \begin{align}\label{eq:client_estimation_error}
        e_t^i &\leq 12\beta^2\omega^2\eta^2\Tcal\sum_{k=t-\Tcal}^{t-1}{e_k} + 12\omega^2\eta^2\Tcal\sum_{k=t-\Tcal}^{t-1}{\bbE\norm{\nab_k}^2} + 4\omega^2\eta^2\sum_{k=t-\Tcal}^{t-1}{\brac{\tilde{\sigma}_S^2 + 4\gamma\bbE\lVert\nab_k\rVert^2 + 2\beta^2 e_k}} \nonumber \\ &\leq \alpha\tilde{\sigma}_S^2 + \nu\sum_{k=t-\Tcal}^{t-1}{\bbE\lVert\nab_k\rVert^2} + \rho\underbrace{\sum_{k=t-\Tcal}^{t-1}{e_k}}_{=(\dag)}\; .
    \end{align}
    From the induction hypothesis~\eqref{eq:induction_hypothesis}, we can bound $e_k$ for every $k\in\sbrac{t-\Tcal, t-1}$ as follows:
    \begin{equation*}
        e_k \leq \alpha\brac{1 + \sum_{\ell=1}^{k}{\ell(\rho\Tcal)^{\ell}}}\tilde{\sigma}_S^2 + \frac{2\gamma}{\beta^2\Tcal}\sum_{\ell=1}^{k}{(\rho\Tcal)^\ell\sum_{m=k-\ell\Tcal}^{k-\ell}{\bbE\lVert\nab_m\rVert^2}}\; .
    \end{equation*}
    Denote this bound by $B(k)\coloneqq\alpha\brac{1 + \sum_{\ell=1}^{k}{\ell(\rho\Tcal)^{\ell}}}\tilde{\sigma}_S^2 + \frac{2\gamma}{\beta^2\Tcal}\sum_{\ell=1}^{k}{(\rho\Tcal)^\ell\sum_{m=k-\ell\Tcal}^{k-\ell}{\bbE\lVert\nab_m\rVert^2}}$; that is, $e_k\leq B(k)$. We can therefore bound $(\dag)$ as
    \begin{align*}
        \sum_{k=t-\Tcal}^{t-1}{e_k} \leq \sum_{k=t-\Tcal}^{t-1}{B(k)} \leq \Tcal B(t-1)\; ,
    \end{align*}
     where the last inequality holds because $B(k)$ is monotonically increasing. Plugging back to Eq.~\eqref{eq:client_estimation_error} and substituting $B(t-1)$ gives
    \begin{align}\label{eq:lem_model_estimation_err}
        e_t^i \leq& \alpha\tilde{\sigma}_S^2 + \nu\sum_{k=t-\Tcal}^{t-1}{\bbE\lVert\nab_k\rVert^2} + \rho\cdot\Tcal B(t-1) \nonumber \\ =& \alpha\tilde{\sigma}_S^2 {+}\nu\sum_{k=t-\Tcal}^{t-1}{\bbE\lVert\nab_k\rVert^2}{+} \rho\Tcal\brac{\alpha\brac{1 + \sum_{k=1}^{t-1}{k(\rho\Tcal)^k}}\tilde{\sigma}_S^2 + \frac{2\gamma}{\beta^2\Tcal}\sum_{k=1}^{t-1}{(\rho\Tcal)^{k}}\sum_{\ell=t-1-k\Tcal}^{t-1-k}{\bbE\lVert\nab_\ell\rVert^2}} \nonumber\\ =& \underbrace{\alpha\brac{1 {+} \rho\Tcal {+} \sum_{k=1}^{t-1}{k(\rho\Tcal)^{k+1}}}}_{=(A)}\tilde{\sigma}_S^2+ \underbrace{\nu\sum_{k=t-\Tcal}^{t-1}{\bbE\lVert\nab_k\rVert^2} + \frac{2\gamma}{\beta^2\Tcal}\sum_{k=1}^{t-1}{(\rho\Tcal)^{k+1}\sum_{\ell=t-1-k\Tcal}^{t-1-k}{\bbE\lVert\nab_\ell\rVert^2}}}_{=(B)}\; .
    \end{align}
    \paragraph{Bounding $(A)$: } Using simple algebra, we have that
    \begin{align}\label{eq:bound_on_A_C2}
        \rho\Tcal + \sum_{k=1}^{t-1}{k(\rho\Tcal)^{k+1}} &= \rho\Tcal + \sum_{k=2}^{t}{(k-1)(\rho\Tcal)^{k}} \leq \rho\Tcal + \sum_{k=2}^{t}{k(\rho\Tcal)^{k}} = \sum_{k=1}^{t}{k(\rho\Tcal)^{k}}\; .
    \end{align}
    This implies that $(A)$ is bounded by $\alpha\brac{1 + \sum_{k=1}^{t}{k(\rho\Tcal)^k}}$. 
     \paragraph{Bounding $(B)$: } Focusing on the first term in $(B)$, we can bound:
     \begin{align}\label{eq:bound_on_B1_C2}
         \nu\sum_{k=t-\Tcal}^{t-1}{\bbE\lVert\nab_k\rVert^2} = \frac{\nu}{\rho\Tcal}\cdot\rho\Tcal\sum_{k=t-\Tcal}^{t-1}{\bbE\lVert\nab_k\rVert^2} \leq \frac{2\gamma}{\beta^2\Tcal}\cdot\rho\Tcal\sum_{k=t-\Tcal}^{t-1}{\bbE\lVert\nab_k\rVert^2}\; .
     \end{align}
     Focusing on the second sum in $(B)$, we can bound
     \begin{align}\label{eq:bound_on_B2_C2}
         \frac{2\gamma}{\beta^2\Tcal}\sum_{k=1}^{t-1}{(\rho\Tcal)^{k+1}\sum_{\ell=t-1-k\Tcal}^{t-1-k}{\bbE\lVert\nab_\ell\rVert^2}} &= \frac{2\gamma}{\beta^2\Tcal}\sum_{k=2}^{t}{(\rho\Tcal)^{k}\sum_{\ell=t-1-(k-1)\Tcal}^{t-1-(k-1)}{\bbE\lVert\nab_\ell\rVert^2}} \nonumber\\ &= \frac{2\gamma}{\beta^2\Tcal}\sum_{k=2}^{t}{(\rho\Tcal)^{k}\sum_{\ell=t-k\Tcal +\Tcal - 1}^{t-k}{\bbE\lVert\nab_\ell\rVert^2}} \nonumber\\ &\leq \frac{2\gamma}{\beta^2\Tcal}\sum_{k=2}^{t}{(\rho\Tcal)^{k}\sum_{\ell=t-k\Tcal}^{t-k}{\bbE\lVert\nab_\ell\rVert^2}}\; ,
     \end{align}
     where the last inequality holds since $\Tcal-1\geq 0$ and $\bbE\lVert\nab_\ell\rVert^2\geq 0$ for all $\ell$. Combining the bounds in Eq.~\eqref{eq:bound_on_B1_C2} and \eqref{eq:bound_on_B2_C2}, we can then bound $(B)$ as
     \begin{align}\label{eq:bound_on_B_C2}
         \nu\sum_{k=t-\Tcal}^{t-1}{\bbE\lVert\nab_k\rVert^2} {+} \frac{2\gamma}{\beta^2\Tcal}\sum_{k=1}^{t-1}{(\rho\Tcal)^{k+1}\sum_{\ell=t-1-k\Tcal}^{t-1-k}{\bbE\lVert\nab_\ell\rVert^2}} \leq& \frac{2\gamma}{\beta^2\Tcal}\cdot\rho\Tcal\sum_{k=t-\Tcal}^{t-1}{\bbE\lVert\nab_k\rVert^2} + \frac{2\gamma}{\beta^2\Tcal}\sum_{k=2}^{t}{(\rho\Tcal)^{k}\sum_{\ell=t-k\Tcal}^{t-k}{\bbE\lVert\nab_\ell\rVert^2}} \nonumber\\ =& \frac{2\gamma}{\beta^2\Tcal}\sum_{k=1}^{t}{(\rho\Tcal)^{k}\sum_{\ell=t-k\Tcal}^{t-k}{\bbE\lVert\nab_\ell\rVert^2}}\; .
     \end{align}
    Plugging back to Eq.~\eqref{eq:lem_model_estimation_err} the bounds on $(A)$ and $(B)$ (from Eq.~\eqref{eq:bound_on_A_C2} and \eqref{eq:bound_on_B_C2}, respectively), we get:
    \begin{align*}
        e_t^i &\leq \alpha\brac{1 + \sum_{k=1}^{t}{k(\rho\Tcal)^k}}\tilde{\sigma}_S^2 + \frac{2\gamma}{\beta^2\Tcal}\sum_{k=1}^{t}{(\rho\Tcal)^{k}\sum_{\ell=t-k\Tcal}^{t-k}{\bbE\lVert\nab_\ell\rVert^2}}\; .
    \end{align*}
    Since this bound is independent of $i$, it also holds for the average $e_t = \frac{1}{N}\sum_{i=1}^{N}{e_t^i}$, establishing the result. 
\end{proof}

In the following two lemmas, we characterize the second moment of the sum of independent random variables. 
\begin{Lem}[Lemma~4, \citealp{karimireddy2020scaffold}]\label{lem:technical_norm_of_sum_of_iid}
    Let $X_1,\ldots,X_N\in\reals^d$ be $N$ independent random variables. Suppose that $\bbE[X_i] = \mu_i$ and $\bbE\norm{X_i - \mu_i}^2\leq\sigma_i^2$. Then, the following holds
    \begin{equation*}
        \bbE\bnorm{\sum_{i=1}^{N}{X_i}}^2\leq 2\bnorm{\sum_{i=1}^{N}{\mu_i}}^2 + 2\sum_{i=1}^{N}{\sigma_i^2}\; .
    \end{equation*}
\end{Lem}

\begin{Lem}\label{lem:sqr_norm_of_sum_to_sum_of_sqr_norm}
    Let $X_1,\ldots,X_N\in\reals^d$ be $N$ orthogonal, zero mean random variables, i.e., $\bbE[X_i] = 0$ for all $i\in\sbrac{N}$, and $\bbE[X_i^\top X_j] = 0$ for all $i\neq j$. Then, the following holds:
    \begin{align*}
        \bbE\bnorm{\sum_{i=1}^{N}{X_i}}^2 = \sum_{i=1}^{N}{\bbE\lVert X_i \rVert^2}\; .
    \end{align*}
\end{Lem}
\begin{proof}
    By the linearity of expectation, and the following property: $\bbE[X_i^\top X_j]=0, \enskip\forall i\neq j$, we immediately get that $\bbE\norm{\sum_{i=1}^{N}{X_i}}^2 = \bbE\sbrac{\sum_{i=1}^{N}{\sum_{j=1}^{N}{X_i^\top X_j}}} = \sum_{i=1}^{N}{\bbE\lVert X_i\rVert^2}$.
\end{proof}

Next, we state a simple result about the squared norm of the sum of vectors.
\begin{Lem}\label{lem:technical_norm_of_sum_to_sum_of_norms}
    For any $u_1,\ldots,u_N\in\reals^d$, it holds that $\norm{\sum_{i=1}^{N}{u_i}}^2 \leq N\sum_{i=1}^{N}{\norm{u_i}^2}$. 
\end{Lem}
\begin{proof}
    By the convexity of $\lVert\cdot\rVert^2$ and Jensen's inequality: $\norm{\frac{1}{N}\sum_{i=1}^{N}{u_i}}^2 {\leq} \frac{1}{N}\sum_{i=1}^{N}{\norm{u_i}^2}$, which implies the result. 
\end{proof}
The next result is a simple bound on a double sum of non-negative numbers. 
\begin{Lem}\label{lem:double_sum_to_single_sum}
    Let $t,\tau\in\naturals$ such that $t\geq \tau+1$. For any sequence of non-negative numbers $x_0,x_1,\ldots,x_{t-\tau-1}$, the following holds
    \[
        \sum_{k=0}^{t-\tau}{\sum_{\ell=0}^{k-1}{x_\ell}}\leq t\cdot\sum_{k=0}^{t-\tau-1}{x_k}\; .
    \]
\end{Lem}
\begin{proof}
    Immediately: $\sum_{k=0}^{t-\tau}{\sum_{\ell=0}^{k-1}{x_\ell}} = \sum_{k=0}^{t-\tau-1}{(t-\tau-k)x_k} \leq t\cdot\sum_{k=0}^{t-\tau-1}{x_k}$.
\end{proof}
The following lemma gives a bound on the derivative of a power series.  
\begin{Lem}\label{lem:term_by_term_differentiation}
Let $a<1$. Then, 
    \begin{equation*}
        \sum_{k=1}^{\infty}{k a^{k}} =\frac{a}{(1-a)^2}\; .
    \end{equation*}
\end{Lem}
\begin{proof}
    Let $f_k(a) = a^k$. 
    \begin{align*}
        \sum_{k=1}^{\infty}{k a^{k}} = a\sum_{k=1}^{\infty}{k a^{k-1}} = a\sum_{k=1}^{\infty}{f'_k(a)}\; .
    \end{align*}    
    Using term-by-term differentiation~\cite{stewart2015calculus}, we have that
    \begin{equation*}
        \sum_{k=1}^{\infty}{f'_k(a)} = \brac{\sum_{k=1}^{\infty}{f_k(a)}}' = \brac{\sum_{k=1}^{\infty}{a^k}}' = \brac{\frac{a}{1-a}}' = \frac{1}{(1-a)^2}\; .
    \end{equation*}
    Multiplying by $a$ gives the result. 
\end{proof}
Next, we state a non-trivial inequality to bound the double sum that appears on the right-hand side of Eq.~\eqref{eq:double_sum_to_be_bounded_by_lemma}. 
\begin{Lem}\label{lem:double_sum_of_weighted_grads_to_single_sum}
    Let $a\in(0,1)$ and $T, \Tcal\in\naturals$. Moreover, let $x_0,\ldots,x_{T-1}$ be a sequence of non-negative numbers. Then,
    \begin{equation*}
        \sum_{t=1}^{T}{\sum_{k=0}^{t-1}{a^{\ceil{\frac{t-k}{\Tcal}}} x_k}} \leq \Tcal\frac{a}{1-a}\sum_{k=0}^{T-1}{x_k}\; .
    \end{equation*}
\end{Lem}
\begin{proof}
    We start with changing the order of summation in the left-hand side. Note that for any fixed $k$, the element $x_k$ appears in the inner sum if $k\leq t-1$, or equivalently, $t\geq k+1$. Therefore,
    \begin{align}\label{eq:change_summation_order_lemma_C5}
        \sum_{t=1}^{T}{\sum_{k=0}^{t-1}{a^{\ceil{\frac{t-k}{\Tcal}}} x_k}} &= \sum_{k=0}^{T-1}{\brac{\sum_{t=k+1}^{T}{a^{\ceil{\frac{t-k}{\Tcal}}}}}x_k} = \sum_{k=0}^{T-1}{\brac{\sum_{t=1}^{T-k}{a^{\ceil{\frac{t}{\Tcal}}}}}x_k}\; .
    \end{align}
    Focusing on the inner sum in the right-hand side, $\sum_{t=1}^{T-k}{a^{\ceil{t/\Tcal}}}$, we can divide the interval of integers from $1$ to $T-k$ into non-overlapping intervals of length $\Tcal$ (and possibly a small residual) and get that
    \begin{equation*}
        \sum_{t=1}^{T-k}{a^{\ceil{\frac{t}{\Tcal}}}} \leq \sum_{m=1}^{\ceil{\frac{T-k}{\Tcal}}}{\sum_{\ell=1}^{\Tcal}{a^{\ceil{\frac{(m-1)\Tcal+\ell}{\Tcal}}}}} \overset{(\dag)}{=} \sum_{m=1}^{\ceil{\frac{T-k}{\Tcal}}}{\sum_{\ell=1}^{\Tcal}{a^{m}}} = \Tcal\sum_{m=1}^{\ceil{\frac{T-k}{\Tcal}}}{a^m} \leq \Tcal\frac{a}{1-a}\; .
    \end{equation*}
    where $(\dag)$ holds because for every $\ell=1,\ldots,\Tcal$ we have $\ceil{\frac{(m-1)\Tcal + \ell}{\Tcal}} = m$, and the last inequality follows from $\sum_{m=1}^{\ceil{\frac{T-k}{\Tcal}}}{a^m}\leq \sum_{m=1}^{\infty}{a^m} = \frac{a}{1-a}$ as $a<1$. Plugging back to Eq.~\eqref{eq:change_summation_order_lemma_C5} concludes the proof.
\end{proof}
The next lemma establishes that for small enough $\eta$, we have $-\eta/2 + \Ocal(\eta^2)\leq -\eta/4$.
\begin{Lem}\label{lem:sum_of_grads_coeff_bound}
    Let $\gamma,\theta\geq 1$. For every $\eta\leq\frac{1}{30\gamma\beta\theta}$, it holds that
    \begin{equation*}
        -\frac{\eta}{2} + 2\gamma\beta\eta^2 + 160\gamma\beta^2\theta^2\eta^3 \leq -\frac{\eta}{4}\; .
    \end{equation*}
\end{Lem}
\begin{proof}
    We equivalently prove that
    \begin{equation*}
        2\gamma\beta\eta^2 + 160\gamma\beta^2\theta^2\eta^3 \leq \frac{\eta}{4}\; .
    \end{equation*}
    Since both $\gamma\geq 1$ and $\theta\geq 1$, we have 
    \begin{align*} 
        2\gamma\beta\eta^2 + 160\gamma\beta^2\theta^2\eta^3 &\leq 2\gamma\beta\theta\eta^2 + 160\gamma^2\beta^2\theta^2\eta^3 \\ &= \frac{\eta}{4}\brac{8\gamma\beta\theta\eta + 640\gamma^2\beta^2\theta^2\eta^2} \\ &\leq \frac{\eta}{4}\brac{\frac{8\gamma\beta\theta}{30\gamma\beta\theta} + \frac{640\gamma^2\beta^2\theta^2}{900\gamma^2\beta^2\theta^2}} = \frac{\eta}{4}\cdot\frac{44}{45} \leq \frac{\eta}{4}\; ,  
    \end{align*}
    where the second inequality follows from the upper bound on $\eta$. 
\end{proof}
We also make use of the following result, which we prove using simple algebra.
\begin{Lem}\label{lem:learning_rate_bound}
    Suppose $\eta = \min\cbrac{\eta_1,\eta_2,\eta_3}$ for some $\eta_1,\eta_2,\eta_3>0$, and let $A,B,C>0$. Then, the following holds:
    \begin{align*}
        \frac{A}{\eta} + B\eta + C\eta^2 \leq A\brac{\frac{1}{\eta_1} + \frac{1}{\eta_2} + \frac{1}{\eta_3}} + B\eta_2 + C\eta_3^2\; .
    \end{align*}
\end{Lem}
\begin{proof}
    Since $\eta$ is the minimum of three terms, $1/\eta$ is the maximum of their inverses. Thus, we can bound $1/\eta$ by the sum of the inverses as follows:
    \begin{align*}
        \frac{A}{\eta} = A\max\cbrac{\frac{1}{\eta_1}, \frac{1}{\eta_2}, \frac{1}{\eta_3}} \leq A\brac{\frac{1}{\eta_1} + \frac{1}{\eta_2} + \frac{1}{\eta_3}}\; .
    \end{align*}
    The terms $B\eta$ and $C\eta^2$ are monotonically increasing with $\eta$. We can therefore bound $\eta$ by $\eta_2$ and $\eta^2$ by $\eta_3^2$. 
\end{proof}

%% file: appendix/ecuq.tex
\section{Entropy-Constrained Uniform Quantization}\label{app:ecuq}
In this section, we describe a new compression technique entitled Entropy-Constrained Uniform Quantization (ECUQ), which we developed for anchor compression, although it can be of independent interest. ECUQ is described in Algorithm~\ref{alg:ecuq}. 

\vspace{-2mm}
Let $x = \brac{x(1),\ldots,x(d)}\in\reals^d$ be some input vector we wish to compress using ECUQ. Denote: $x_{\min}\coloneqq \min_{i}{x(i)}$, $x_{\max}\coloneqq \max_{i}{x(i)}$. Given some bandwidth budget of $b$ bits/coordinate, ECUQ initially divided the interval $\sbrac{x_{\min}, x_{\max}}$ into $K=2^b$ non-overlapping bins of equal size $\Delta = (x_{\max} - x_{\min})/K$. Then, it sets the quantization values, which we denote by $\Q$, to be the centers of these bins. Afterwards, the vector $x$ is quantized into elements of $\Q$, that is, each element $x(i)$ is assigned to its closest quantization value $q\in\Q$ to generate the quantized vector $\hat{x}_\Q$, whose elements are all in $\Q$. We then compute the empirical distribution of the quantized vector by counting for every $q\in\Q$ the number of times it appears in $\hat{x}_\Q$, and the entropy of the resulting distribution. Note that the entropy is upper bounded by $\log{K} = b$. Finally, for some small tolerance parameter $\epsilon$ (we use $\epsilon=0.1$), we check whether the entropy is within $\epsilon$ distance from the budget $b$: if it is not the case, then we perform a double binary search, repeating the above procedure with increased number of quantization values $K$, to find the maximal number of uniformly spaced quantization values such that the entropy of the empirical distribution of the resulting quantized vector is within $\epsilon$ distance from $b$. Only after this entropy condition is satisfied, we encode $\hat{x}_\Q$ using some entropy encoding (we use Huffman coding).  

\vspace{-2mm}
\begin{algorithm}
\caption{Entropy-Constrained Uniform Quantization (ECUQ)}
\begin{algorithmic}\label{alg:ecuq}
    \STATE{\bfseries Input:} Vector $x\in\reals^d$, bandwidth budget $b$ (bits/coordinate), tolerance $\epsilon$. 
    \STATE $x_{\max}\leftarrow \max_{i}{x(i)}, \enskip x_{\min}\leftarrow \min_{i}{x(i)}$ \algorithmiccomment{Get max/min values of input vector}
    \STATE $K\leftarrow 2^{b}, \enskip \Delta\leftarrow(x_{\max} - x_{\min})/K$ \algorithmiccomment{Initialize \# of quantization values and bin length}
    \STATE $\Q \leftarrow\cbrac{x_{\min} + \brac{k+\frac{1}{2}}\cdot\Delta: k=0,\ldots,K-1}$ \algorithmiccomment{Set uniformly spaced quantization values}
    \STATE $\hat{x}_\Q\leftarrow\texttt{Quantize}(x, \Q)$ \algorithmiccomment{$\hat{x}_{\Q}(i) = \arg\min_{q\in\Q}{\norm{x(i) - q}}$}
    \STATE $p_{\mathcal{Q}}\leftarrow\texttt{Empirical\_Density}(\hat{x}_\Q)$ \algorithmiccomment{$p_{\mathcal{Q}}(q) = \frac{1}{N}\sum_{i\in\sbrac{N}}{\mathbbm{1}\cbrac{\hat{x}_\Q(i)=q}}, \enskip \forall q\in\Q$}
    \STATE $\entropy(p_{\mathcal{Q}})\leftarrow\texttt{Entropy}(p_\Q)$ \algorithmiccomment{$\entropy(p_{\mathcal{Q}}) = -\sum_{q\in\Q}{p_{\mathcal{Q}}(q)\log{p_{\mathcal{Q}}(q)}}$} 
    \IF{$\entropy(p_{\mathcal{Q}})< b-\epsilon$}
        \STATE $\hat{x}_\Q\leftarrow\textsc{Double\_Binary\_Search\_Num\_Quantization\_Levels}(x, b)$
    \ENDIF
    \STATE $\hat{x}_e\leftarrow\texttt{Huffman\_Coding}(\hat{x}_\Q)$ \algorithmiccomment{Entropy encoding of the quantized vector}
    \STATE{\bfseries Return:} $\hat{x}_e$
    \vspace{0.05in}
    \STATE \textbf{Procedure} \textsc{Double\_Binary\_Search\_Num\_Quantization\_Levels}$(x, b)$ 
        \STATE \hspace{1em} Initialize: $\text{low}\leftarrow 2^b,\enskip \text{high}\leftarrow\infty,\enskip p\leftarrow -1$
        \STATE \hspace{1em} {\bfseries while} $\text{low}\leq\text{high}$ {\bfseries do}
        \STATE \hspace{2em} {\bfseries if} $\text{high} == \infty$ {\bfseries then}
        \STATE \hspace{3em} $p\leftarrow p+1$
        \STATE \hspace{3em} $\text{mid}\leftarrow 2^b + 2^p$ \algorithmiccomment{Increase \# of levels exponentially}
        \STATE \hspace{2em} {\bfseries else} 
        \STATE \hspace{3em} $\text{mid}\leftarrow (\text{low} + \text{high})/2$
        \STATE \hspace{2em} {\bfseries end if}
        \STATE \hspace{2em} $K\leftarrow \text{mid}, \enskip \Delta\leftarrow(x_{\max} - x_{\min})/\text{mid}$
        \STATE \hspace{2em} $\Q \leftarrow\cbrac{x_{\min} + \brac{k+\frac{1}{2}}\cdot\Delta: k=0,\ldots,K-1}$
        \STATE \hspace{2em} $\hat{x}_\Q\leftarrow\texttt{Quantize}(x, \Q)$ 
        \STATE \hspace{2em} $p_{\mathcal{Q}}\leftarrow\texttt{Empirical\_Density}(\hat{x}_\Q)$ 
        \STATE \hspace{2em} $\entropy(p_{\mathcal{Q}})\leftarrow\texttt{Entropy}(p_\Q)$ 
        \STATE \hspace{2em} {\bfseries if} $\entropy(p_{\mathcal{Q}})>b$ {\bfseries then}
        \STATE \hspace{3em} $\text{high}\leftarrow\text{mid}-1$
        \STATE \hspace{2em} {\bfseries else if} $\entropy(p_{\mathcal{Q}})<b-\epsilon$ {\bfseries then}
        \STATE \hspace{3em} $\text{low}\leftarrow\text{mid}+1$
        \STATE \hspace{2em} {\bfseries else}
        \STATE \hspace{3em} return $\hat{x}_\Q$
        \STATE \hspace{2em} {\bfseries end if}
\end{algorithmic}
\end{algorithm}

Figure~\ref{fig:ecuq} illustrates ECUQ's encoder, as described in the text above. The corresponding decoder is fairly simple as it only performs entropy decoding in linear time (Huffman decoding). 

\begin{wrapfigure}{r}{0.45\textwidth} 
    \centering
    \includegraphics[width=\linewidth]{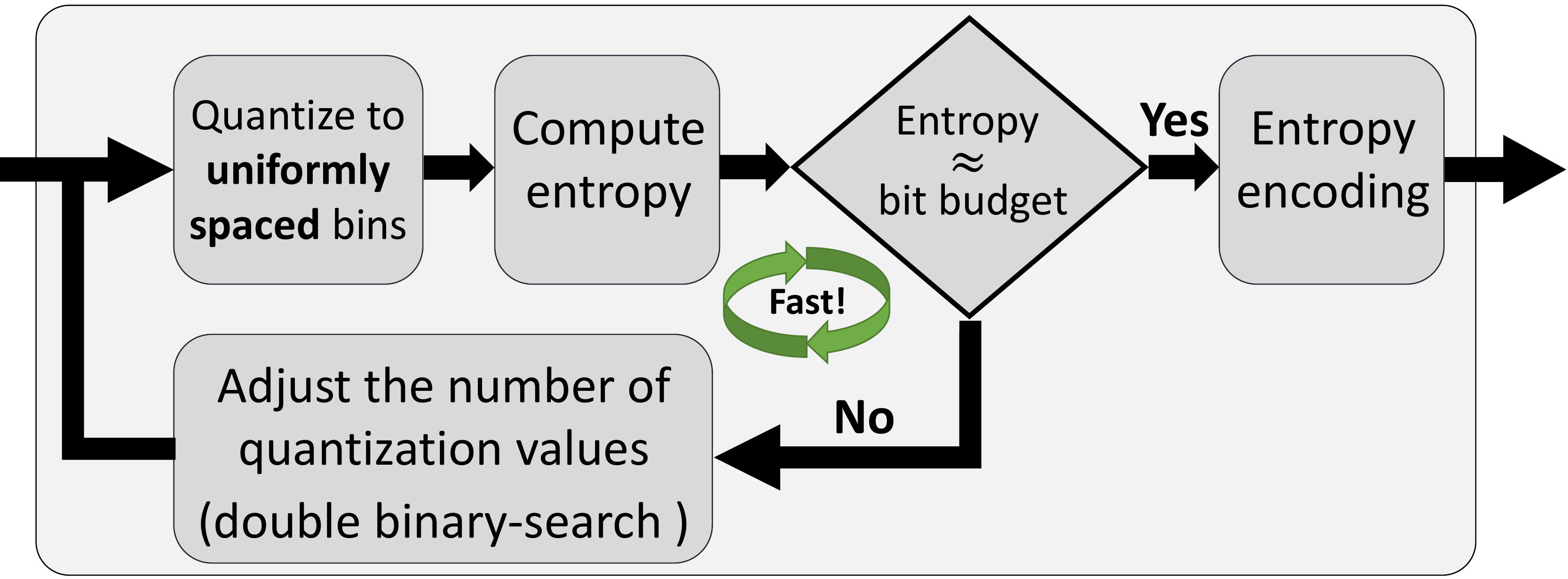}
    \vspace{-5mm}
    \caption{ECUQ encoder's illustration.}
    \vspace{-3mm}
    \label{fig:ecuq}
\end{wrapfigure}

While devising ECUQ, we also considered an additional method to approximate ECQ. It is similar to ECUQ, but instead of using uniformly spaced quantization values, it uses K-Means clustering to find the quantization values that minimize the overall squared error. We used a double binary search to find the largest number of levels $K$ such that, after entropy encoding, the bandwidth constraint is satisfied. We termed this method \emph{Entropy-Constrained K-Means (ECK-Means)}.

We compare the performance of ECUQ with ECQ and ECK-Means in terms of their NMSE, and we also measure their encoding time. As we mentioned in the Section~\ref{sec:anchor_compression}, ECQ is sensitive to hyperparameters; thus, we implemented it using a grid search over its hyperparameters to guarantee near-optimal performance.\footnote{While such implementation may increase the encoding time, we are not aware of any other approach to guarantee an optimal performance. ECQ aims at solving a hard non-convex problem, and different hyperparameters may result in different local minima.} We evaluate the three methods on vectors drawn from three different synthetic distributions: \textbf{(1)} \textit{LogNormal}$(0, 1)$; \textbf{(2)} \textit{Normal}$(0, 1)$; and \textbf{(3)} \textit{Normal}$(1, 0.1)$. In Figure~\ref{fig:synth} (\textbf{top}) we show the NMSE and encoding time for different sizes of input vectors when the budget constraint is $b=2$ bits/coordinate. As a complementary result, in Figure~\ref{fig:synth} (\textbf{bottom}) we fix the dimension of the input vectors to $d=2^{12}=4096$ and vary the bandwidth budget constraint from $2$ to $5$ bits/coordinate. The results imply that ECUQ exhibits a good speed-accuracy trade-off: it consistently outperforms ECK-Means while being an order of magnitude faster, and it is competitive with ECQ but about three orders of magnitude faster. Note additionally that it takes $\approx 20$ minutes for ECQ to encode even a small vectors of size $2^{12}$ with budget constraint of $4$ bits/coordinate; this means that ECQ without some acceleration is not suitable for compressing neural networks with millions and even billions of parameters.      

\begin{figure}[h]
\centering
\vspace{-6mm}
\begin{subfigure}{0.74\linewidth}
  \centering
    \includegraphics[clip, trim=1cm 0cm 1cm 0.5cm, width=\textwidth]{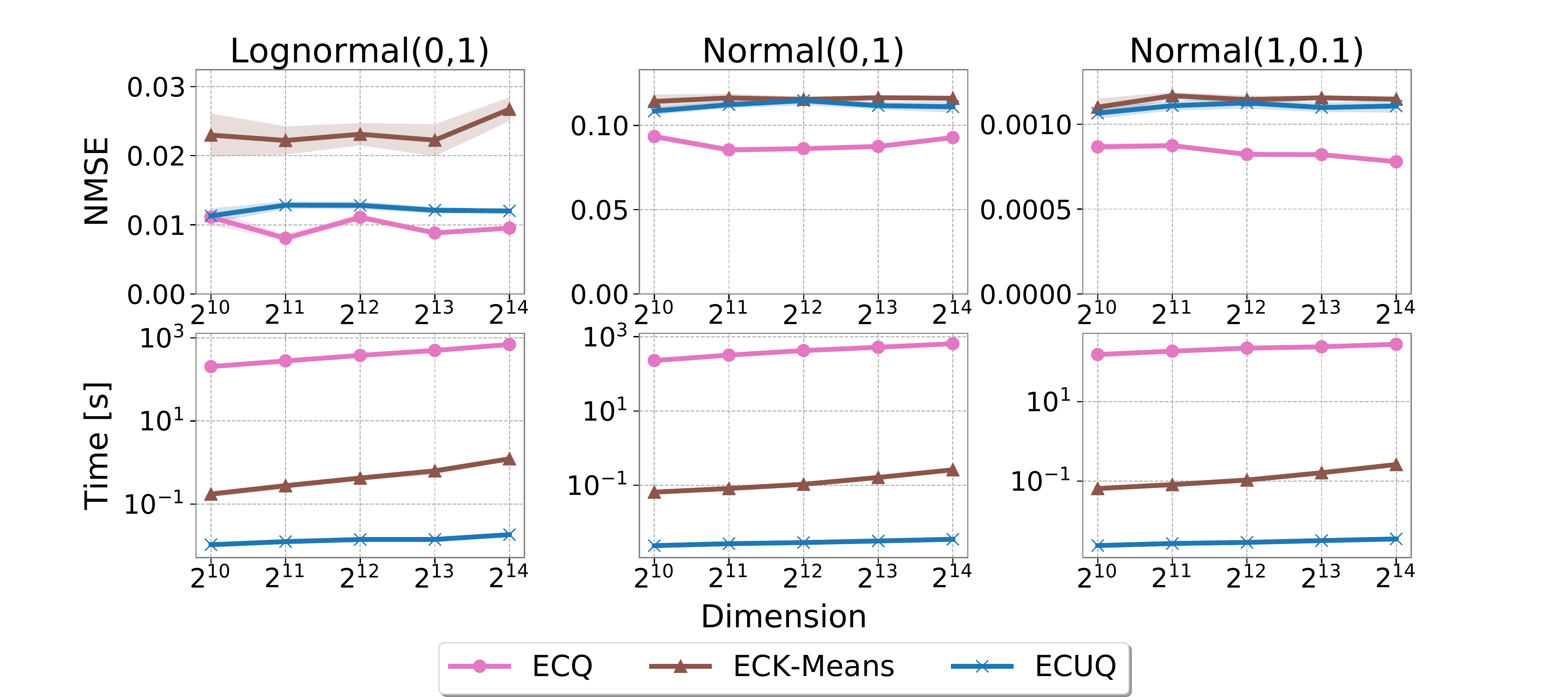}
  \label{fig:synth_dim}
  \vspace{-3mm}
\end{subfigure}
\begin{subfigure}{0.74\linewidth}
  \centering
    \includegraphics[clip, trim=1cm 0cm 1cm 0.5cm, width=\textwidth]{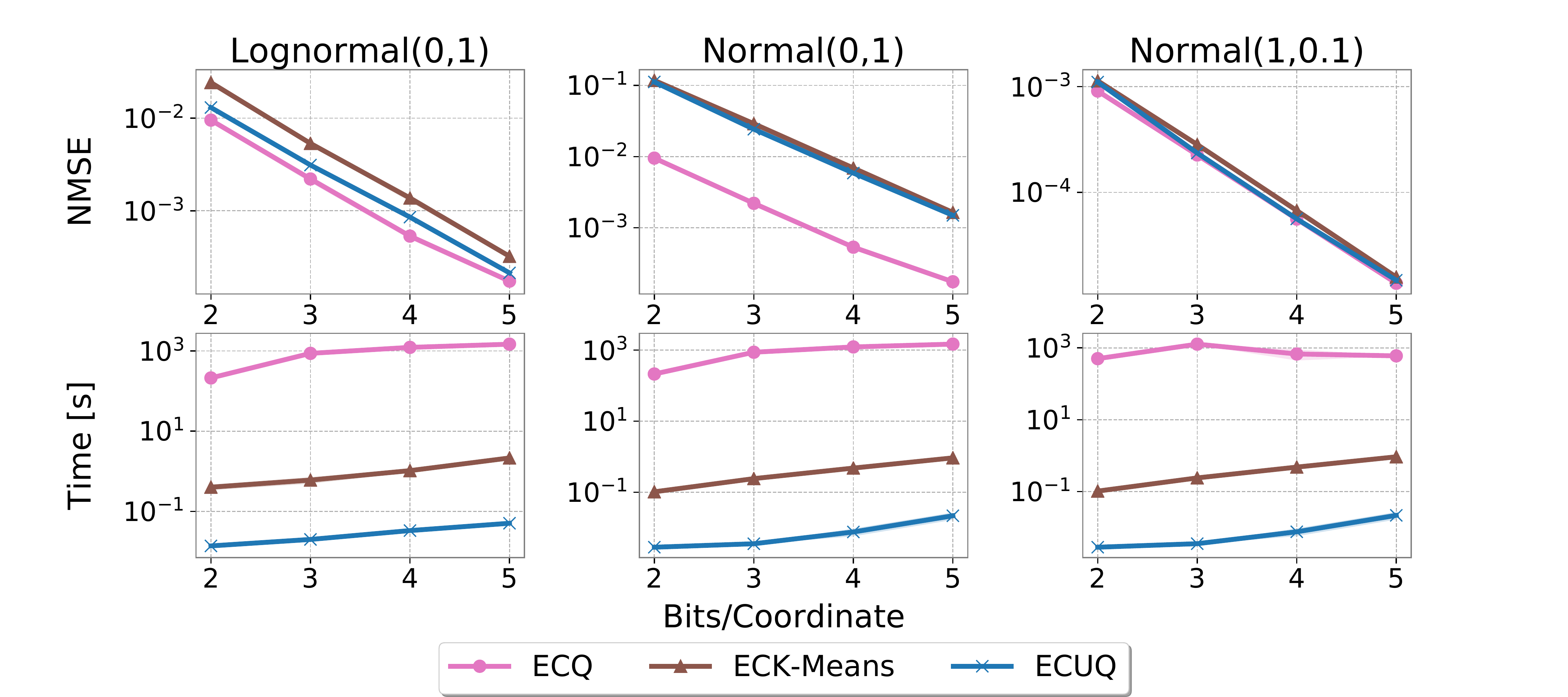}
  \label{fig:synth_bits}
\end{subfigure}
\vspace{-4mm}
\caption{ECQ vs. ECK-Means vs. ECUQ: NMSE and encoding time for different input distributions: \textbf{(top)} as a function of the bandwith budget for fixed input dimension of $d=2^{12}$; and \textbf{(bottom)} as a function of the input dimension for fixed bandwidth budget of $2$ bits/coordinate.}
\label{fig:synth}
\end{figure}


%% file: appendix/ecuq_vs_sparsification_vs_sketching.tex
\blue{\section{Additional ECUQ Evaluations}\label{app:ecuq_vs_spars_vs_sketch}
Since ECUQ is a quantization-based method, in \Cref{sec:anchor_compression} and \Cref{app:ecuq} we compare it with quantization-based techniques. In \Cref{fig:ecuq_vs_spars_vs_sketch}, we give a complementary result comparing it with sparsification methods (Rand-K, Top-K) and sketching (Count-Sketch), where similar trends are observed. Note, however, that such techniques are mostly orthogonal to quantization-based methods and they can be used in conjunction.}
\begin{figure}[h]
    \centering
    \vspace{-3mm}
    \includegraphics[clip, trim=1cm 0cm 1cm 0.3cm, width=0.7\textwidth]{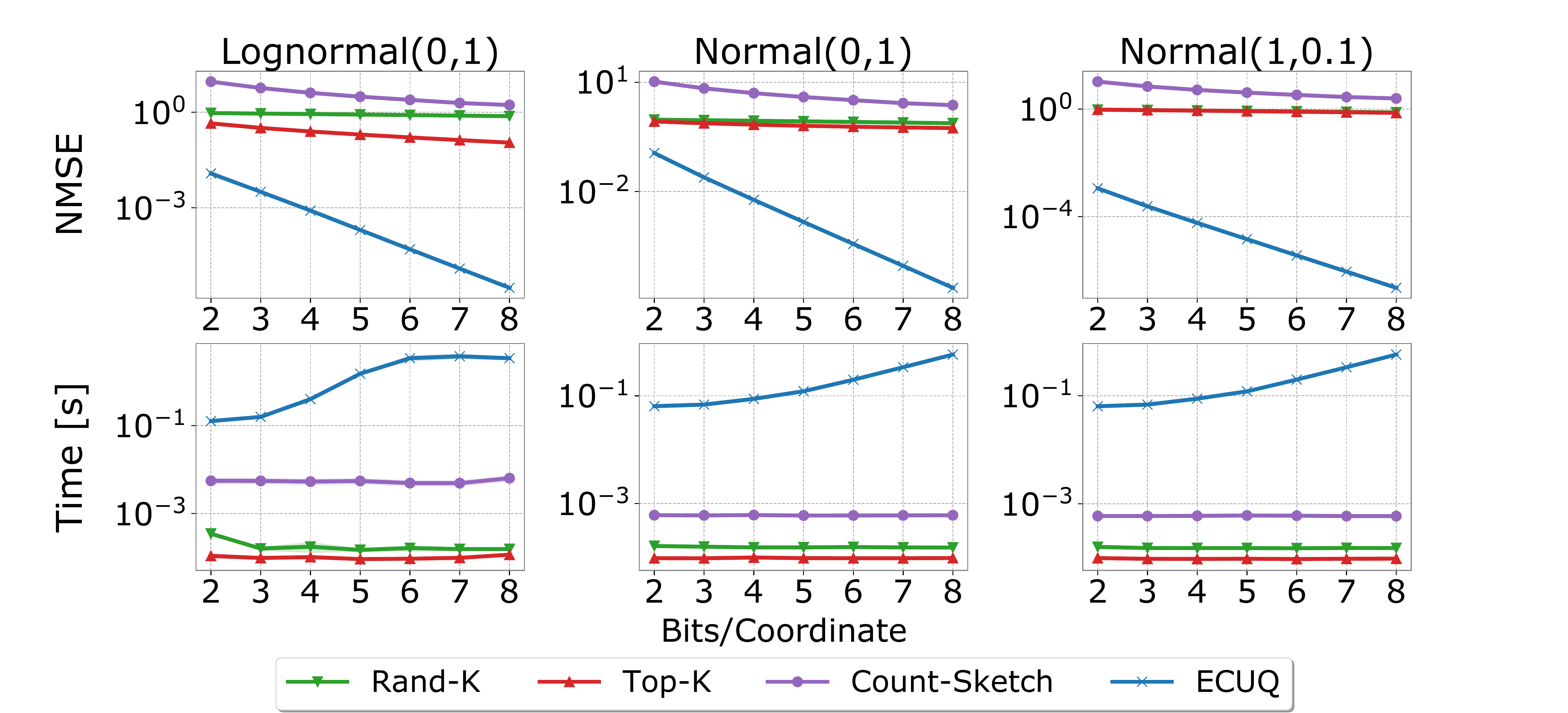}
    \vspace{-2mm}
    \caption{ECUQ vs. sparsification and sketching: NMSE (\textbf{top}) and encoding time (\textbf{bottom}) as a function of the bandwidth budget for different input distributions and a fixed dimension of $d=2^{20}$.}
    \vspace{-5mm}
    \label{fig:ecuq_vs_spars_vs_sketch}
\end{figure}

%% file: appendix/experiments_config.tex
\section{Experimental Details}\label{app:experiments_config}
We implemented \methodName{} in PyTorch~\cite{paszke2019pytorch}. In all experiments, the PS uses Momentum SGD as optimizer with a momentum of $0.9$ and $L_2$ regularization (i.e., weight decay) with parameter $10^{-5}$. The clients, on the other hand, use vanilla SGD for all tasks but Amazon Reviews, for which Adam provided better results. In~\cref{tab:exp_details} we report the hyperparameters used in our experiments. To ease the computational burden and long training times, in the Shakespeare task we reduced the amount of train and validation data for each speaker (i.e., client) by a factor of $10$ by using only the first $10\%$ of train and validation data, but no less than $2$ samples per speaker. 

\input{tables/table_experimental_details.tex}

%% file: tables/table_experimental_details.tex
\begin{table*}[h]
\vspace{-5mm}
\centering
\caption{Hyperparameters for our experiments.}
\vspace{0.1in}
\begin{tabular}{lllll}
\hline
\textbf{Task} & \textbf{Batch size} & \textbf{Client optimizer} & \textbf{Client lr} &\textbf{Server lr}  \\ 
\hline
EMNIST & $64$ & SGD & $0.05$ & $1$ \\
CIFAR-100 & $128$ & SGD & $0.05$ & $1$ \\
Amazon Review & $64$ & Adam & $0.005$ & $0.1$ \\
Shakespeare & $4$ & SGD & $0.5$ & $1$ \\ \hline
\end{tabular}
\vspace{-3mm}
\label{tab:exp_details}
\end{table*}

%% file: appendix/additional_results.tex
\section{Additional Results}\label{app:additional_results}
In this section we present additional results that were deferred from the main text.
\subsection{\blue{Learning Curves}}\label{subsec:learning_curves}
We next provide the learning curves for the experiments we conducted in \S~\ref{sec:experiments}. In \cref{fig:valid,fig:train} we show the validation and train accuracy throughout training, respectively. We measure train and validation accuracy every $50$ rounds for EMNIST and Amazon Reviews, every $500$ rounds for CIFAR-100, and every $1000$ rounds for Shakespeare.

\begin{figure}[h]
    \centering
    \includegraphics[clip, trim=3.5cm 1cm 2.8cm 0.1cm,width=0.95\linewidth]{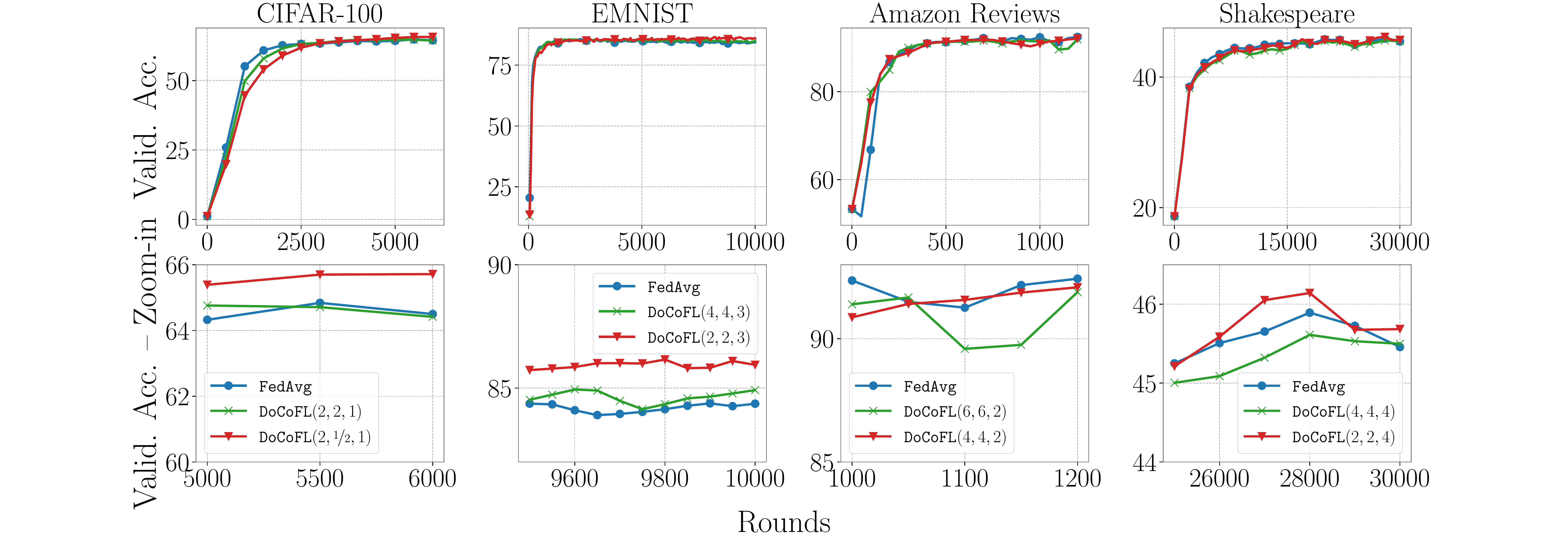}
    \vspace{-3mm}
    \caption{Validation accuracy for different tasks. }
    \label{fig:valid}
\end{figure}
\begin{figure}[h]
    \centering
    \vspace{2mm}
    \includegraphics[clip, trim=3.5cm 1cm 2.8cm 0.15cm,width=0.95\linewidth]{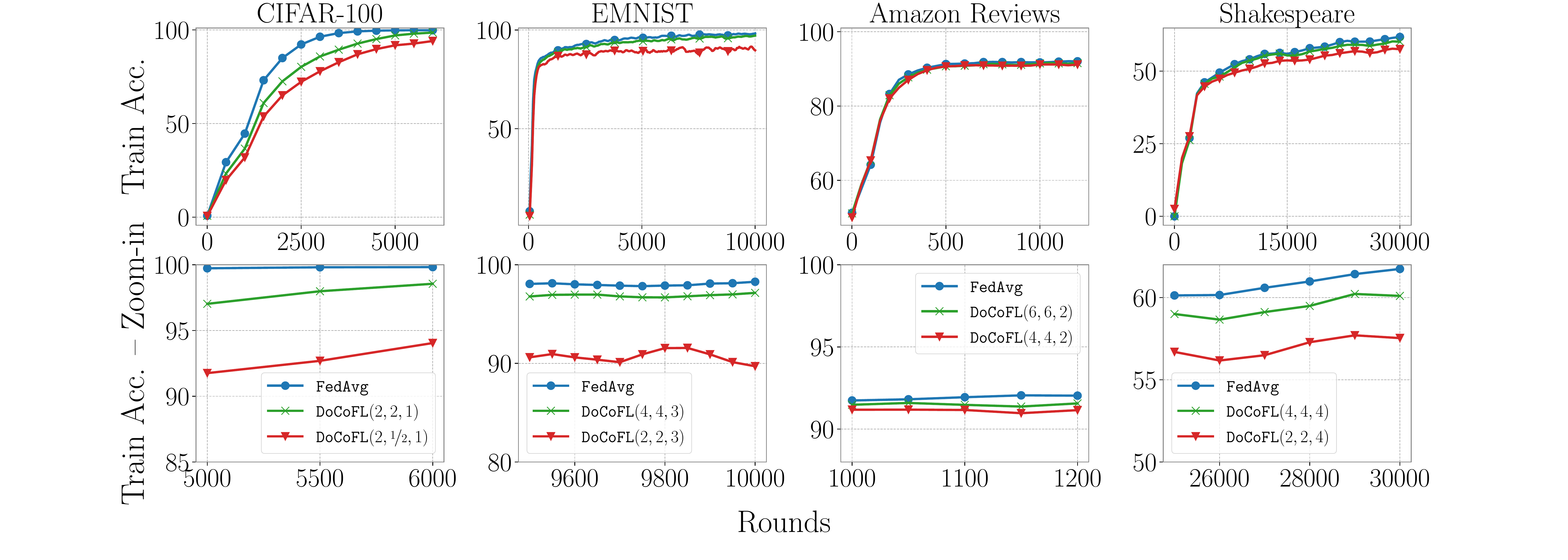}
    \vspace{-3mm}
    \caption{Train accuracy for different tasks. }
    \label{fig:train}
\end{figure}

Following the discussion in \S~\ref{sec:experiments}, \cref{fig:valid} demonstrates that using less bandwidth may improve the generalization ability as it can serve as a form of regularization. For example, consider the EMNIST task, where \methodName{}$(2,2,3)$ (i.e., $2$ bits per coordinate for anchor and correction compression, and $3$ bits per coordinate for uplink compression) outperforms both \fedavg{} and \methodName{}$(4,4,3)$. Unsurprisingly, examining \cref{fig:train} reveals a reverse image -- less bandwidth implies lower train accuracy. This suggests that in some settings using less bandwidth (but not too little) may help to prevent overfitting.

\subsection{\blue{Bandwidth Budget Ablation}}\label{subsec:bandwidth_budget_ablation}
\blue{Next, we provide numerical results that demonstrate the effect of the downlink (anchor and correction) bandwidth budget on \methodName{}'s performance. We consider the CIFAR-100 with ResNet-9 experiment with anchor deployment rate $K=10$ and anchor queue capacity $\V=3$.} 

\begin{figure}[H]
\centering
\begin{subfigure}{0.48\linewidth}
  \centering
  \includegraphics[clip, trim=1.5cm 0cm 0cm 0.1cm,width=\linewidth]{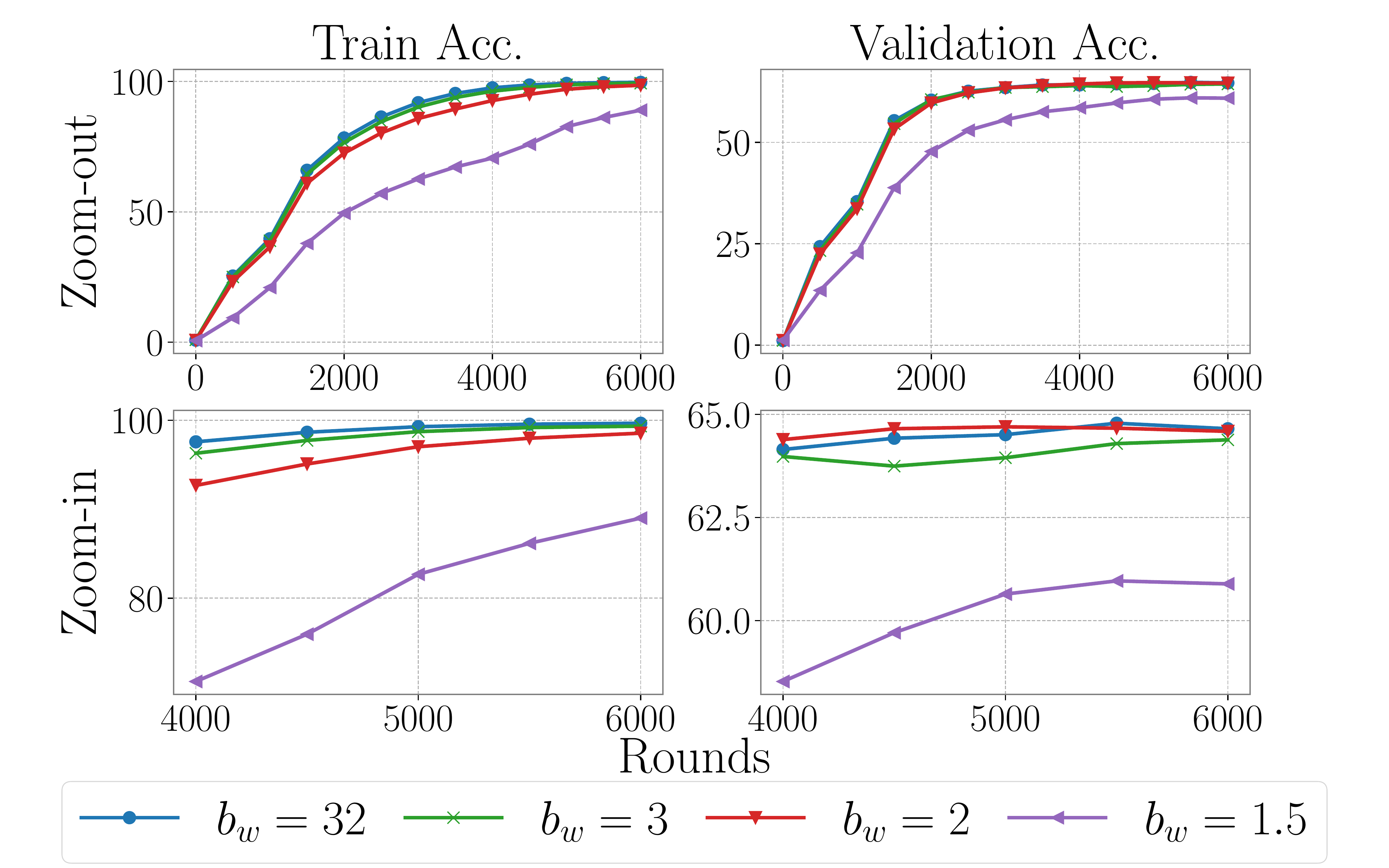}  
  \label{fig:varying_b_w_rounds}
\end{subfigure}
\hfill
\begin{subfigure}{0.48\linewidth}
  \centering
  \includegraphics[clip, trim=1.5cm 0cm 0cm 0.1cm, width=\linewidth]{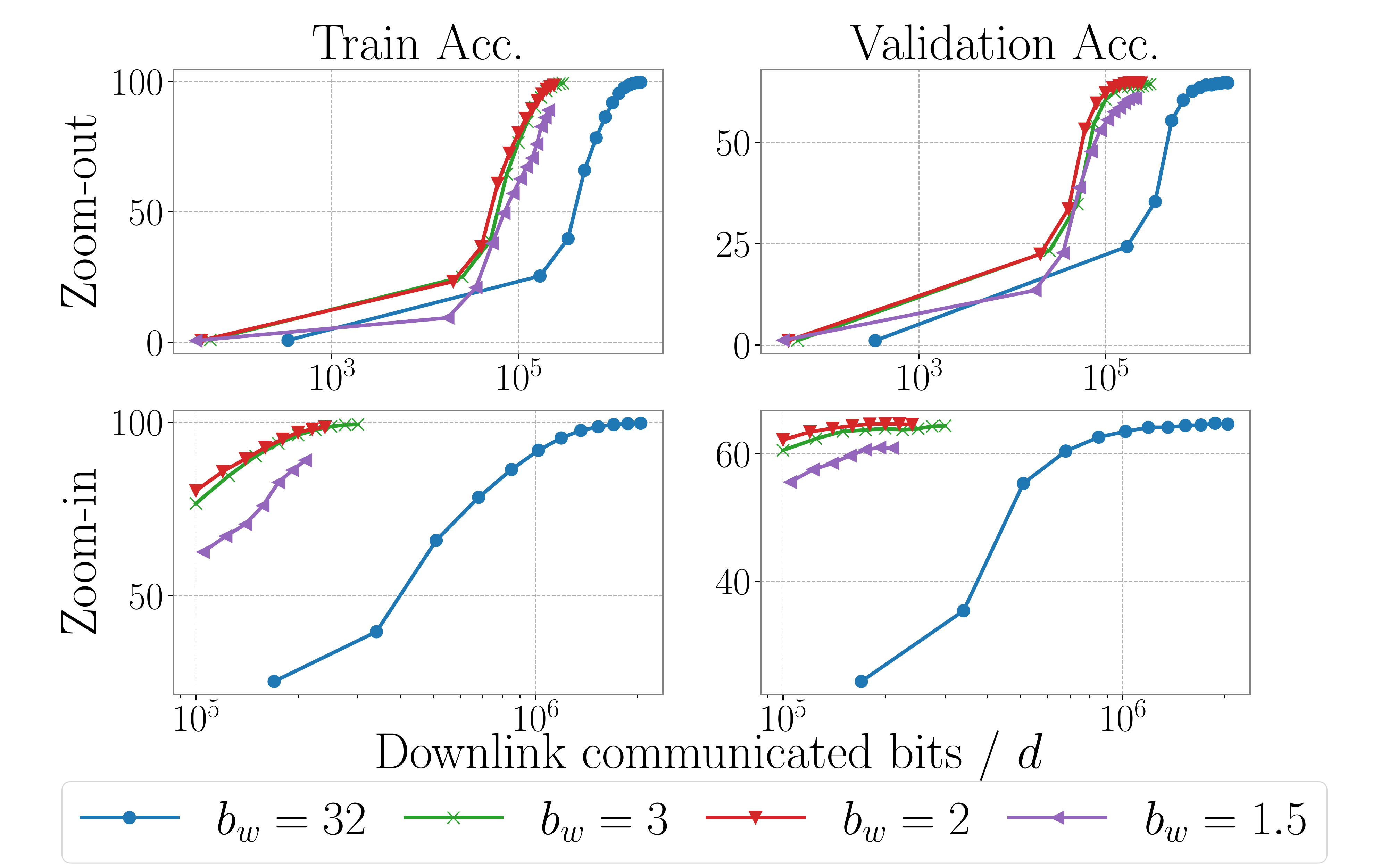}  
  \label{fig:varying_b_w_comm}
\end{subfigure}
\vspace{-3mm}
\caption{Train and validation accuracy for different anchor bandwidth budgets ($32, 3, 2$, and $1.5$ bits per coordinate) as a function of the number of rounds (\textbf{left}) and the number of communicated bits in the downlink direction (\textbf{right}).}
\label{fig:varying_b_w}
\end{figure}

\blue{In \Cref{fig:varying_b_w} we show the train and validation accuracy for different anchor bandwidth budgets $b_w$, namely, $32$ (full-precision), $3, 2$ and $1.5$ bits per coordinate, while the correction budget is fixed and equals $b_c=2$ bits per coordinate, as a function of both number of rounds and number of communicated bits in the downlink direction. The results indicate that one can significantly reduce the bandwidth used for communicating the anchors and use as low as $b_w=2$ bits per coordinate for anchor compression ($16\times $ reduction), without degrading validation accuracy. Again, similarly to evidence from the previous section, less bandwidth typically results in lower train accuracy.}

\blue{In \Cref{fig:varying_b_c} we show the train and validation accuracy for different correction bandwidth budgets $b_c$ ($32,4,2,1$ and $0.5$ bits per coordinate), while the anchor budget is fixed and equals $b_w=2$ bits per coordinate. We observe similar trends, where less bandwidth leads to lower train accuracy but possibly higher validation accuracy. Additionally, we see that one may even use a sub-bit compression ratio for the correction term, allowing for significant \emph{online} bandwidth reduction, which is especially important in our context.}

\begin{figure}[t]
\centering
\begin{subfigure}{0.48\linewidth}
  \centering
  \includegraphics[clip, trim=0.84cm 0cm 0cm 0.1cm,width=\linewidth]{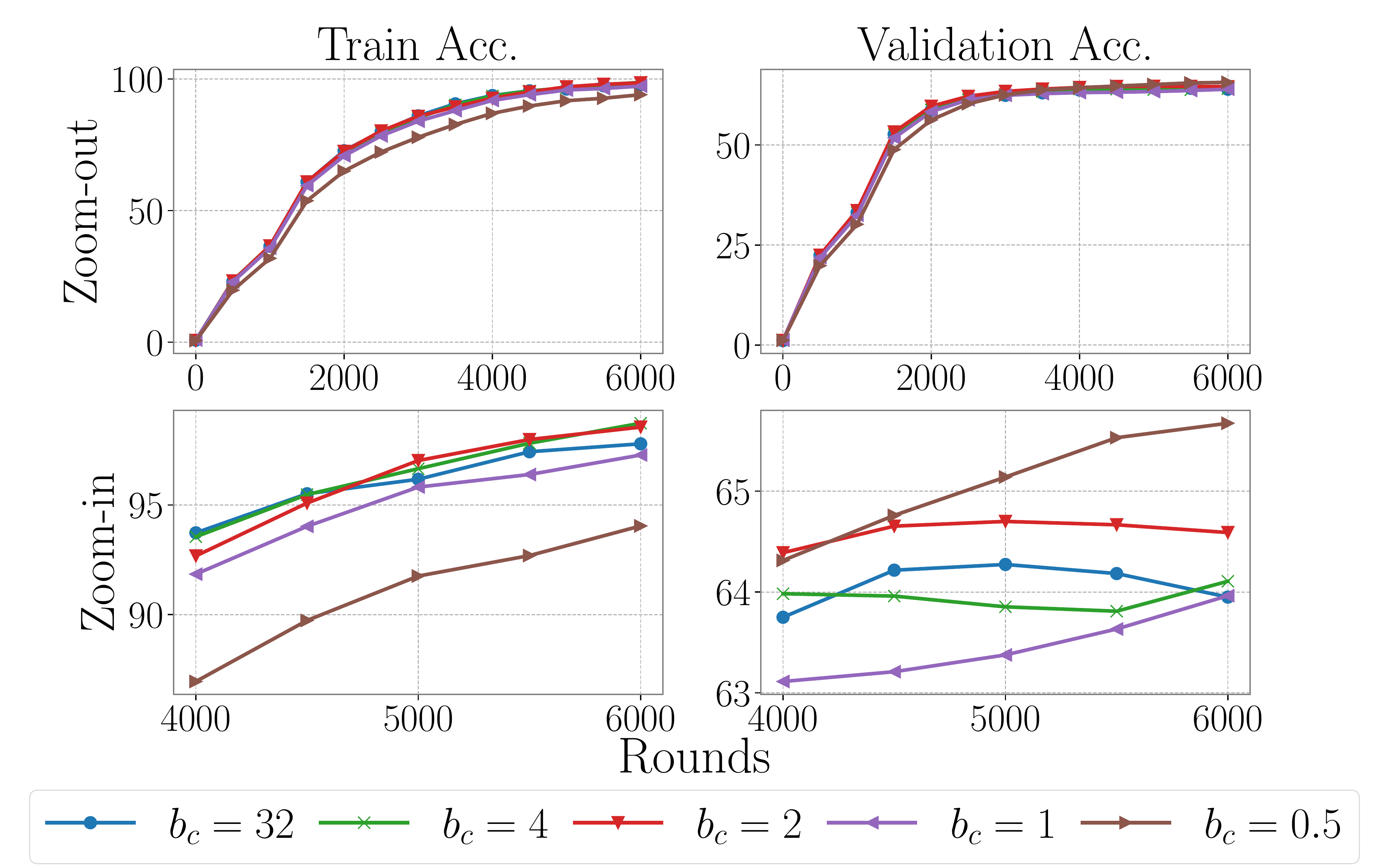}  
  \label{fig:varying_b_c_rounds}
\end{subfigure}
\hfill
\begin{subfigure}{0.48\linewidth}
  \centering
  \includegraphics[clip, trim=0.84cm 0cm 0cm 0.1cm, width=\linewidth]{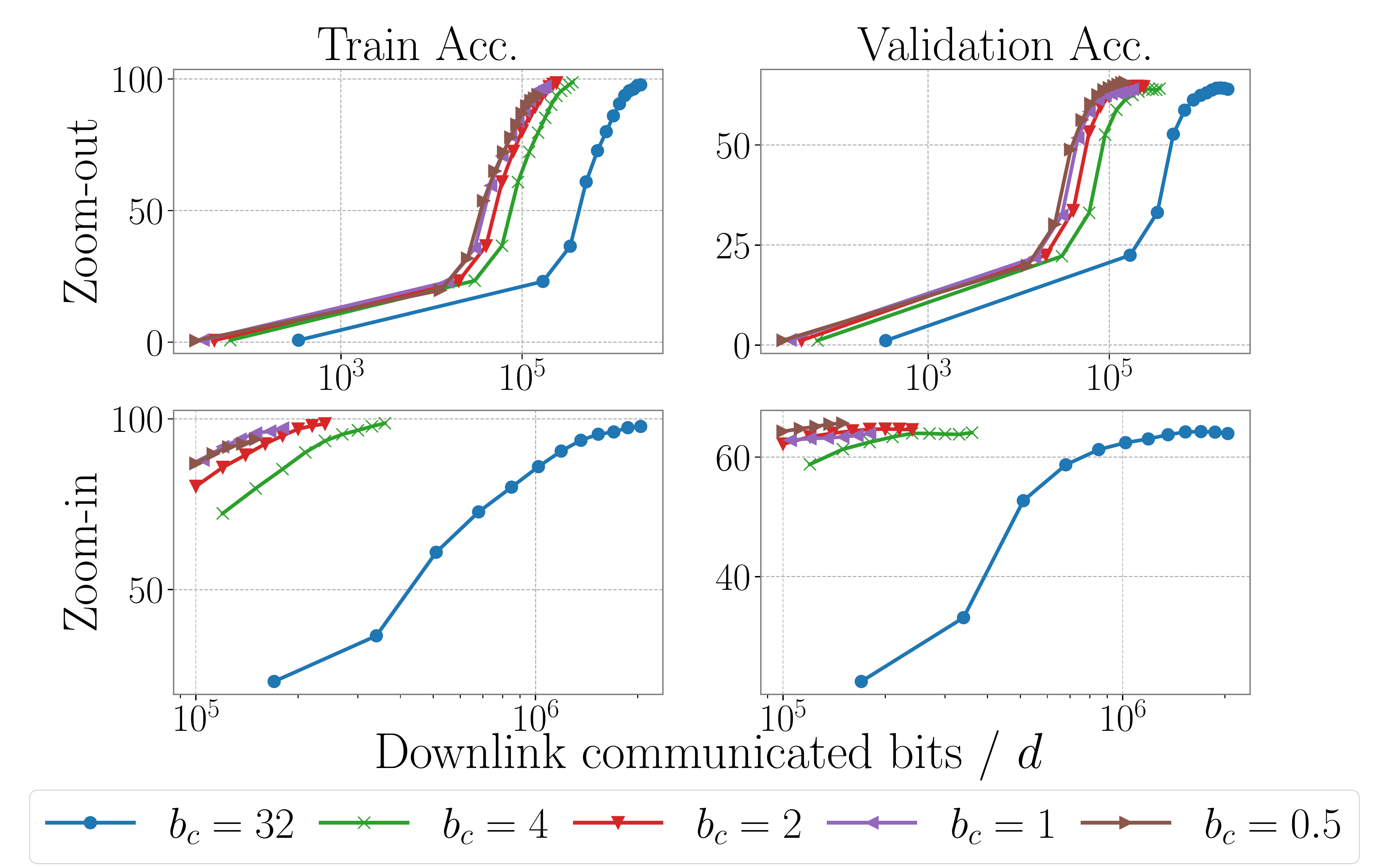}  
  \label{fig:varying_b_c_comm}
\end{subfigure}
\vspace{-3mm}
\caption{Train and validation accuracy for different correction bandwidth budgets ($32, 4, 2, 1$, and $0.5$ bits per coordinate) as a function of the number of rounds (\textbf{left}) and the number of communicated bits in the downlink direction (\textbf{right}).}
\label{fig:varying_b_c}
\end{figure}


\subsection{\blue{The Value of the Correction term}}\label{subsec:ignoring_correction}
\blue{In this section, we discuss the effect of ignoring the correction term on \methodName{}'s performance, namely, we consider the case where clients only obtain an anchor (i.e., a previous model) and use it perform local optimization. As mentioned in \S\ref{sec:experiments}, ignoring the correction may resemble other frameworks such as delayed gradients. Delayed SGD (DSGD, \citet{arjevani2020tight}) is well-studied in the literature both theoretically and empirically. Indeed, theory supports that optimization with delay can work, e.g., \citet{stich2019error} showed that as long as the maximal delay is bounded by $\Ocal(\sqrt{T})$, DSGD enjoys the same asymptotic convergence rate as SGD; \citet{cohen2021asynchronous} later improved the dependence on the maximal delay to average delay with a variant of DSGD, allowing for arbitrary delays. However, it has also been observed that, in practice, introducing delay can slow down and even destabilize convergence, and as a result hyperparameters should be chosen with great care to ensure stability~\cite{DBLP:conf/iclr/GiladiNHS20}. We thus convey that sending the correction is crucial and allows for improved performance.}

\begin{figure}[H]
    \centering
    \vspace{-5mm}
    \includegraphics[clip, trim=1.5cm 0cm 0cm 0.1cm,width=0.75\linewidth]{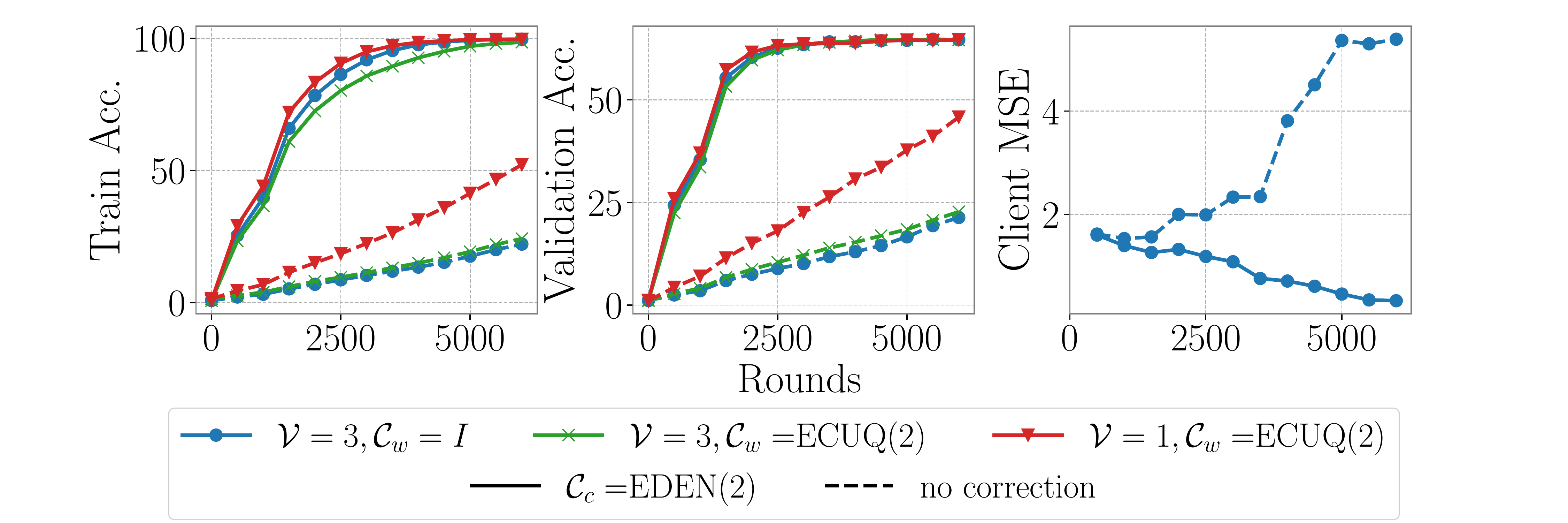}
    \vspace{-3mm}
    \caption{The effect of ignoring the correction term. Train (\textbf{left}) and validation (\textbf{middle}) accuracy for different configurations, and average client's model estimation error (\textbf{right}).}
    \label{fig:ignore_corr}
\end{figure}

\blue{To reinforce this, we conducted an experiment to numerically evaluate the effect of ignoring the correction. We consider the CIFAR-100 with ResNet-9 experiment. We test \methodName{} with and without sending the correction to the clients for three different configurations: \textbf{(1)} $\V=3$ and no anchor compression (i.e., 32 bits per coordinate); \textbf{(2)} $\V=3$ and $\C_w$ is ECUQ with $2$ bits per coordinate; and \textbf{(3)} $\V=1$ with $\C_w$ as in the second configuration. For all configurations, we use an anchor deployment rate of $K=10$. In \Cref{fig:ignore_corr} we present the train and validation accuracy, and also the average client squared model estimation error, i.e., $\frac{1}{S}\sum_{i\in\mathcal{S}_t}{\lVert w_t - \hat{w}_t^i\rVert^2}$, for the first configuration. The results clearly indicate that accounting for the correction term results in faster convergence. While ignoring the correction may eventually still result in similar performance, it is expected to take significantly more communication rounds; this is evident even when the anchor is sent with full precision. Examining the rightmost plot, we observe that ignoring the correction leads to larger model estimation error, which provides insight into why the performance deteriorates when the correction is ignored.} 




\subsection{\blue{\methodName{} and \textsf{EF21}}}\label{subsec:docofl_vs_ef21}
\blue{While our focus in on setups where a client may participate in training only once or a few times, in some setups, partial but repeated participation can be expected. For such setups, we consider some additional related work. Specifically, we focus on \textsf{EF21}~\cite{richtarik2021ef21} and some of its extensions. To assess the value of \methodName{} in this context, we attempted to extend \textsf{EF21-BC}~(Algorithm 5 of \citet{fatkhullin2021ef21}) to the partial participation setting, but were not able to achieve convergence. We suspect that it is attributed to an accumulated discrepancy between the models of the clients and the server, and thus a more sophisticated extension is required, which is out of scope. Instead, we extended \textsf{EF21-PP}~(Algorithm 4 of \citet{fatkhullin2021ef21}) to support downlink compression in two different ways: \textbf{(1)} direct compression of the model parameters using EDEN; and \textbf{(2)} using \methodName{}. We compare these approaches with a baseline that sends the exact model to the clients (i.e., no downlink compression).} 
\begin{figure}[h]
\centering
    \vspace{3mm}
  \includegraphics[clip, trim=0.58cm 0cm 0cm 1cm,width=0.9\linewidth]{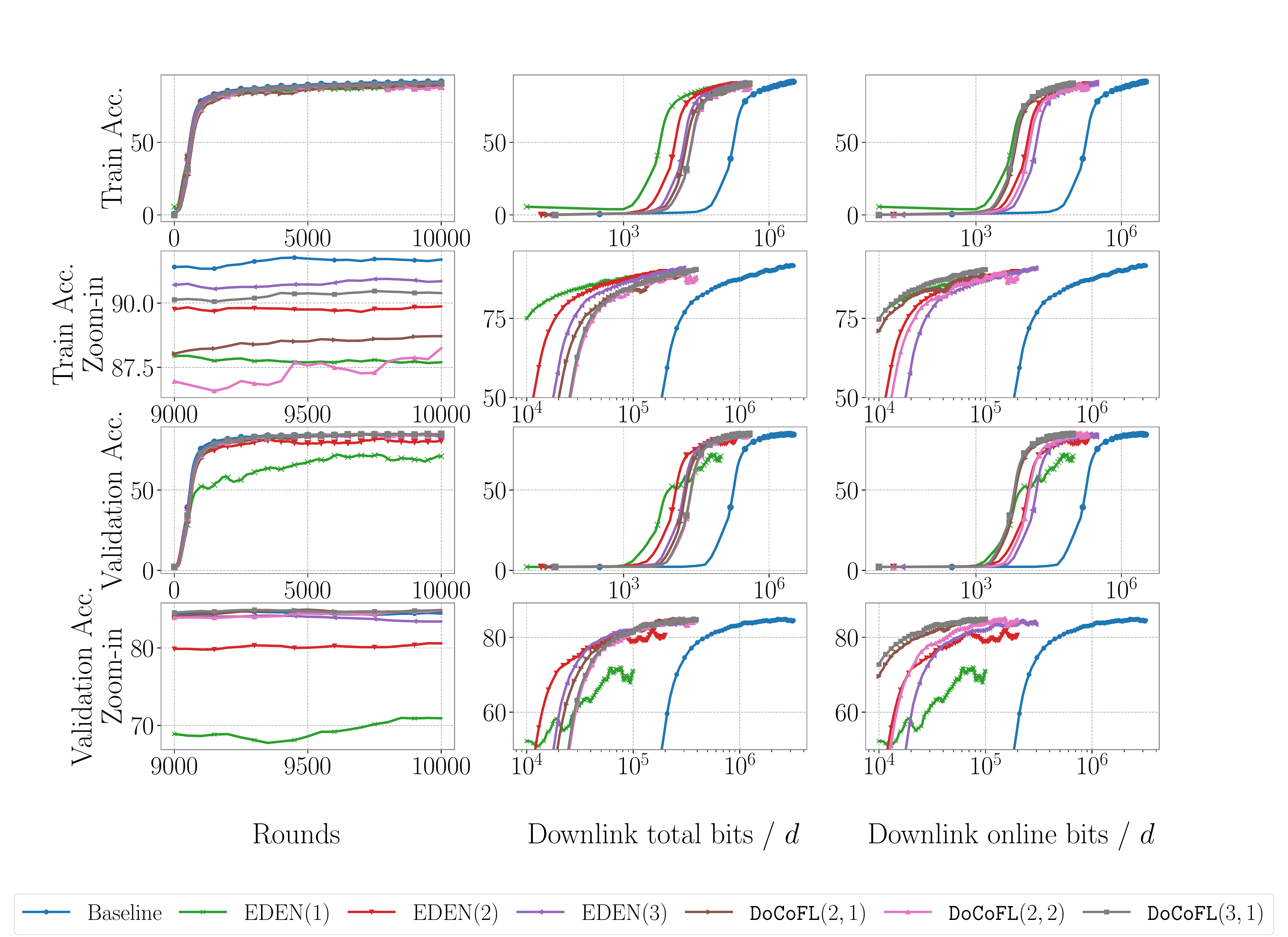}  
  \caption{Train and validation accuracy for \textsf{EF21-PP}~\cite{fatkhullin2021ef21} (baseline) and its extensions supporting downlink compression, either directly with EDEN, or with \methodName{}, over EMNIST. Results are displayed against the number of communication rounds (\textbf{left}), the total number of downlink communicated bits (\textbf{middle}), and the number of online downlink communicated bits (\textbf{right}).}
  \label{fig:ef21_emnist}
\end{figure}

\blue{We consider two tasks: \textbf{(1)} EMNIST + LeNet with $N=200$ clients and $S=10$ participating clients per-round; and \textbf{(2)} CIFAR-100 + ResNet-9 with $N=25$ clients and $S=5$ participating clients per-round. We used less clients here compared to the experiments in the main text due to GPU memory limitations (\textsf{EF21} requires keeping all $N$ clients persistent). In both experiments we use EDEN with $1$ bit/coordinate for uplink compression. \Cref{fig:ef21_emnist,fig:ef21_cifar} depict the train and validation accuracy as a function of the number of communication rounds, the total number of communicated bits in the downlink direction, and the number of communicated bits in the downlink direction required \emph{online} (i.e., at the clients' participation round) for EMNIST and CIFAR, respectively. We note that using a direct compression of the model with $1$ or $2$ bits per coordinate results in a notable drop in validation accuracy compared to the baseline. Indeed, using EDEN with $3$ bits per coordinate performs similarly to the baseline. Examining \methodName{} with $2$ and $1$ bits/coordinate for anchor and correction, respectively, reveals that it performs similarly to direct downlink compression with $3$ bits/coordinate, i.e., when using the same overall downlink bandwidth; however, it requires $3\times$ less online bandwidth. This is especially important in our context since online bandwidth demand directly translates to client delays; indeed, this is a main design goal of \methodName{}. Additionally, one may improve the results even further by increasing the anchor budget to $3$ bits/coordinate, while keeping the online bandwidth usage the same or even lower (e.g., see \Cref{fig:varying_b_c}).}

\blue{Another important point of comparison is the \textsf{EF21-P} + \textsf{DIANA} method~\cite{gruntkowska2022ef21}, which supports bi-directional compression. In particular, their server compression mechanism is similar to ours in the following sense: their server and clients hold control variates that track the global model; these control variates can be seen as an anchor that is being updated in each round and the server sends to the clients a compressed correction with respect to the control variates. However, their approach requires client-side memory with full participation (i.e., updated control variates). The authors propose to study an extension of their framework to partial participation as future work. It is interesting to investigate whether \methodName{} can be used in conjunction with this framework to achieve this.} 

\begin{figure}[t]
\centering
  \includegraphics[clip, trim=0.58cm 0cm 0cm 1cm,width=0.9\linewidth]{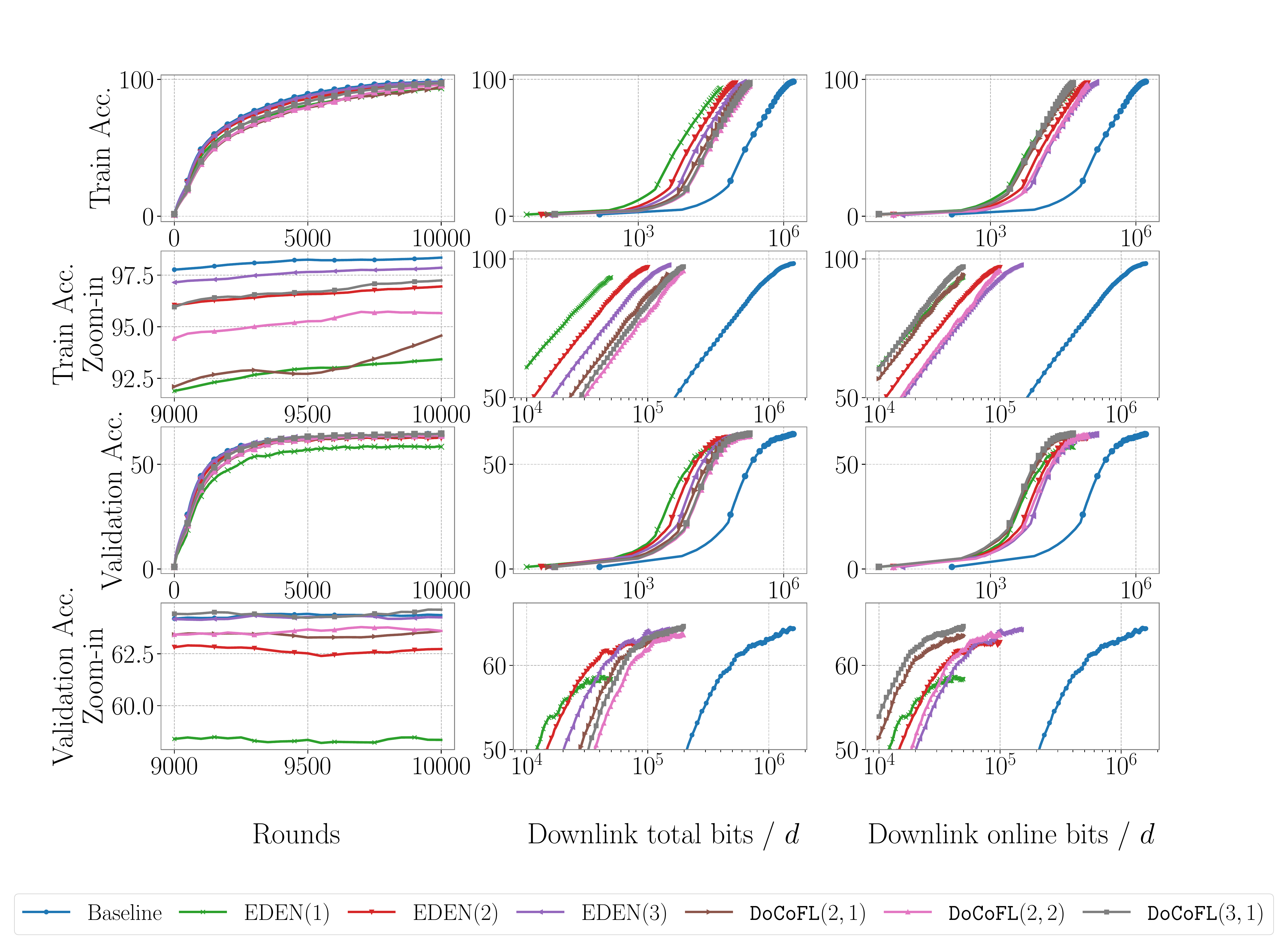}  
  \caption{Repeating the experiments from \Cref{fig:ef21_emnist} for CIFAR-100.}
  \label{fig:ef21_cifar}
\end{figure}


%% file: appendix/future_work.tex
\section{Future Work}\label{app:future_work}
\blue{We point out several directions for future research: \textbf{(1)} an interesting avenue would be to investigate how to combine \methodName{} with the delayed gradients framework~\cite{stich2019error}. While delayed gradients do not reduce downlink bandwidth, they are especially useful for clients that may require a long time to perform local updates and communicate them back to the PS. Thus, accounting for delayed gradients may enhance \methodName{}'s versatility and robustness in real FL deployments; \textbf{(2)} our theoretical framework focuses on the SGD optimizer. Exploring the implications of using adaptive optimizers, such as Adam, on the theoretical analysis and guarantees would be of great interest; \textbf{(3)} as we convey in \cref{app:goal2_intuition}, an intriguing extension of \methodName{} involves the incorporation of adaptive bandwidth budget for anchor and correction compression; although it introduces a significant theoretical challenge due to the coupling between optimization and compression, it may yield a convergence guarantee for \methodName{} with anchor compression and achieve even larger bandwidth savings; \textbf{(4)} while we employ extensive simulations and account for various overheads of \methodName{}, it is desired to further strengthen our conclusions through real deployments.}